\documentclass[twoside,11pt]{article}

\usepackage{blindtext}

\usepackage{jmlr2e}

\usepackage{microtype}
\usepackage{graphicx}
\usepackage{subfigure}
\usepackage{booktabs} 

\usepackage{enumitem}

\usepackage{hyperref}

\usepackage{algorithm}
\usepackage{algorithmic}

\hypersetup{hidelinks}

\usepackage{amsthm}
\usepackage{amsmath}
\usepackage{amssymb}
\usepackage{mathtools}
\usepackage{ulem}
\usepackage{bm}
\usepackage{amsfonts}
\usepackage{dsfont}
\usepackage{stmaryrd}

\usepackage{graphicx}
\usepackage{float}

\usepackage{multirow}
\usepackage{makecell}
\usepackage{color}
\usepackage{tablefootnote}
\usepackage{caption}

\usepackage{wrapfig,lipsum,booktabs}

\definecolor{myblue}{RGB}{0, 119, 182}
\definecolor{mygrey}{RGB}{51, 65, 92}

\usepackage[capitalize,noabbrev]{cleveref}

\theoremstyle{plain}
\newtheorem{thm}{Theorem}

\newtheorem{property}{Property}
\newtheorem{lem}{Lemma}
\newtheorem{corol}[thm]{Corollary}
\theoremstyle{definition}
\newtheorem{defi}{Definition}
\newtheorem{assumption}{Assumption}
\theoremstyle{remark}
\newtheorem{rem}{\textit{\textbf{Remark}}}

\newcommand{\x}{\bm{x}}
\newcommand{\y}{\bm{y}}
\newcommand{\xdot}{\dot{\bm{x}}}
\newcommand{\ydot}{\dot{\bm{y}}}
\newcommand{\X}{\bm{X}}
\newcommand{\Y}{\bm{Y}}
\newcommand{\xset}{\Theta_{\bm{\x}}}
\newcommand{\yset}{\Theta_{\bm{\y}}}
\newcommand{\E}{\mathbb{E}}
\newcommand{\D}{\mathcal{D}}
\newcommand{\W}{\bm{W}}
\newcommand{\nablaf}{\bm{\nabla\! f}}
\newcommand{\dx}{\nabla_{\!\bm{x}}}
\newcommand{\dy}{\nabla_{\!\bm{y}}}
\newcommand{\A}{\mathcal{A}}
\newcommand{\R}{\mathbb{R}}
\renewcommand{\S}{\mathcal{S}}
\newcommand{\FS}{F_{\mathcal{S}}}

\newcommand{\I}{\bm{I}}
\newcommand{\1}{\bm{1}}
\newcommand{\Pm}{\bm{P}_{\!m}}
\newcommand{\Tr}{\mathrm{Tr}}
\newcommand{\T}{\mathrm{T}}
\newcommand{\F}{\mathrm{F}}
\newcommand{\m}{\interleave}
\newcommand{\lambdabm}{\boldsymbol{\lambda}}

\usepackage{lastpage}
\jmlrheading{}{2024}{1-\pageref{LastPage}}{11/24; Revised --/--}{--/--}{21-0000}{Miaoxi Zhu, Yan Sun, Li Shen, Bo Du, Dacheng Tao}

\ShortHeadings{Stability and Generalization}{Distributed SGDA}
\firstpageno{1}

\begin{document}

\title{Stability and Generalization for Distributed SGDA}

\author{\name Miaoxi Zhu \email zhumx@whu.edu.cn \\
       \addr School of Computer Science\\
       Wuhan University, Wuhan, China
       \AND
       \name Yan Sun \email ysun9899@uni.sydney.edu,au \\
       \addr School of Computer Science\\
       The University of Sydney, Sydney, Australia
       \AND
       \name Li Shen \email mathshenli@gmail.com\\
       \addr School of Cyber Science and Technology\\
       Shenzhen Campus of Sun Yat-sen University, Shenzhen, China
       \AND
       \name Bo Du \email dubo@whu.edu.cn\\
       \addr School of Computer Science\\
       Wuhan University, Wuhan, China
       \AND
       \name Dacheng Tao \email dacheng.tao@ntu.edu.sg\\
       \addr College of Computing $\&$ Data Science \\
       Nanyang Technological University, Singpore
       }

\editor{}

\maketitle

\begin{abstract}

Minimax optimization is gaining increasing attention in modern machine learning applications. Driven by large-scale models and massive volumes of data collected from edge devices, as well as the concern to preserve client privacy, communication-efficient distributed minimax optimization algorithms become popular, such as Local Stochastic Gradient Descent Ascent (Local-SGDA), and Local Decentralized SGDA (Local-DSGDA). While most existing research on distributed minimax algorithms focuses on convergence rates, computation complexity, and communication efficiency, the generalization performance remains underdeveloped, whereas generalization ability is a pivotal indicator for evaluating the holistic performance of a model when fed with unknown data. In this paper, we propose the stability-based generalization analytical framework for Distributed-SGDA, which unifies two popular distributed minimax algorithms including Local-SGDA and Local-DSGDA, and conduct a comprehensive analysis of stability error, generalization gap, and population risk across different metrics under various settings, e.g., (S)C-(S)C, PL-SC, and NC-NC cases. Our theoretical results reveal the trade-off between the generalization gap and optimization error and suggest hyperparameters choice to obtain the optimal population risk. Numerical experiments for Local-SGDA and Local-DSGDA validate the theoretical results.

\end{abstract}


\section{Introduction}\label{sec:intro}

\captionsetup[table]{labelfont={bf}}
\begin{table*}[t]
\caption{Summary on theoretical results for Local-SGDA in comparison with SGDA.}
\vspace{-0.5cm}
\label{table:results:localsgda}
\renewcommand\arraystretch{1.8}
\begin{center}
\scalebox{0.95}{\begin{tabular}{c|c|c|l}
\hline
Assumption                        & Measure                                       & Algorithm   & Bound         \\ \hline
\multirow{4}{*}{$\mu$-SC-SC,$G$,$L$ \textcolor{mygrey}{$^\text{a}$}}       & \multirow{2}{*}{\makecell[c]{Argument \\ Stability}}             & Local-SGDA & $\mathcal{O}(\frac{\sqrt{K}}{T}+\frac{1}{n})$ \textcolor{mygrey}{$^\text{b}$} \textcolor{myblue}{(Thm. \ref{thm:local:sc-sc-stability})} \textcolor{mygrey}{$^\text{e}$}  \\ \cline{3-4} 
    & & SGDA  & $\mathcal{O}(\frac{\sqrt{T}}{n}+\frac{1}{\sqrt{n}})$  (\cite{lei2021stability}) \\ \cline{2-4} 
    & \multirow{2}{*}{\makecell[c]{Weak PD\\Population Risk}}       & Local-SGDA & $\widetilde{\mathcal{O}}(\frac{\sqrt{K}}{T}+\frac{1}{\sqrt{TK}}+\frac{1}{n})$ \textcolor{mygrey}{$^\text{c}$} \textcolor{myblue}{(Thm. \ref{thm:sc-sc-population})}   \\ \cline{3-4} 
     &    & SGDA  & $\mathcal{O}(\frac{\ln n}{n\mu})$ (\cite{lei2021stability}) \\ \hline
\multirow{2}{*}{\makecell[c]{$\rho$PL-$\mu$SC,$G$,$L$\\$F(\x,\cdot)$ is $\mu$-SC}} & \makecell[c]{Primal\\ Stability}             & Local-SGDA & $\mathcal{O}(\frac{K^{\frac{1+\alpha}{2}}}{\sqrt{T}}+\frac{1}{n}),\frac{1}{2}<\alpha<1$ \textcolor{myblue}{(Thm. \ref{thm:primal})}   \\ \cline{2-4} 
    & \makecell[c]{Excess Primal \\Population Risk} & Local-SGDA & $\mathcal{O}(\frac{K^{\frac{1+\alpha}{2}}}{\sqrt{T}}+\frac{1}{mn}),\frac{1}{2}<\alpha<1$ \textcolor{myblue}{(Thm. \ref{thm:ep:population})}  \\  \hline
\multirow{2}{*}{NC-NC,$G$,$L$}            & \multirow{2}{*}{Weak Stability}                &  Local-SGDA & $\widetilde{\mathcal{O}}(\frac{m^{\frac{4}{5}}}{n}+\frac{m^{\frac{4}{5}}}{n^{\frac{4}{5}}}T^{\frac{1}{5}}K^{\frac{3}{5}})$ \textcolor{myblue}{(Thm. \ref{thm:weak})}  \\ \cline{3-4} 
    &   & SGDA  & $\mathcal{O}(\frac{T^{\frac{2c\rho}{2c\rho+3}}}{n^{\frac{2c\rho+1}{2c\rho+3}}})$ (\cite{lei2021stability}) \\ \hline
\end{tabular}}
\end{center}
\vspace{-0.5cm}
\end{table*}

\begin{table*}[t]
\caption{Summary on theoretical results for Local-DSGDA in comparison with DSGDA.}
\vspace{-0.5cm}
\label{table:results:localdsgda}
\renewcommand\arraystretch{1.8}
\begin{center}
\scalebox{0.85}{\begin{tabular}{c|c|c|l}
\hline
Assumption                        & Measure                                                        & Algorithm   & Bound         \\ \hline
\multirow{4}{*}{$\mu$-SC-SC,$G$,$L$}      & \multirow{2}{*}{\makecell[c]{Argument \\ Stability}}            &  Local-DSGDA & $\mathcal{O}(\frac{(\sqrt{\lambdabm_1}+\sqrt{\lambdabm_2})\sqrt{K}}{T}+\frac{1}{n})$ \textcolor{mygrey}{$^\text{d}$} \textcolor{myblue}{(Thm. \ref{thm:local:sc-sc-stability})} \\ \cline{3-4} 
    &     & DSGDA  & $\mathcal{O}(\frac{\eta}{1-\lambdabm}+\frac{1}{n})$ (\cite{zhu2023stability}) \\ \cline{2-4} 
    & \multirow{2}{*}{\makecell[c]{Weak PD\\Population Risk}}       &  Local-DSGDA & $\widetilde{\mathcal{O}}((\sqrt{\lambdabm_1}+\sqrt{\lambdabm_2})\frac{\sqrt{K}}{T}+\frac{1}{\sqrt{KT}}+\frac{1}{n})$ \textcolor{myblue}{(Thm. \ref{thm:sc-sc-population})}  \\ \cline{3-4} 
    &     & DSGDA  & $\mathcal{O}(\frac{1}{n}+\frac{1}{(1-\lambdabm)\sqrt{T}})$ (\cite{zhu2023stability})  \\ \hline
\multirow{2}{*}{\makecell[c]{$\rho$PL-$\mu$SC,$G$,$L$\\$F(\x,\cdot)$ is $\mu$-SC}} & \makecell[c]{Primal\\ Stability}               & Local-DSGDA & $\mathcal{O}(\frac{(\sqrt{\lambdabm_1}+\sqrt{\lambdabm_2})K^{\frac{1+\alpha}{2}}}{\sqrt{T}}+\frac{1}{n}),\frac{1}{2}<\alpha<1$ \textcolor{myblue}{(Thm. \ref{thm:primal})} \\ \cline{2-4} 
    & \makecell[c]{Excess Primal \\Population Risk} & Local-DSGDA & $\mathcal{O}(\frac{(\sqrt{\lambdabm_1}+\sqrt{\lambdabm_2})K^{\frac{1+\alpha}{2}}}{\sqrt{T}}+\frac{1}{mn}),\frac{1}{2}<\alpha<1$ \textcolor{myblue}{(Thm. \ref{thm:ep:population})} \\ \hline
\multirow{2}{*}{NC-NC,$G$,$L$}            & \multirow{2}{*}{Weak Stability}                 & Local-DSGDA & $\widetilde{\mathcal{O}}((\frac{1}{n}+\sqrt{\lambdabm_1}+\sqrt{\lambdabm_2}K^{\frac{1}{2}})^{\frac{1}{5}}(\frac{m}{n})^{\frac{4}{5}}T^{\frac{1}{5}}K^{\frac{3}{5}})$ \textcolor{myblue}{(Thm. \ref{thm:weak})} \\ \cline{3-4} 
     &    & DSGDA  & $\mathcal{O}(\frac{\eta T}{n}+\frac{\eta^2T}{1-\lambdabm})$ (\cite{zhu2023stability}) \\ \hline
\end{tabular}}
\end{center}
\footnotesize{\textcolor{mygrey}{$^\text{a}$} $G$ and $L$ represents for Lipschitz continuity and smoothness on local function $f_i(\x,\y;\xi_{i,l})$;\\
\textcolor{mygrey}{$^\text{b}$} $T$ and $K$ denote the number of communication round and local update iteration, $m$ denotes the number of agents and $n$ represents the sample size of local dataset;\\
\textcolor{mygrey}{$^\text{c}$} $\widetilde{\mathcal{O}}$ means neglect of a logarithmic term;\\
\textcolor{mygrey}{$^\text{d}$} $\lambdabm_1$ and $\lambdabm_2$ are $\lambdabm$-relevant constants which is speified in Remark \ref{remark:lambda};\\
\textcolor{mygrey}{$^\text{e}$} the \textcolor{myblue}{blue} text refers to results presented in this paper.}
\vspace{-0.5cm}
\end{table*}

Minimax problems have been widely investigated in machine learning community, including generative adversarial networks (GANs) \citep{goodfellow2014generative,arjovsky2017wasserstein,cohen2019certified}, generative adversarial imitation learning \citep{ho2016generative}, reinforcement learning \citep{cheng2023prescribed}, robustness \citep{madry2018towards}, algorithmic fairness \citep{song2019learning}, multi-agent games \citep{cao2002internet}, domain adaptation \citep{zhao2018adversarial}, federated learning \citep{sharma2022federated}, etc.

Large amounts of data obtained from individual devices pose challenges to the training process in large-scale machine learning scenarios, which calls for communication-efficient distributed minimax algorithms \citep{liu2020decentralized,chen2021proximal,beznosikov2022decentralized}. Considering $m$ agents with local loss function $F_i(\x,\y)$ and local data $\xi_i$ adhere to the local data distribution $\D_i$, the objective function can be formulated as follows:
\begin{equation}\label{eq:obj}
            \min_{\x\in\xset}\max_{\y\in\yset}F(\x,\y)
            :=\frac{1}{m}\sum_{i=1}^m F_i(\x,\y) :=\frac{1}{m}\sum_{i=1}^m\E_{\xi_i\sim\D_i}[f_i(\x,\y;\xi_i)].
\end{equation}
Existing works on distributed minimax problems \eqref{eq:obj} mainly focus on faster convergence rate \citep{xian2021faster,sharma2022federated,shen2023stochastic} and lower communication overhead \citep{beznosikov2022distributed,yang2022sagda}, in various minimax settings including convex-convex (C-C) \citep{liao2022local}, nonconvex-convex (NC-C) \citep{deng2020distributionally}, and nonconvex-nonconcave with Polyak-{\L}ojasiewicz condition (NC-PL) \citep{reisizadeh2020robust}.  However, limited research investigates the generalization performance of distributed minimax algorithms which is an important pivot indicating the overall performance.

Optimization behavior is not the only factor deciding the success of a model obtained from a stochastic distributed minimax algorithm. In practice, data distribution $\D_i$ is unknown and we have no access to the absolute or gradient of the expectation value. As a compromise, we have to sample from local datasets and estimate the expectation by the averaged loss on dataset $\S=\{\S_1,...,\S_m\}$ with local dataset denoted as $\S_i=\left\{\xi_{i,1},...,\xi_{i,l_i}\right\}$:
    \begin{equation}\label{eq:emp_obj}
        \min_{\x\in\xset}\max_{\y\in\yset}F_{\S}(\x,\y)
         :=\frac{1}{m}\sum_{i=1}^mF_{\S_i}(\x,\y)
        :=\!\frac{1}{m}\sum_{i=1}^m\frac{1}{n}\sum_{l_i=1}^nf_i(\x,\y;\xi_{i,l_i}).
\end{equation}
Consequently, the model learned from the empirical loss (see Eq.\eqref{eq:emp_obj}) on the finite training dataset may not exhibit comparable performance, let alone an improvement, on the unknown data distribution, i.e., can not optimize the objective function (see Eq.\eqref{eq:obj}). The discrepancy between Eq.\eqref{eq:obj} and Eq.\eqref{eq:emp_obj} reflects the generalization ability of models $(\x,\y)$, which is of vital importance to predict how the model behaves on the overall distribution.

However, it is not trivial to measure the generalization performance by directly subtracting Eq.\eqref{eq:obj} from Eq.\eqref{eq:emp_obj} owing to the complex structure of minimax problems, unlike standard learning theory where there is only one variable to minimize. Besides, the communication between agents and the accumulated local models increases the difficulty in both the parameter-server and peer-to-peer manner. Thus a natural question arises: 
\begin{center}
\vspace{-0.1cm}
    \textit{\textbf{What factors influence the generalization performance for the general distributed minimax problems?}}
 \vspace{-0.1cm}
\end{center}

In this paper, we propose a generalized analyzing \textit{Distributed-SGDA} framework, which unifies two popular distributed minimax optimization algorithms as special cases, including both the Local-SGDA and Local-DSGDA in the centralized and decentralized setting, respectively. Based on the \textit{Distributed-SGDA} framework, we develop the comprehensive theoretical analysis for analyzing the stability-based generalization and population risk for Local-SGDA and Local-DSGDA and the theoretical results for Local-SGDA and Local-DSGDA are summarized in Table \ref{table:results:localsgda} and Table \ref{table:results:localdsgda} in comparison with corresponding SGDA and DSGDA. Preliminary experiments are in accordance with our theoretical findings.

\textbf{Contribution.} Our main contributions are specified as:
\begin{itemize}[leftmargin=0pt,itemindent=\labelwidth+\labelsep,labelindent=0pt,itemsep=2pt,topsep=0pt,parsep=0pt]
    \item \textit{First work on the stability-based generalization analysis on Local-SGDA and Local-DSGA.} We unify the analyzing framework of Distributed-SGDA on algorithmic stability, generalization gap and population risk, which haven't been explored to our best knowledge.
    \item \textit{Theoretical results.} For generalized distributed minimax algorithms (not limited to \textit{Distributed-SGDA}), we establish the connections \textbf{(\textit{i})} between argument stability and weak PD generalization gap in (strongly) convex-(strongly) concave \big((S)C-(S)C\big) case; \textbf{(\textit{ii})} between primal stability and excess primal generalization gap in nonconvex with Polyak-{\L}ojasiewicz condition-strong concave (PL-SC) case; \textbf{(\textit{iii})} between weak stability and weak PD generalization gap in nonconvex-nonconcave (NC-NC) case. 
    We provide corresponding error bounds for \textit{Distributed-SGDA}, and evaluate the optimization error in (S)C-(S)C and PL-SC cases, demonstrating the trade-off between generalization gap and optimization error, which hints at the parameter choice for a better balance and the optimal population risk. 
    \item \textit{Experiments.} We conduct experiments of Local-SGDA and Local-DSGDA on vanilla GAN to simulate NC-NC case and on AUC Maximization to simulate C-C case, and vary factors to explore isolated influence. The experimental results validate theoretical fingdings.
\end{itemize}

\section{Related Work}

\textbf{Distributed Minimax Optimization.} \citet{deng2021local} provides convergence analysis for Local-SGDA in both homogeneous and heterogeneous data under SC-SC and NC-SC settings and proposes Local SGDA+ for NC-NC. \citet{hou2021efficient} derives convergence rates for SCAFFOLD-S and FedAvg-S in the strongly convex-strongly concave case. \citet{beznosikov2023unified} estimates the convergence rates, the number of local computations, and communications for the local extra step method in the (strongly) monotone case. \citet{tarzanagh2022fednest} establishes new convergence guarantees for nonconvex FedSVRG. \citet{xie2023cdma} proves convergence rates and communication cost for CDMA under a class of nonconvex-nonconcave settings. \citet{zhang2023communication} establishes convergence rates of their proposed novel communication-efficient algorithms for distributed minimax optimization problems. \citet{sharma2023federated} provides gradient complexity and communication rounds for Local-SGDA under NC-PL, NC-C, and NC-1PL conditions. \citet{zhang2024federated} establishes the theoretical convergence rate of their proposed federated minimax optimization algorithm. \citet{wu2023solving} presents convergence rates for their proposed FedSGDA+ and FedSGDA-M under NC-C, NC-PL and NC-SC settings. For more distributed minimax optimization algorithms, \citet{razaviyayn2020nonconvex} makes a survey on the literature.

\textbf{Generalization for Minimax Optimization.} Stability-based generalization analysis is pioneered by \citet{bousquet2002stability}, and then developed by \citet{elisseeff2005stability,rakhlin2005stability,shalev2010learnability,hardt2016train}. For general minimax optimization, comprehensive investigation has been explored. \cite{farnia2021train} analyzes generalization properties of GDA and PPM algorithms in C-C case, as well as the generalization performance of SGDA and SGDmax in NC-NC case. \cite{lei2021stability} establishes the connection between stability and generalization for minimax optimization algorithms, establishes stability and generalization bounds for SGDA under C-C condition, and extends their analysis to nonsmooth loss functions. \cite{xing2021algorithmic} proposes noise-injected method to enhance the stability performance of adversarial learning. \cite{zhang2021generalization} establishes a uniform stability argument for the ESP solution and further provides its generalization bounds. \cite{ozdaglar2022good} proposes a new metric that outperforms the traditional one in indicating the generalization performance in stochastic minimax optimization. \cite{yang2022differentially} analyzes the population risk of Differentially Private SGDA under C-C and NC-SC settings. While in distributed minimax scenario, generalization performance is rarely explored. \cite{zhu2023stability} analyzes the stability and generalization of decentralized SGDA. \citet{zhao2018adversarial} presents generalization bound in specific scenarios of multiple-source domain adaptation. \citet{mohri2019agnostic} presents data-dependent Rademacher complexity for agnostic federated learning. \citet{sun2022communication} generalize the results and provide the generalization bounds for distributed minimax learning.  
\section{Problem Formulation}

In this section, we first present the unified distributed-SGDA framework, which includes Local-SDGA and Local-DSGDA as special cases. Then, we introduce several terminologies for characterizing the generalization for distributed-SGDA.

\subsection{Distributed-SGDA Framework}

We specify that there are $m$ workers in the distributed setting with local dataset $\S_i=\{\xi_{i,1},\xi_{i,2},...\xi_{i,n}\}$ and the local models are parameterized by $(\x_i,\y_i)$. We use $\A(T, K, \W)$ to denote the distributed SGDA algorithms, where $T$ and $K$ denote the number of the communication round and the local update iteration respectively. At the same time, $\W$ represents the way every node communicates with each other. There are mainly two phases in the update schemes: local update and global communication (see Algorithm~\ref{alg:distributedsgda}).

\begin{wrapfigure}{l}{0.55\textwidth}
\begin{minipage}{0.55\textwidth}
    \vspace{-0.4cm}
    \begin{algorithm}[H]
    \renewcommand{\algorithmicrequire}{\textbf{Initialize:}}
    \renewcommand{\algorithmicensure}{\textbf{Output:}}
    \caption{Distributed-SGDA ($\A(T,K,\W)$) }
    \label{alg:distributedsgda}
    \begin{algorithmic}[1]
        \REQUIRE $\x^0_i=0$; $\y^0_i=0$, $i=1,...,m$
        \FOR{$t = 1, 2, \cdots, T$}
        \STATE $\x^t_{i,0}=\x^{t-1}_i$, $\y^t_{i,0}=\y^{t-1}_i$; $\forall i=1,\cdots,m$
        \FOR{$k=0, 1, \cdots, K-1$}
        \STATE $\x^t_{i,k+1}=\x^t_{i,k}-\eta^t_k\dx f_i(\x^t_{i,k},\y^t_{i,k};\xi_{i,j_k^t(i)})$
        \STATE $\y^t_{i,k+1}=\y^t_{i,k}+\eta^t_k\dy f_i(\x^t_{i,k},\y^t_{i,k};\xi_{i,j_k^t(i)})$
        \ENDFOR
        \STATE $\x^t_i=\sum_h \omega_{ih}\x^t_{h,K}$; $\y^t_i=\sum_h \omega_{ih}\y^t_{h,K}$
        \ENDFOR
        \ENSURE $\x^T=\frac{1}{m}\sum_{i=1}^m\x^T_i,\y^T=\frac{1}{m}\sum_{i=1}^m\y^T_i$.
    \end{algorithmic}
    \end{algorithm}
\end{minipage}
\end{wrapfigure}

\vspace{4pt}
\noindent\textbf{(1). Local Update.}
At the $k$-th local iteration, every worker $(\x^t_{i.k},\y^t_{i,k})$ samples $\xi_{i,j^t_k(i)}$ from local dataset $\S_i$ and conducts stochastic gradient descent ascent to update local model parameters to $(\x^t_{i,k+1},\y^t_{i,k+1})$.

\vspace{3pt}
\noindent\textbf{(2). Global Communication.}
During the $t$-th communication round, after $K$ iterations of local update, every worker $(\x^t_{i,K},\y^t_{i,K}\!)$ communicates with each other through the mixing matrix $\W\!\!=\!\!(\omega_{ij})\!\!\in\!\!\R^{m\times m\!}$. The updated models $(\x^t_i,\y^t_i)$ act as the initialization of the next communication round.

\begin{assumption}[\textbf{Mixing Matrix}]\label{as:mixing}
We define the mixing matrix $\W=(\omega_{ij})\in\R^{m\times m}$ to prescribe how different nodes communicate with each other, where $0\leq\omega_{ij}\leq1$ signifies the weight assigned to worker $j$ when receiving information from worker $i$. In our analysis, we assume the mixing matrix to be symmetric doubly stochastic that $\W=\W^{\T}$ and $\W\cdot\1_m=\1_m$, where $\1_m\in\R^m$ denotes the all-one vector.
\end{assumption}

\begin{defi}\label{def:crucial}
    For a mixing matrix satisfying Assumption \ref{as:mixing}, we define a critical figure associated with the mixing matrix: $\lambdabm\triangleq\max\{\lambda_2(\W),...\lambda_m(\W)\}$, where $\lambda_i(\W)$ represents the $i$-th largest eigenvalue. We assume $\lambdabm<1$ in the analyzing literature of this paper.
\end{defi}

\begin{rem} 
Assumption~\ref{as:mixing} is common for undirected connected graphs. It is trivial that $\lambda_1(\W)=1$ since $\1_m$ acts as the eigenvector. As $\lambdabm$ approaches $1$, the matrix tends to sparsity, while $\lambdabm=0$ denotes a fully-connected graph. More precise relations between the value of $\lambdabm$ and the topology have been discussed in \citep{nedic2018network,ying2021exponential}.
\end{rem}

\begin{rem}
In Algorithm \ref{alg:distributedsgda}, the learning rates for variables $\x$ and $\y$ are set to be the same without loss of generalization. While \citet{zhu2023stability} focuses on decentralized SGDA and presents a precise analysis of the case when variables $\x$ and $\y$ choose different learning rates. 
\end{rem}

\begin{rem}
By setting different values of mixing matrix $\W$ for the Distributed SGDA, we can find the correspondence with existing distributed algorithms. 
\textbf{(i)} For the mixing matrix $\W$ satisfying Assumption \ref{as:mixing}, Generalized $\A(T,K,\W)$ is Local-DSGDA. 
\textbf{(ii)} When we choose the mixing matrix as $\Pm\in\R^{m\times m}$ whose elements are all $\frac{1}{m}$, $\A(T,K,\Pm)$ becomes Local-SGDA. The communication stage turns into the process of averaging local models that $(\x^t_i,\y^t_i)=(\frac{1}{m}\sum_{i=1}^m\x^t_{i,K},\frac{1}{m}\sum_{i=1}^m\y^t_{i,K})$. 
\textbf{(iii)} When $K=1$, $\A(T,K,\W)$ degenerates to Decentralized-SGDA, where each node conducts a single step of SGDA and communicates with their neighbors to update models.
\end{rem}

\subsection{Generalization Gap}
Recall the objective function and the empirical loss function in Section \ref{sec:intro}, the optimal solution of problem \eqref{eq:emp_obj} cannot be optimal for problem \eqref{eq:obj}, and their gap is evaluated as generalization performance. Unlike the minimization structure where the generalization gap can be directly calculated as the difference between the population loss and the empirical loss, in the distributed minimax setting, the population risk, the empirical risk, and the generalization gap have to be defined in different manners, according to the duality-gap.

\begin{defi}[\textbf{Weak Primal-Dual(PD) Generalization Gap}]
For a randomized model $(\x,\y)$, the weak PD population risk and weak PD empirical risk are defined as:
    \begin{gather*}
    \Delta^{\!\mathrm{w}}(\x,\y)\triangleq\sup_{\y'\in\yset}\E[F(\x,\y')]-\inf_{\x'\in\xset}\E[F(\x',\y)]\\
    \Delta^{\!\mathrm{w}}_{\S}(\x,\y)\triangleq\sup_{\y'\in\yset}\E[F_{\S}(\x,\y')]-\inf_{\x'\in\xset}\E[F_{\S}(\x',\y)].   
    \end{gather*}
    For model $(\A_{\x}(\S),\A_{\y}(\S))$ which is the output of the algorithm $\A$ trained on the dataset $\S$, the corresponding weak PD generalization gap is defined by:
    \[\zeta^{\mathrm{w}}_{\mathrm{gen}}\!(\A,\S)\triangleq\Delta^{\!\mathrm{w}}\!(\A_{\x}(\S),\A_{\y}(\S))\!-\!\Delta^{\!\mathrm{w}}_{\!\S}\!(\A_{\x}(\S),\A_{\y}(\S)).\]
\end{defi}

Similarly, we further define the \textbf{Strong PD Generalization Gap} analogously as the discrepancy between  $\Delta^{\mathrm{s}}(\x,\y)\!\triangleq\!\E[\sup_{\y'\in\yset}F(\x,\y')\!-\!\inf_{\x'\in\xset}F(\x',\y)]$ and $\Delta^{\mathrm{s}}_{\S}(\x,\y)\!\triangleq\E[\sup_{\y'\in\yset}\!F_{\S}(\x,\y')\!-\!\inf_{\x'\in\xset}\!F_{\S}(\x',\y)]$, i.e.,
\[
\zeta^{\mathrm{s}}_{\mathrm{gen}}\!(\A,\S)\!\triangleq\!\Delta^{\!\mathrm{s}}(\A_{\x}(\S),\A_{\y}(\S))\!-\!\Delta^{\!\mathrm{s}}_{\!\S}(\A_{\x}(\S),\A_{\y}(\S)).
\]
We have $\zeta^{\mathrm{w}}_{\mathrm{gen}}\leq\zeta^{\mathrm{s}}_{\mathrm{gen}}$ owing to the order of the expectation and operation of $\sup,\inf$. The strong PD generalization gap is not our main concern since the conduction usually requires SC-SC condition, while the weak one can suit more cases even in NC-NC condition. Also, notice that the expectation here is taken on the randomness of $\A,\S$; we may omit the subscripts if no misunderstanding occurs.

\begin{defi}[\textbf{Excess Primal Generalization Gap}]
For a randomized model $\x$, firstly we define the primal population risk and the primal empirical risk as $R(\x)=\sup_{\y'\in\yset}F(\x,\y')$, $R_{\S}(\x)=\sup_{\y'\in\yset}F_{\S}(\x,\y')$ respectively. The excess primal population risk and the excess primal empirical risk are defined below:
    \begin{gather*}
    \Delta^{\!\mathrm{ep}}(\x)\triangleq \E[R(\x)-\inf_{\x'\in\xset}R(\x')],\\
    \Delta^{\!\mathrm{ep}}_{\S}(\x)\triangleq \E[R_{\S}(\x)-\inf_{\x'\in\xset}R_{\S}(\x')]. 
       \end{gather*}
    For model $\A_{\x}(\S)$ which is the output of the algorithm $\A$ trained on the dataset $\S$, the corresponding excess primal generalization gap is defined by:
    \[\zeta^{\mathrm{ep}}_{\mathrm{gen}}\!(\A,\S)\triangleq\Delta^{\!\mathrm{ep}}(\A_{\x}(\S))\!-\!\Delta^{\!\mathrm{ep}}_{\S}(\A_{\x}(\S)).\]
\end{defi}
As validated in \citet{ozdaglar2022good}, the excess primal generalization gap (named generalization error for primal gap) can indicate the generalization ability better than conventional primal generalization error defined as $R(\x)-R_{\S}(\x)$, especially in the nonconvex-convex (NC-C) setting. The fact that the optimal solution of $R_{\S}(\cdot)$ cannot minimize $R(\cdot)$ comparably well may account for the shortage.

\subsection{Algorithmic Stability }
We introduce various metrics for distributed algorithmic stability based on the complex structure of the minimax problem, such as SC-SC, NC-SC, and NC-NC. Firstly, we present the concept of the distributed neighboring dataset, which degenerates into the vanilla definitions of the neighboring dataset (\citet{lei2021stability},\citet{ozdaglar2022good}) when $i=1$. Subsequently, algorithmic stability can be assessed by evaluating the discrepancy between output models trained on the distributed neighboring dataset.

\begin{defi}[\textbf{Distributed neighboring dataset}]\label{def:neighbouring:dataset}
    We call $\S=\{\S_1,...,\S_m\}$ and $\S'=\{\S'_1,...,\S'_m\}$ the distributed neighboring dataset when there exists at most one different sample between each local dataset $\S_i$ and $\S'_i$, $\forall i=1,...,m$.
\end{defi}

\begin{defi}[\textbf{Distributed algorithmic stability}]\label{def:stability}
For a randomized distributed algorithm $\A$, we say it is $\epsilon$-stable if for any distributed neighboring dataset $\S,\S'$, it holds that:
\begin{enumerate}[label=(\textit{\roman*}),wide=\parindent,labelindent=0pt,itemsep=2pt,topsep=0pt,parsep=0pt]
        \item \textbf{argument stability:}
        \begin{equation*}
            \E_{\A}\left\|\left(\begin{array}{c}
           \A_{\x}(\S)-\A_{\x}(\S')  \\
           \A_{\y}(\S)-\A_{\y}(\S')
        \end{array}\right)\right\|_2\leq\epsilon;
        \end{equation*}
        \item \textbf{primal stability:}
    $\E_{\A}\|\A_{\x}(\S)-\A_{\x}(\S')\|_2\leq\epsilon$;  
        \item \textbf{weak stability:}
            \begin{equation*}
                \begin{aligned}
                &\sup_{\xi_{i,l_{i}}}\big[\sup_{\y'\in\yset}\E_{\A}[\frac{1}{m}\sum_{i=1}^m\big(f_i(\A_{\x}(\S),\y';\xi_{i,l_i})-f_i(\A_{\x}(\S'),\y';\xi_{i,l_i})\big)]\\
                    &+\sup_{\x'\in\xset}\E_{\A}[\frac{1}{m}\sum_{i=1}^m \big(f_i(\x',\A_{\y}(\S);\xi_{i,l_i})-f_i(\x',\A_{\y}(\S');\xi_{i,l_i})\big)\big]\leq\epsilon.
                \end{aligned}
            \end{equation*}
    \end{enumerate}
    \vspace{0.1cm}
\end{defi}
\begin{rem}
    Under Assumption \ref{as:Lip:con}, $\epsilon$-argument stability suggests $\sqrt{2}G\epsilon$-weak stability. It is important to highlight that we propose a novel stability measure of primal stability for minimax problems. Although $\epsilon$-argument stability implies $\epsilon$-primal stability, adopting primal stability is more intuitive in some cases, especially when the loss function is strongly concave on the second argument $\y$. Besides, primal stability focuses on the behavior of the minimizer parameter, which better fits into models mainly concerned with the minimizer, such as GAN and robust adversarial training.
\end{rem}

\subsection{Assumptions and Notations}
Below, we introduce several useful assumptions that are common in the literature for analyzing distributed minimax problems (\citet{beznosikovnear,deng2021local}).

\begin{assumption}[\textbf{Lipschitz continuity}]\label{as:Lip:con}
    Let $G>0$, assume for any $\x\in\xset,\y\in\yset$, each local loss function is $G$-Lipschitz continuous on any given sample $\xi_i\in\S_i$, i.e.,
    \[|f_i(\x,\y;\xi_i)-f_i(\x',\y';\xi_i)|\leq G\left\|\left(\begin{array}{c}
        \x-\x'   \\
        \y-\y'
    \end{array}\right)\right\|_2.\]
    It also implies the bounded gradient  $\|\nabla f_i(\x,\y;\xi_i)\|_2\leq G.$
\end{assumption}
    
\begin{assumption}[\textbf{Lipschitz smoothness}]\label{as:Lip:smmoth}
    Let $L>0$, assume for any $\x\in\xset,\y\in\yset$, each local loss function is $L$-Lipschitz smooth on any given sample $\xi_i\in\S_i$, i.e.,
    \[\left\|\left(\!\!\!\begin{array}{c}
        \dx f_i(\x,\y;\xi_i)\!-\!\dx f_i(\x',\y';\xi_i)\\
        \dy f_i(\x,\y;\xi_i)\!-\!\dy f_i(\x',\y';\xi_i)
    \end{array}\!\!\!\!\right)\right\|_2\!\!\leq L \left\|\left(\!\!\!\!\begin{array}{c}
        \x\!-\!\x'   \\
        \y\!-\!\y'
    \end{array}\!\!\!\right)\right\|_2.\]
\end{assumption}

\begin{defi}[\textbf{Convexity-Concavity}]\label{def:convexconcave}
    We say $g(\x,\y)$ is $\mu$-SC-SC, if and only if function $g$ is $\mu$-strongly-convex on argument $\x$ and $\mu$-strongly-concave on argument $\y$, i.e.,
    \begin{align*}
        g(\x'\!,\y)\!\geq\! g(\x,\y)+\!<\!\dx g(\x,\y),\x'\!-\!\x\!>\!+\frac{\mu}{2}\|\x'\!-\x\|^2_2;\\
        g(\x\!,\y')\!\leq\! g(\x,\y)+\!<\!\dy g(\x,\y),\y'\!-\!\y\!>\!-\frac{\mu}{2}\|\y'\!-\y\|^2_2.
    \end{align*}
    And $\mu=0$ refers to C-C case, NC-$\mu$-SC means nonconvex on $\x$ and $\mu$-strongly-concave on $\y$.
\end{defi}

\begin{defi}[\textbf{PL-Condition}]\label{def:pl}
    A differentiable function $g(\x)$ satisfies $\rho$-PL (Polyak-{\L}ojasiewicz) condition \cite{polyak1963gradient,lojasiewicz1963topological,karimi2016linear} when:
    \begin{equation*}
        \frac{1}{2}\|\nabla g(\x)\|_2^2\geq\rho(g(\x)-\inf_{\x'\in\xset}g(\x')).
    \end{equation*}
\end{defi}


\section{Theoretical Results}
In this section, we provide the main stability and generalization results for the generalized \textit{Distributed-SGDA}. 
We provide a generalized connection between algorithmic stability and generalization gap in distributed minimax setting, \textit{not limited to our Distributed-SGDA}, in Sec.\ref{subsec:connection} with proof in Appendix \ref{sec:connection}. By the connection, once we have the stability bounds, the generalization gap can be guaranteed. So what we do under case of SC-SC (Sec.\ref{subsec:scsc}), NC-SC (Sec. \ref{subsec:ncnc}) is to prove the stability bound (Thm.\ref{thm:local:sc-sc-stability}, \ref{thm:primal}) first, then combined with the optimization error proved in Thm. \ref{thm:weakpdempirical} in Appendix \ref{sec:sc-sc} and Thm.\ref{thm:ep:empirical} in Appendix \ref{sec:NC-SC} respectively to balance the optimal population risk (Thm.\ref{thm:sc-sc-population}, \ref{thm:ep:population}). While under NC-NC case (Sec.\ref{subsec:ncsc}), we can only provide the stability bound (Thm. \ref{thm:weak}) which implies corresponding generalization gap through the connection. 
Among the conduction, there exists difficulty dealing with the accumulated local gradients arising from $K$ iterations of local update from $m$ different agents. We address this issue by virtually introducing the global average $\bar{\x}^t_k\triangleq\frac{1}{m}\sum_{i=1}^m\x_{i,k}^t,\bar{\y}^t_k\triangleq\frac{1}{m}\sum_{i=1}^m\y^t_{i,k}$. By Lipschitz property (Assumption \ref{as:Lip:con},\ref{as:Lip:smmoth}), the error can be bounded by a critical consensus term $\Delta^t_k\triangleq\frac{1}{m}\sum_{i=1}^m\left\|\left(\begin{array}{cc}
        \x^t_{i,k}-\bar{\x}^t_k  \\
        \y^t_{i,k}-\bar{\y}^t_k 
    \end{array}\right)\right\|_2$.

\subsection{Connection in Distributed Setting}\label{subsec:connection}
We first provide a comprehensive theorem to demonstrate the connection between algorithmic stability and different measures of generalization gap in the distributed minimax setting. 

\begin{thm}[\textbf{Connection}]\label{thm:connection}
    For a randomized $\epsilon$-stable distributed minimax algorithm $\A$, we have the following generalization gap for model $(\A_{\x}(\S),\!\A_{\y}(\S))$ training on the dataset $\S$,
    \begin{enumerate}[wide=\parindent,label=(\roman*),ref=(\textit{\roman*}),labelindent=0pt,itemsep=2pt,topsep=0pt,parsep=0pt]
    \item\label{connection:weak} under Assumption \ref{as:Lip:con}, the following relationship holds: $\epsilon$-\textbf{argument stability} $\rightarrow$ $\sqrt{2}G\epsilon$-\textbf{weak stability} $\rightarrow$ \textbf{weak PD generalization gap}:        $\zeta^{\mathrm{w}}_{\mathrm{gen}}(\A,\S)\leq\sqrt{2}G\epsilon$;
    \item\label{connection:ex} under Assumptions \ref{as:Lip:con} and \ref{as:Lip:smmoth}, when the local functions $f_i(\x,\cdot)$ and $F(\x,\cdot)$ are $\mu$-strongly-concave, we have $\epsilon$-\textbf{primal stability} $\rightarrow$ \textbf{excess primal generalization gap}:
        $$\zeta^{\mathrm{ep}}_{\mathrm{gen}}(\A,\S)\leq G\sqrt{1+\frac{L^2}{\mu^2}}\epsilon+\frac{4G^2}{\mu mn};$$
    \item\label{connection:strong} under Assumptions \ref{as:Lip:con} and \ref{as:Lip:smmoth}, when $F(\cdot,\cdot)$ is $\mu$-SC-SC, we have $\epsilon$-\textbf{argument stability} $\rightarrow$ \textbf{strong PD generalization gap}: $\zeta^{\mathrm{s}}_{\mathrm{gen}}(\A,\S)\leq G\sqrt{2+\frac{2L^2}{\mu^2}}\epsilon$.
    \end{enumerate}
\end{thm}

\begin{rem}
For \ref{connection:weak}, weak stability is sufficient to indicate the weak PD generalization gap, which can be applicable in the NC-NC case. Conclusion in \ref{connection:ex} is special due to the extra isolated term $\frac{4G^2}{\mu mn}$, which implies that primal stability combined with a considerable amount of $mn$ can guarantee the generalization gap. It can be attributed to the fact that the excess primal generalization gap also assesses the influence of argument $\y$ not encompassed in primal stability but can be bounded by leveraging the property of strong concavity.
\end{rem}

\subsection{(Strongly) Convex-(Strongly) Concave Case}\label{subsec:scsc}

\begin{thm}[\textbf{Argument stability}]\label{thm:local:sc-sc-stability}
For Distributed-SGDA $\A(T,K,\W)$, under Assumptions \ref{as:mixing}, \ref{as:Lip:con}, and \ref{as:Lip:smmoth},  local function $f_i$ being $\mu$-SC-SC, we have its $\epsilon$-argument stability:
    \begin{align*}
    \epsilon\leq2\sum_{t=1}^T\prod_{t'=t+1}^T\prod_{k_1=0}^{K-1}(1-\eta^{t'}_{k_1}\frac{L\mu}{L+\mu})\cdot\sum_{k=0}^{K-1}\eta^t_k(L\E_{\A}[\Delta^t_k]+\frac{G}{n})\prod_{k'=k+1}^{K-1}(1-\eta^t_{k'}\frac{L\mu}{L+\mu}).
    \end{align*}
\begin{enumerate}[wide=\parindent,label=(\roman*),ref=(\textit{\roman*}),labelindent=0pt,itemsep=2pt,topsep=0pt,parsep=0pt]
        \item\label{argument:fixed} Choosing a fixed learning rate $\eta$ yields the stability:
            \begin{equation*}
            \epsilon\leq\frac{2G(L+\mu)}{L\mu}\Big(\eta L\sqrt{\frac{1}{1-\lambdabm}}\frac{K^2}{2}+\eta L\sqrt{\frac{2\lambdabm}{(1-\lambdabm)(1-\lambdabm^2)}}K^2+\frac{K}{n}\Big);
            \end{equation*}
        \item\label{argument:decaying} For decaying learning rates $\eta^t_k\!=\!\frac{1}{(k+1)^\alpha t}$ with $\frac{1}{2}\!<\!\alpha\!<\!1$, and $C_{\lambdabm}\triangleq\frac{(\alpha/e)^\alpha}{\lambdabm(-\ln{\lambdabm})^\alpha}+\frac{2}{e\lambdabm(-\ln{\lambdabm})}+\frac{2^\alpha}{\lambdabm(-\ln{\lambdabm})}$, it holds:
        \begin{small}
            \begin{align*}
                \epsilon\leq\frac{1}{T\!+\!1}\bigg(\frac{(K^{\frac{3}{2}-\alpha}\!+\!\frac{3}{2}\!-\!\alpha)\sqrt{\frac{\frac{2\alpha}{2\alpha-1}}{1-\lambdabm}}\frac{2GL(1-\alpha)}{(\frac{3}{2}-\alpha)}}{\frac{L\mu}{L+\mu}(K^{1-\alpha}+1-\alpha)+\alpha-1}\!+\!\frac{(K^{1-\alpha}\!+\!1\!-\!\alpha)2GL\sqrt{(C_{\lambdabm^2}+\frac{C_{\lambdabm}}{1-\lambdabm})\frac{2\alpha}{2\alpha-1}K}}{\frac{L\mu}{L+\mu}(K^{1-\alpha}+1-\alpha)+\alpha-1}\bigg)\!+\!\frac{G}{n}\frac{L\!+\!\mu}{L\mu}.
            \end{align*}
        \end{small}
    \end{enumerate}
\end{thm}

\begin{rem}\label{remark:lambda}
Define $\lambdabm$-relevant constants $\mathcal{O}(\sqrt{\lambdabm_1})$ and $\mathcal{O}(\sqrt{\lambdabm_2})$ that $\lambdabm_1\sim\frac{1}{1-\lambdabm}$, $\lambdabm_2\sim C_{\lambdabm^2}+\frac{C_{\lambdabm}}{1-\lambdabm}$ (or $\lambdabm_2\sim\frac{2\lambdabm}{(1-\lambdabm)(1-\lambdabm^2)}$) respectively where $\lambdabm_1=1$ while $\lambdabm_2=0$ when $\lambdabm=0$ (see Def.\ref{def:crucial}) which represents the fully-connected graph. Then, for fixed learning rates \ref{argument:fixed}, argument stability is bounded by $\mathcal{O}((\sqrt{\lambdabm_1}+\sqrt{\lambdabm_2})\eta K^2+\frac{K}{n})$ which can achieve optimal of $\mathcal{O}(\frac{(\sqrt{\lambdabm_1}+\sqrt{\lambdabm_2})K}{T}+\frac{K}{n})$ when choosing $\eta\sim\frac{1}{TK}$. For decaying learning rates \ref{argument:decaying}, argument stability is bounded by $\mathcal{O}((\sqrt{\lambdabm_1}+\sqrt{\lambdabm_2})\frac{K^{\frac{1}{2}}}{T}+\frac{1}{n})$.
   
\end{rem}

\noindent\textbf{\textit{Local-DSGDA}} is just vanilla $\A(T,K,\W)$, and the analysis is the same as above;

\vspace{1.5pt}
\noindent\textbf{\textit{Local-SGDA}} ($\W=\Pm$, $\lambdabm=0$ and $\lambdabm_1=1,\lambdabm_2=0$): we have the argument stability of $\mathcal{O}(\eta K^2+\frac{K}{n})$ under fixed learning rates and $\mathcal{O}(\frac{K^{\frac{1}{2}}}{T}+\frac{1}{n})$ under decaying learning rates;

\vspace{1.5pt}
\noindent\textbf{\textit{Decentralized-SGDA}}($K=1$): the argument stability is bounded by $\mathcal{O}((\sqrt{\lambdabm_1}+\sqrt{\lambdabm_2})\eta+\frac{1}{n})$ under fixed learning rates and bounded by $\mathcal{O}(\frac{\sqrt{\lambdabm_1}+\sqrt{\lambdabm_2}}{T}+\frac{1}{n})$, which matches the result in \citep{zhu2023stability} of $\mathcal{O}(\frac{\eta}{1-\lambdabm}+\frac{1}{n})$.

\vspace{1.5pt}
According to Theorem \ref{thm:connection}, \textbf{\textit{weak and strong PD generalization gap}} can also be guaranteed by $\sqrt{2}G\epsilon$ and $G\sqrt{2\!+\!\frac{2L^2}{\mu^2}}\epsilon$, where the latter one further requires $\mu$-strong concavity on the function $F(\x,\cdot)$. So the discussion about the generalization gap is equivalent to the above discussions.

\begin{corol}
    When $\mu=0$, i.e., each local function $f_i$ is C-C, under Assumption \ref{as:mixing}-\ref{as:Lip:smmoth}, we have the $\epsilon$-argument stability $\epsilon\leq2\sum_{t=1}^T\sum_{k=0}^{K-1}\eta^t_k(L\E_{\A}[\Delta^t_k]+\frac{G}{n})$.
\end{corol}

\begin{rem}\label{cc}
    It is interesting to see that argument stability can be bounded only for fixed learning rates with $\mathcal{O}((\sqrt{\lambdabm_1}+\sqrt{\lambdabm_2})\eta^2TK^2+\frac{\eta TK}{n})$, slightly looser than results in SC-SC. While choosing $\eta\sim\frac{1}{TK}$ it can achieve optimal of $\mathcal{O}(\frac{\sqrt{\lambdabm_1}+\sqrt{\lambdabm_2}}{T}+\frac{1}{n})$. Other results align with what we discussed about SC-SC case.
\end{rem}

Next, we present the weak PD population risk, which can be divided into generalization gap and empirical risk:
    \begin{equation*}
        \Delta^{\mathrm{w}}(\widetilde{\x}^T_K,\widetilde{\y}^T_K)=\underbrace{\Delta^{\mathrm{w}}(\widetilde{\x}^T_K,\widetilde{\y}^T_K)-\Delta^{\mathrm{w}}_{\S}(\widetilde{\x}^T_K,\widetilde{\y}^T_K)}_{\text{weak PD generalization gap}}+\underbrace{\Delta^{\mathrm{w}}_{\S}(\widetilde{\x}^T_K,\widetilde{\y}^T_K)}_{\text{weak PD empirical risk}},
\vspace{-2pt}
\end{equation*}
where we use the averaged model defined as $(\widetilde{\x}^T_K,\widetilde{\y}^T_K\!)\!\triangleq\!(\frac{1}{TK}\!\sum_{t=1}^T\!\sum_{k=0}^{K-1}\!\bar{\x}^t_k,\frac{1}{TK}\!\sum_{t=1}^T\!\sum_{k=0}^{K-1}\!\bar{\y}^t_k)$ instead of the last iterate since the last-iterate optimization is hard to solve, while the generalization gap holds for any ($t$,$k$)-th iteration.

\begin{thm}[\textbf{Weak PD population risk}]\label{thm:sc-sc-population}
    For Distributed-SGDA $\A(T,K,\W)$, under Assumptions \ref{as:mixing}-\ref{as:Lip:smmoth}, when each local function $f_i$ is $\mu$-SC-SC on any given sample, further assuming $\sup_{\x\in\xset}\|\x\|_2\leq B_{\x},\sup_{\y\in\yset}\|\y\|_2\leq B_{\y}$, 
    \begin{enumerate}[wide=\parindent,label=(\roman*),ref=(\textit{\roman*}),labelindent=0pt,itemsep=2pt,topsep=0pt,parsep=0pt]
    \vspace{2pt}
        \item\label{wp:fixed} for fixed learning rates,  weak PD population risk holds:
            \begin{align*}
                \Delta^{\!\mathrm{w}}(\widetilde{\x}^T_K,\widetilde{\y}^T_K)\!&\leq\!\underbrace{\frac{(1\!-\!\eta)(B_{\x}^2\!+\!B_{\y}^2)}{2\eta TK}\!+\!\frac{2G(B_{\x}\!+\!B_{\y})}{\sqrt{TK}}\!+\!\frac{\eta GKL(B_{\x}\!+\!B_{\y})(\mathcal{O}(\!\sqrt{\lambdabm_1})\!+\!\mathcal{O}(\!\sqrt{\lambdabm_2}))}{2}\!+\!\eta G^2}_{\text{optimization error}}\\
                &+\underbrace{\frac{2\sqrt{2}G^2(L+\mu)}{L\mu}\bigg(\frac{\mathcal{O}(\sqrt{\lambdabm_1})L}{2}\eta K^2+\mathcal{O}(\sqrt{\lambdabm_2})L\eta K^2+\frac{K}{n}\bigg);}_{\text{generalization gap}}
            \end{align*}
        \item\label{wp:decaying} for decaying learning rates $\eta^t_k=\frac{1}{(k+1)^\alpha t}$ with $\frac{1}{2}<\alpha<1$, weak PD population risk holds:

        \vspace{-0.5cm}
            \begin{align*}
                \Delta^{\!\mathrm{w}}(\widetilde{\x}^T_K,\widetilde{\y}^T_K)&\leq\underbrace{\frac{2G(B_{\x}\!+\!B_{\y})}{\sqrt{TK}}\!+\frac{1\!+\!\ln{T}}{T}\!\bigg(\!\frac{G^2}{(1\!-\!\alpha)K^\alpha}\!+\!(B_{\x}\!+\!B_{\y})GL(\mathcal{O}(\!\sqrt{\lambdabm_1})\!+\!\mathcal{O}(\!\sqrt{\lambdabm_2}))K^{\frac{1}{2}}\!\bigg)}_{\text{optimization error}}\\
                &+\underbrace{\frac{\sqrt{2}G(L\!+\!\mu)}{L\mu}\frac{K^\frac{1}{2}}{T\!+\!1}(\mathcal{O}(\sqrt{\lambdabm_1})\!+\!GL\mathcal{O}(\sqrt{\lambdabm_2}))\!+\!\frac{\sqrt{2}G^2}{n}\frac{L\!+\!\mu}{L\mu}.}_{\text{generalization gap}}
            \end{align*}
    \end{enumerate}
\end{thm}

\begin{rem}
For fixed learning rates \ref{wp:fixed}, the weak PD population risk is bounded by $\mathcal{O}((\sqrt{\lambdabm_1}\!+\!\sqrt{\lambdabm_2})\eta K^2\!+\!\frac{K}{n}\!+\!\frac{B^2_{\x}\!+\!B^2_{\y}}{\eta TK}\!+\!\frac{B_{\x}\!+\!B_{\y}}{\sqrt{TK}})$, which attains optimal of $\mathcal{O}((\sqrt{\lambdabm_1}\!+\!\sqrt{\lambdabm_2})\frac{K}{\sqrt{T}}\!+\!\frac{K}{n})$ when we choose $\eta\sim\frac{1}{\sqrt{T}K}$. For decaying learning rates \ref{wp:decaying}, the population risk is bounded by $\widetilde{\mathcal{O}}((\sqrt{\lambdabm_1}\!+\!\sqrt{\lambdabm_2})\frac{K^{\frac{1}{2}}}{T}+\frac{B_{\x}+B_{\y}}{\sqrt{TK}}+\frac{1}{n})$ which can obtain optimal when $K=\sqrt{T}$.
\end{rem}

There exists a \textbf{\textit{trade-off}} between the generalization gap and optimization error with fixed learning rates. We derive that $\eta\sim\frac{1}{TK}$ optimizes the generalization gap while this type of learning rate leads to the divergence of optimization error.

\noindent\textbf{\textit{Local-SGDA}}($\W=\Pm$,$\lambdabm_1=1$,$\lambdabm_2=0$): the weak PD population risk is bounded by $\mathcal{O}(\eta K^2+\frac{K}{n}+\frac{1}{\eta TK}+\frac{1}{\sqrt{TK}})$ under fixed learning rates and $\widetilde{\mathcal{O}}(\frac{K^{\frac{1}{2}}}{T}+\frac{1}{\sqrt{TK}}+\frac{1}{n})$ under decaying learning rates. And it matches the convergence results of $\widetilde{\mathcal{O}}(\frac{1}{\sqrt{NT'}})$ in \citep{deng2021local} under SC-SC homogeneous setting, where $N=mn$, $T'=TK$ and the square root is due to the measure of quadratic term there.

\vspace{1.5pt}
\noindent\textbf{\textit{Decentralized-SGDA}}($K=1$): under fixed learning rates, the population risk is bounded by $\mathcal{O}((\sqrt{\lambdabm_1}\!+\!\sqrt{\lambdabm_2})\eta+\frac{1}{\eta T}+\frac{1}{\sqrt{T}}+\frac{1}{n})$ which achieves optimal of $\mathcal{O}(\frac{(\sqrt{\lambdabm_1}\!+\!\sqrt{\lambdabm_2})}{\sqrt{T}}+\frac{1}{n})$ when $\eta\sim\frac{1}{\sqrt{T}}$ by the above theorem. Under decaying learning rates, we have the weak PD population risk bounded by $\widetilde{\mathcal{O}}(\frac{(\sqrt{\lambdabm_1}\!+\!\sqrt{\lambdabm_2})}{\sqrt{T}}+\frac{1}{n})$. The results matches the result of $\mathcal{O}(\frac{1}{n}+\frac{1}{(1-\lambdabm)\sqrt{T}})$ in \citep{zhu2023stability}.

\begin{corol}
    When $\mu=0$, i.e., each local function $f_i$ is C-C, under Assumption \ref{as:mixing}-\ref{as:Lip:smmoth}, further assuming $\sup_{\x\in\xset}\|\x\|_2\leq B_{\x},\sup_{\y\in\yset}\|\y\|_2\leq B_{\y}$, weak PD population risk $\Delta^{\!\mathrm{w}}(\widetilde{\x}^T_K,\widetilde{\y}^T_K)$ is bounded by $\mathcal{O}((\sqrt{\lambdabm_1}\!+\!\sqrt{\lambdabm_2})\eta^2TK^2+\frac{\eta TK}{n}+\frac{B^2_{\x}+B^2_{\y}}{\eta TK}+\frac{B_{\x}+B_{\y}}{\sqrt{TK}})$.
\end{corol}

\begin{rem}
    Since the weak PD generalization gap can only be bounded for fixed learning rates under C-C (Remark. \ref{cc}), the corresponding population bounded can be barely guaranteed for fixed learning rates. Notice that if we continue choose $\eta\sim\frac{1}{TK}$, the optimization error will diverge. Thus, choosing $\eta\sim\frac{1}{T\sqrt{K}}$ can help the population risk achieve the optimal.
\end{rem}


\subsection{Nonconvex-Strongly-Concave Case}\label{subsec:ncsc}


\begin{thm}[\textbf{Primal stability}]\label{thm:primal}
    For Distributed-SGDA $\A(T,K,\W)$, under Assumption \ref{as:Lip:con}, when each local function $f_i$ satisfies $\rho$PL-$\mu$SC condition on any given sample, the primal stability is bounded by $\E_{\A}\|\A_{\x}(\S)-\A_{\x}(\S')\|\leq\frac{2G}{n}\sqrt{\frac{1}{4\rho}+\frac{1}{\mu}}+2\sqrt{\frac{2}{\rho}\Delta^{\!\mathrm{ep}}_{\S}}$,
    where $\Delta^{\!\mathrm{ep}}_{\S}$ is defined as $\Delta^{\!\mathrm{ep}}_{\S}:=\sup_{\S}\Delta^{\!\mathrm{ep}}_{\S}(\A,\S)$.
    
\end{thm}

\begin{rem}
    The result of primal stability is different from other stability measures since it can be bounded by excess primal empirical risk. It is reasonable that the optimization performance can influence the stability as well as the generalization behavior. On the other hand, we use an optimization algorithm to minimize the empirical risk which can be a rather small scale. This phenomenon has been investigated in the literature \cite{lei2023stability}.
\end{rem} 

According to Theorem \ref{thm:connection}\ref{connection:weak}, the \textbf{\textit{excess primal generalization gap}} can also be guaranteed by $\zeta^{\mathrm{ep}}_{\mathrm{gen}}\leq\mathcal{O}(\frac{1}{mn}+\sqrt{\Delta^{\!\mathrm{ep}}_{\S}})$ under further requirement on the strong concavity of the function $F(\x,\cdot)$, where the convergence of the excess primal empirical risk is presented in Theorem \ref{thm:ep:empirical} in Appendix \ref{sec:NC-NC}. In addition, the excess primal population risk which can be decomposed into generalization gap and empirical risk in the form of the last iterate:
    \begin{align*}
        \Delta^{\!\mathrm{ep}}({\x}^T)&=\underbrace{\Delta^{\!\mathrm{ep}}({\x}^T)-\Delta^{\!\mathrm{w}}_{\S}({\x}^T)}_{\text{excess primal generalization gap}}+\underbrace{\Delta^{\!\mathrm{ep}}_{\S}({\x}^T)}_{\text{excess primal empirical risk}}\\
        &\leq\mathcal{O}(\frac{1}{mn}+\sqrt{\Delta^{\!\mathrm{ep}}_{\S}})+\mathcal{O}(\Delta^{\!\mathrm{ep}}_{\S}). 
    \end{align*}
It can be observed that $\mathcal{O}(\sqrt{\Delta^{\!\mathrm{ep}}_{\S}})$ dominates the error bound and therefore the excess primal generalization gap and population risk can be upper bounded by the same rate.

\begin{thm}[\textbf{Excess primal generalization gap / population risk}]\label{thm:ep:population}
    For Distributed-SGDA $\A(T,K,\W)$, under Assumptions \ref{as:mixing}-\ref{as:Lip:smmoth}, when each local function $f_i$ satisfies $\rho$PL-$\mu$SC condition and further requiring function $F(\x,\cdot)$ being $\mu$-SC, choosing decaying learning rates $\eta^t_k=\frac{4}{\mu(k+1)^\alpha t}$ with $\frac{1}{2}\!<\!\alpha\!<1$, the excess primal population risk holds: 
    \begin{align*}
        \Delta^{\!\mathrm{ep}}\!(\!\A_{\x}(\S))\!&\leq\!\!\frac{2G^2}{n}\!\sqrt{\!1\!+\!\kappa^2}\!\sqrt{\!\frac{1}{4\rho}\!\!+\!\frac{1}{\mu}}\!+\!\!\frac{4G^2}{\mu mn}\!+\!\frac{2\sqrt{2}G^2}{\sqrt{T}}\!\sqrt{\!\frac{1}{\rho}\!+\!\kappa^2}\!\bigg(\!\!\sqrt{\frac{64B_{\y}\kappa^3}{\mu^2}}\!\sqrt{\!\mathcal{O}(\!\sqrt{\lambdabm_1})\!+\!\mathcal{O}(\!\sqrt{\lambdabm_2})}K^{\frac{1}{4}}\\
        &+\!(\mathcal{O}(\sqrt{\lambdabm_1})\!+\!\mathcal{O}(\sqrt{\lambdabm_2}))\sqrt{\frac{128\kappa^2}{\mu^3}}K^{\frac{1+\alpha}{2}}+\sqrt{\frac{32\rho\kappa^2}{\mu^2}}\frac{1}{K^{\frac{1-\alpha}{2}}}\bigg).
    \end{align*}
    \end{thm}
 
 \begin{rem}
     The excess primal generalization gap $\zeta^{\!\mathrm{ep}}$ and the excess primal population risk $\Delta^{\!\mathrm{ep}}$ can be bounded by $\mathcal{O}(\frac{1}{mn}+\frac{(\sqrt{\lambdabm_1}+\sqrt{\lambdabm_2})K^{\frac{1+\alpha}{2}}}{\sqrt{T}})$. From the discussion before, the optimal generalization gap can not guarantee the optimal optimization error since $\sqrt{\Delta^{\!\mathrm{ep}}_{\S}}$ is dominant.
\end{rem}

\vspace{-0.3cm}
\noindent\textit{\textbf{Local-SGDA}}($\W=\Pm$,$\lambdabm_1=1,\lambdabm_2=0$): $\zeta^{\!\mathrm{ep}}_{\mathrm{gen}}$ and $\Delta^{\!\mathrm{ep}}$ can be bounded by $\mathcal{O}(\frac{1}{mn}+\frac{K^{\frac{1+\alpha}{2}}}{\sqrt{T}})$. Comparable with the convergence results of $\mathcal{O}(\frac{1}{(NT')^{1/3}})$ in \citep{deng2021local} under NC-SC heterogeneous setting.

\vspace{1.8pt}
\noindent\textit{\textbf{Decentralized-SGDA}} ($K\!=\!1$): we have the excess primal generalization gap and population risk for D-SGDA as $\mathcal{O}(\frac{1}{mn}\!+\!\frac{(\sqrt{\lambdabm_1}+\sqrt{\lambdabm_2})}{\sqrt{T}})$, which is new under NC-SC setting.


\subsection{Nonconvex-Nonconcave Case}\label{subsec:ncnc}

\begin{thm}[\textbf{Weak Stability and Weak PD Generalization Gap}]\label{thm:weak}
    For Distributed-SGDA $\A(T,K,\W)$, under Assumptions \ref{as:mixing}-\ref{as:Lip:smmoth}, we say $\A$ is $\epsilon$-weakly stable considering:
    \begin{enumerate}[wide=\parindent,label=(\roman*),ref=(\textit{\roman*}),labelindent=0pt,itemsep=2pt,topsep=0pt,parsep=0pt]
        \vspace{2.5pt}
        \item\label{weak:fixed} fixed learning rates, $\epsilon\leq2\sqrt{2}\eta G^2\bigg(\eta KL\big(\mathcal{O}(\sqrt{\lambdabm_1})+\mathcal{O}(\sqrt{\lambdabm_2})\big)+\frac{1}{n}\bigg)K(T+1)$;

        \item\label{weak:decaying} for decaying learning rates of $\eta^t_k=\frac{1}{L(k+1)t}$,
        \vspace{-3pt}
        \begin{equation*}
            \epsilon\leq2\big(\frac{2G}{L}\frac{T^{1+\ln{K}}}{1+\ln{K}}(\frac{K}{n}+\mathcal{O}(\sqrt{\lambdabm_1})K+\mathcal{O}(\sqrt{\lambdabm_2})K^{\frac{3}{2}})\big)^{\frac{1}{5}}\frac{(2Mm)^{\frac{4}{5}}K^{\frac{2}{5}}}{n^{\frac{4}{5}}}.
        \end{equation*}       
    \end{enumerate}
\end{thm}
\vspace{0.15cm}

\begin{rem}
    
For fixed learning rates \ref{weak:fixed}, weak stability is bounded by $\mathcal{O}(\eta^2 KT(\sqrt{\lambdabm_1}+\sqrt{\lambdabm_2})+\frac{\eta KT}{n}))$. When we choose $\eta\sim\frac{1}{KT}$, it can reach $\mathcal{O}(\frac{\sqrt{\lambdabm_1}+\sqrt{\lambdabm_2}}{KT}+\frac{1}{n})$. For decaying learning rates \ref{weak:decaying}, we have the weak stability bounded by $\widetilde{\mathcal{O}}((\frac{1}{n}+\sqrt{\lambdabm_1}+\sqrt{\lambdabm_2}K^{\frac{1}{2}})^{\frac{1}{5}}(\frac{m}{n})^{\frac{4}{5}}T^{\frac{1}{5}}K^{\frac{3}{5}})$.
\end{rem}

\noindent\textit{\textbf{Local-SGDA}}($\W=\Pm$,$\lambdabm_1=1$,$\lambdabm_2=0$): the weak stability is bounded by $\mathcal{O}(\eta^2KT+\frac{\eta KT}{n})$ for fixed learning rates and $\widetilde{\mathcal{O}}(\frac{m^{\frac{4}{5}}}{n}+\frac{m^{\frac{4}{5}}}{n^{\frac{4}{5}}}T^{\frac{1}{5}}K^{\frac{3}{5}})$ for decaying learning rates.

\noindent\textit{\textbf{Decentralized-SGDA}}($K=1$): we have the weak stability of $\mathcal{O}(\eta^2T(\sqrt{\lambdabm_1}+\sqrt{\lambdabm_2})+\frac{\eta T}{n}))$ under fixed learning rates, which is consistent with the result in \citet{zhu2023stability} of $\mathcal{O}(\frac{\eta T}{n}+\frac{\eta^2T}{1-\lambdabm})$. While the weak stability is bounded by $\widetilde{\mathcal{O}}((\frac{1}{n}+\sqrt{\lambdabm_1}+\sqrt{\lambdabm_2})^{\frac{1}{5}}(\frac{m}{n})^{\frac{4}{5}}T^{\frac{1}{5}})$.

\vspace{3pt}
By Theorem \ref{thm:connection}\ref{connection:weak}, the \textbf{\textit{weak PD generalization gap}} can be guaranteed by weak stability. Since there exist challenges in the optimization analysis of the NC-NC minimax problems, we do not provide the population risk under NC-NC case.

\captionsetup[figure]{labelfont={bf}}

\begin{figure*}[t]
\centering
\includegraphics[width=0.22\textwidth]{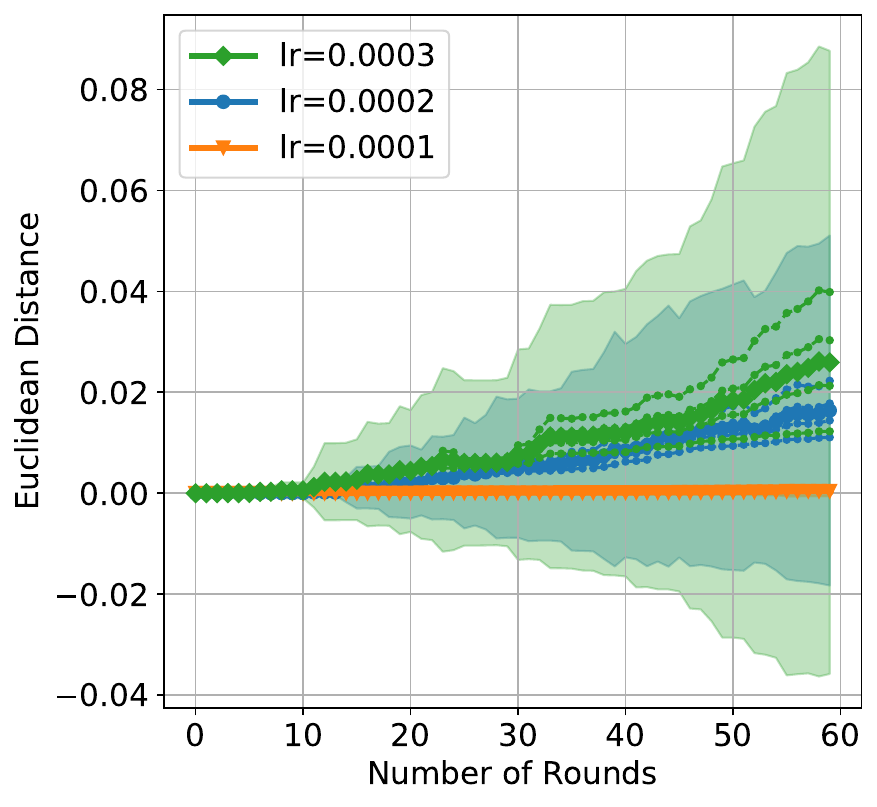}
\includegraphics[width=0.22\textwidth]{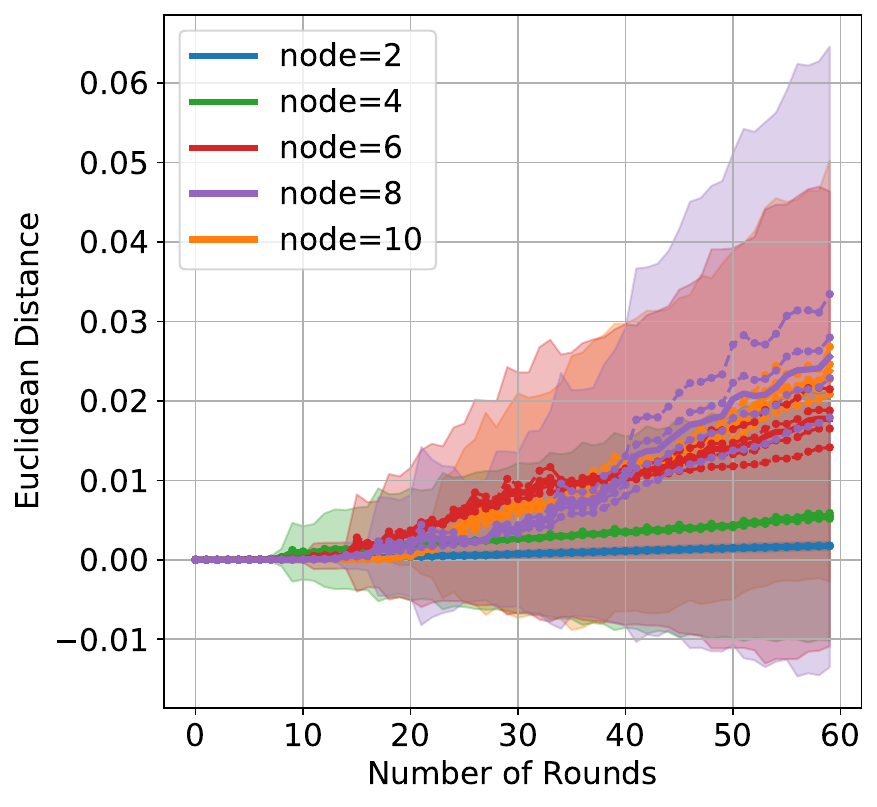}
\includegraphics[width=0.215\textwidth]{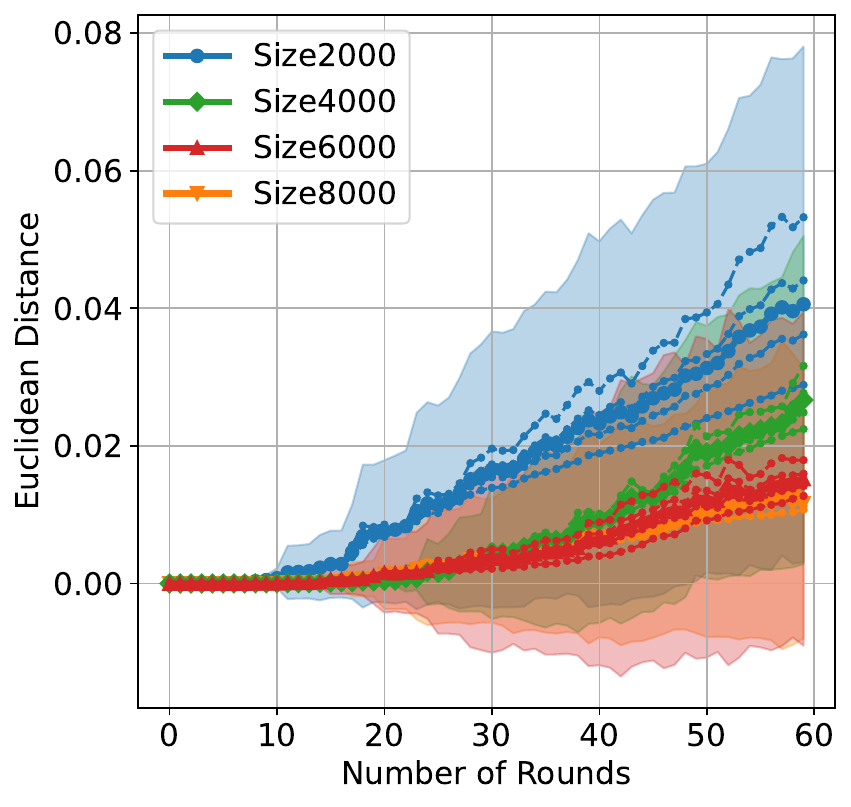}
\includegraphics[width=0.215\textwidth]{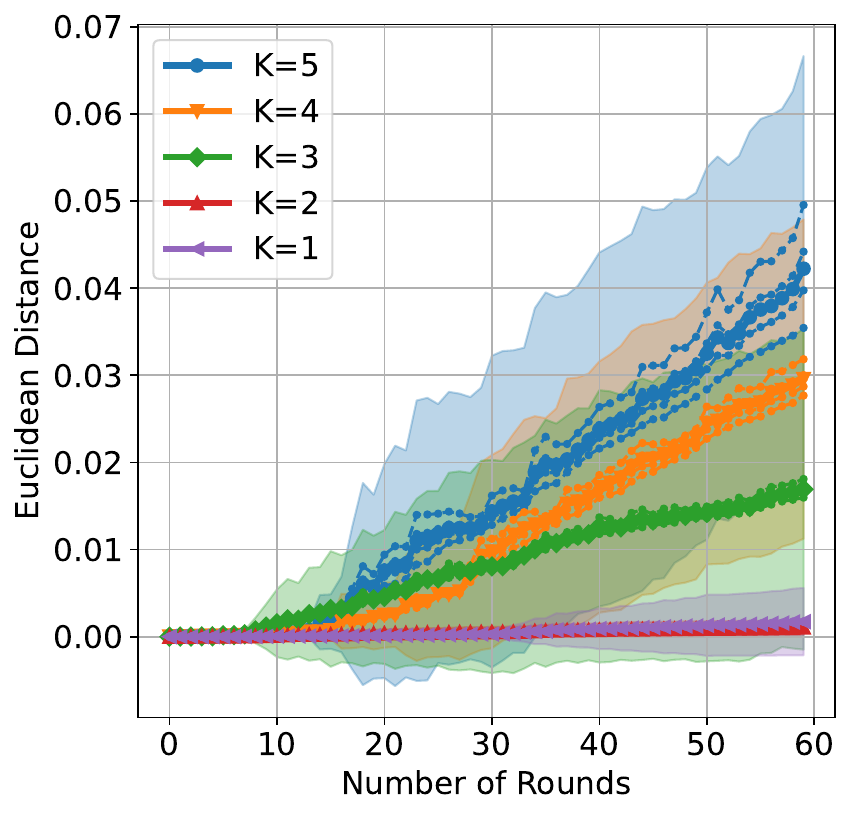} 
\vspace{-0.5em}
\includegraphics[width=0.22\textwidth]{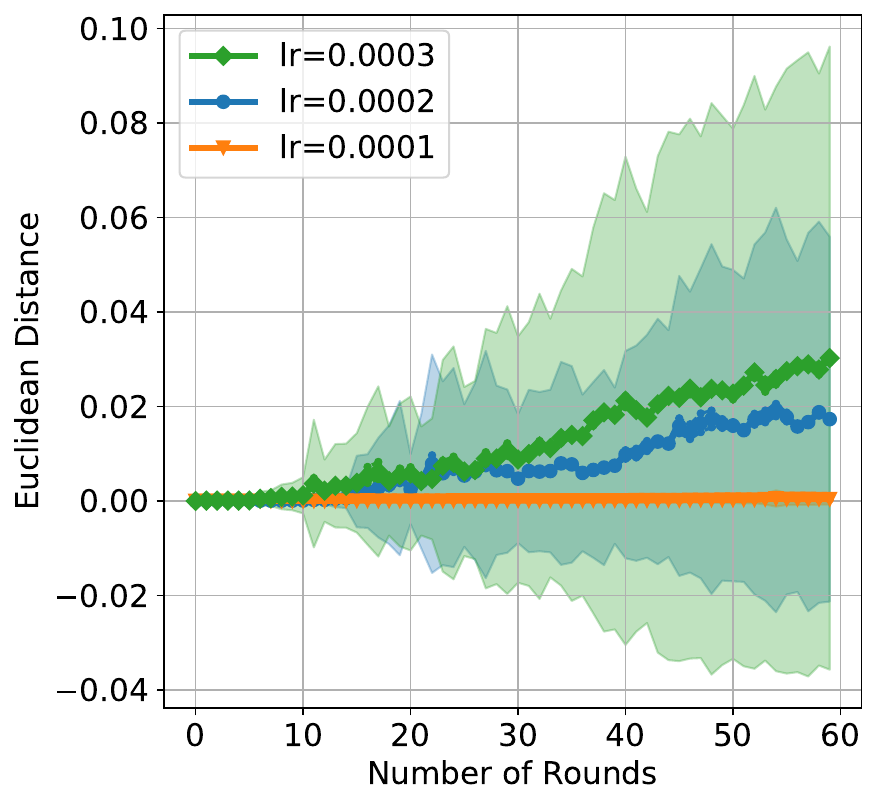}
\includegraphics[width=0.22\textwidth]{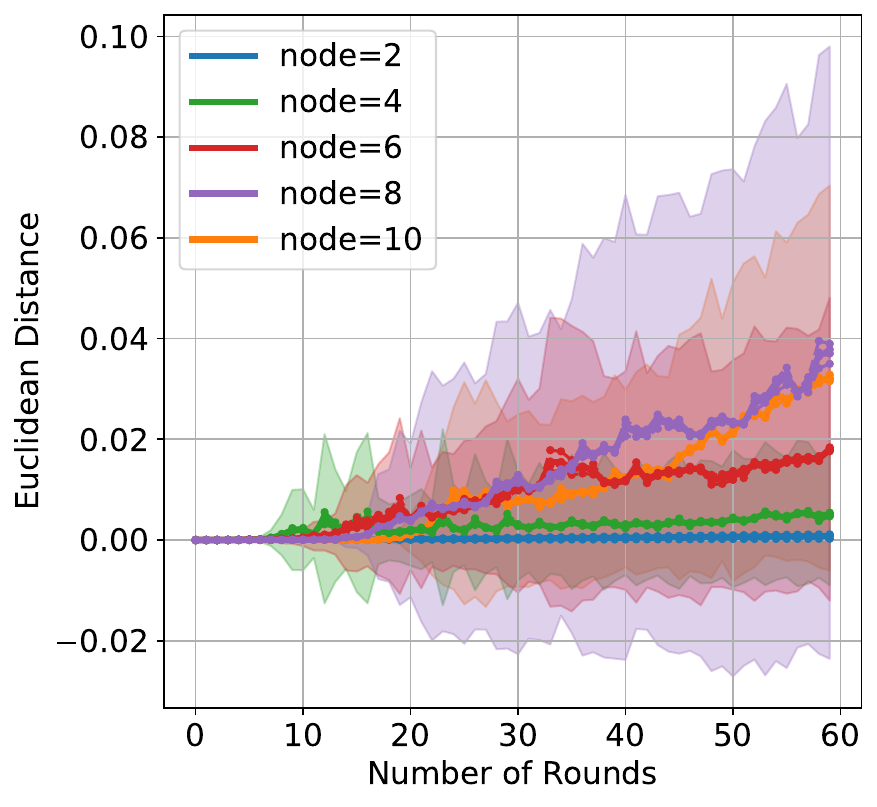}
\includegraphics[width=0.22\textwidth]{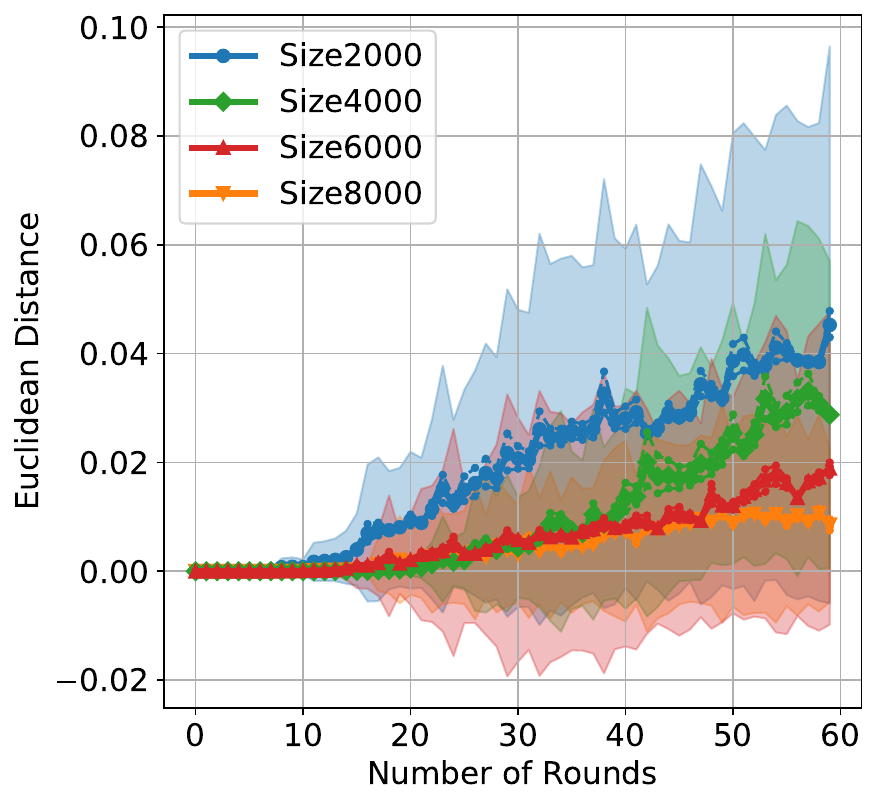}
\includegraphics[width=0.215\textwidth]{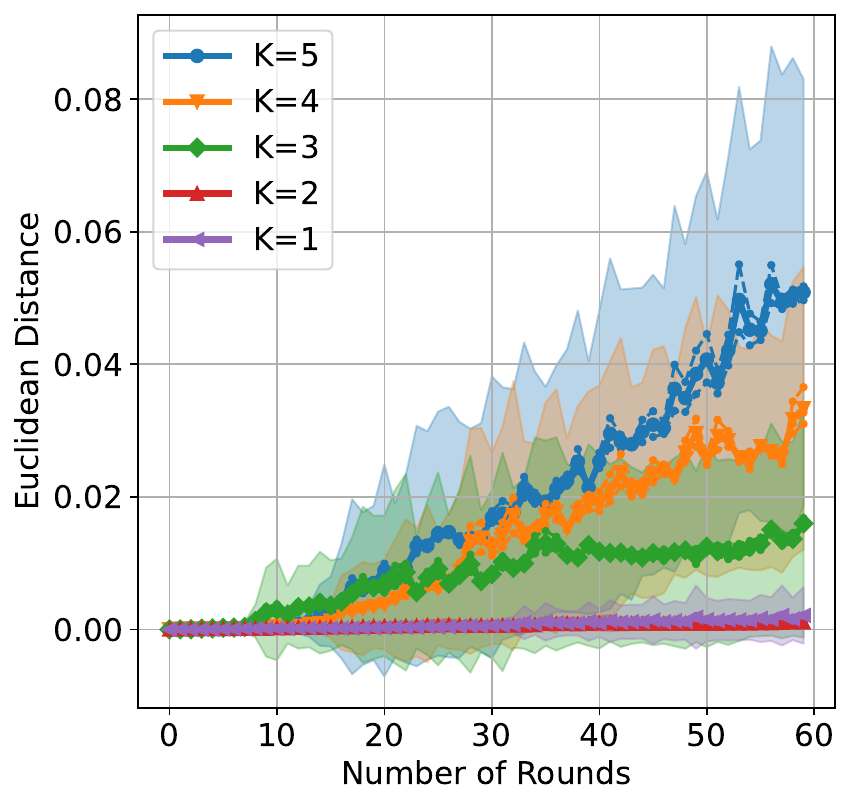}
\vspace{-0.3em}
\caption{The first row shows the stability of the generator model using Local-SGDA method, the second row shows the discriminator model. From left to right, the figures correspond to the varying learning rates, the number of nodes, the local dataset size, and the number of local steps. Each layer is independently assessed and shown as the dashed lines.}
\label{fig:local:discriminator} 
\centering
\includegraphics[width=0.193\textwidth]{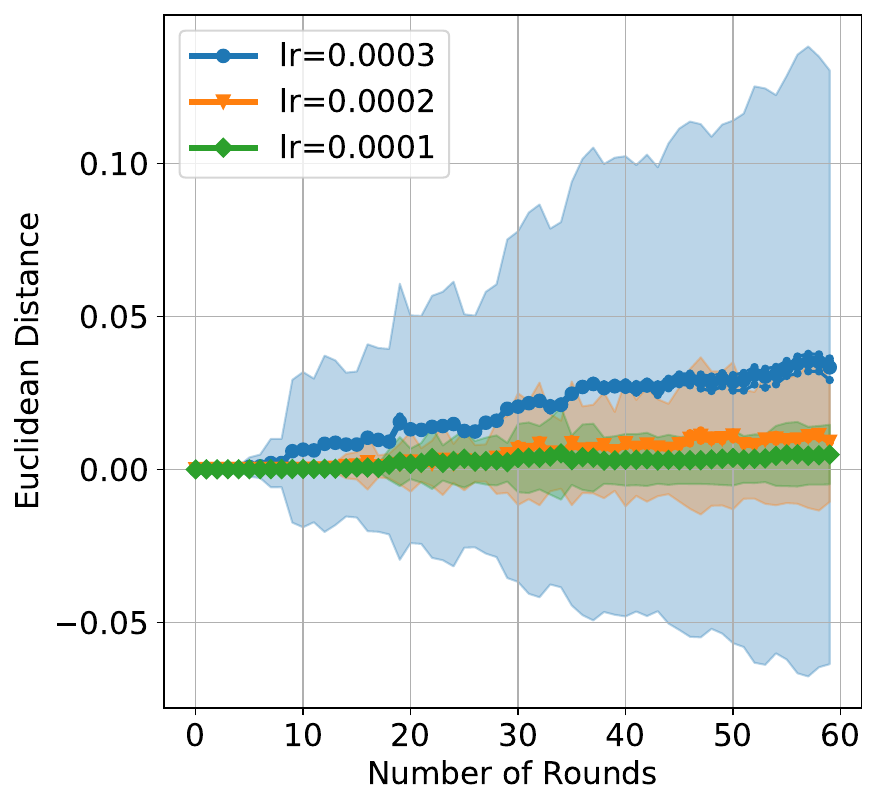}
\includegraphics[width=0.193\textwidth]{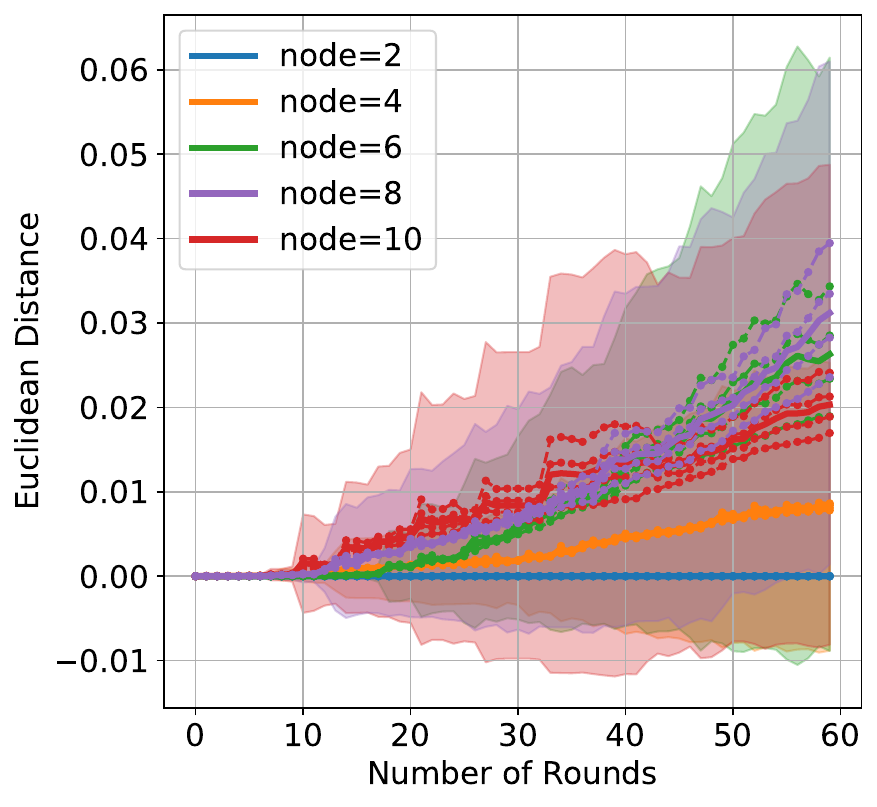}
\includegraphics[width=0.189\textwidth]{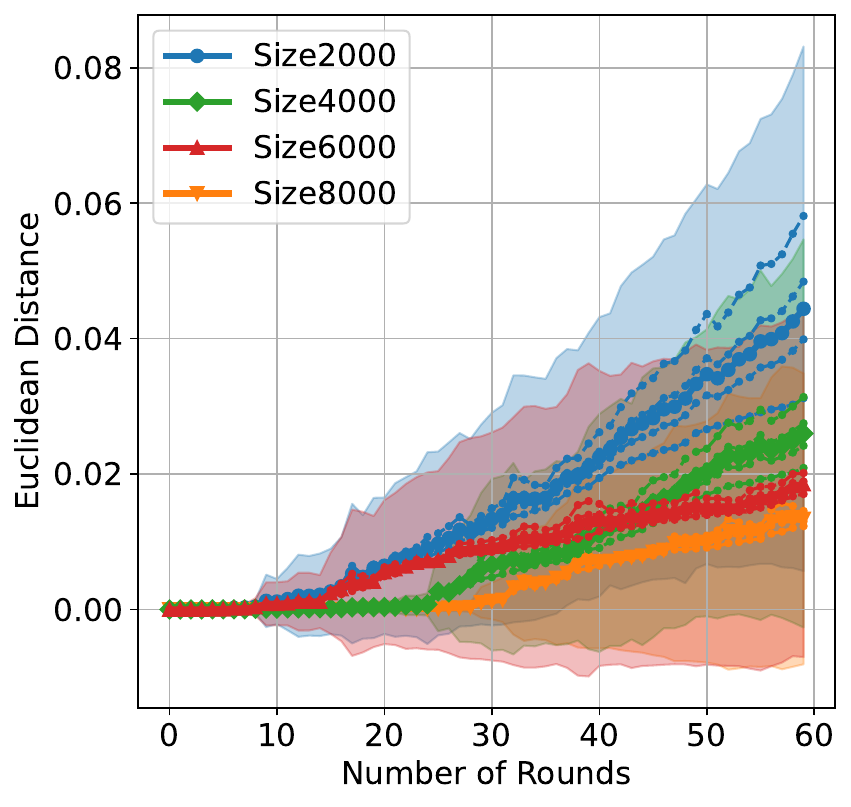}
\includegraphics[width=0.194\textwidth]{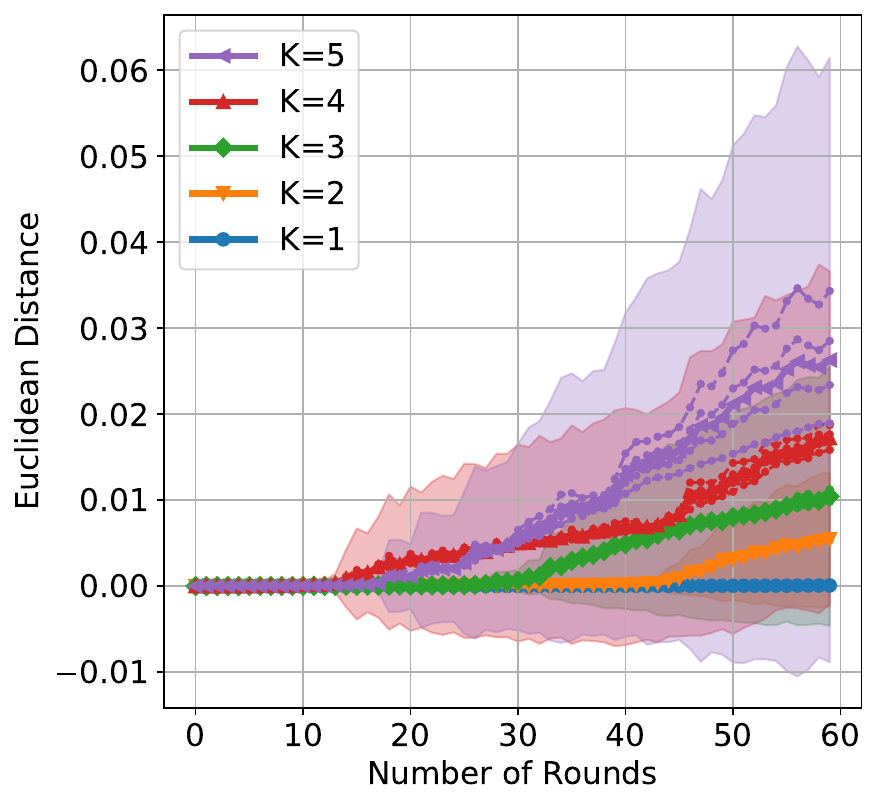}
\includegraphics[width=0.194\textwidth]{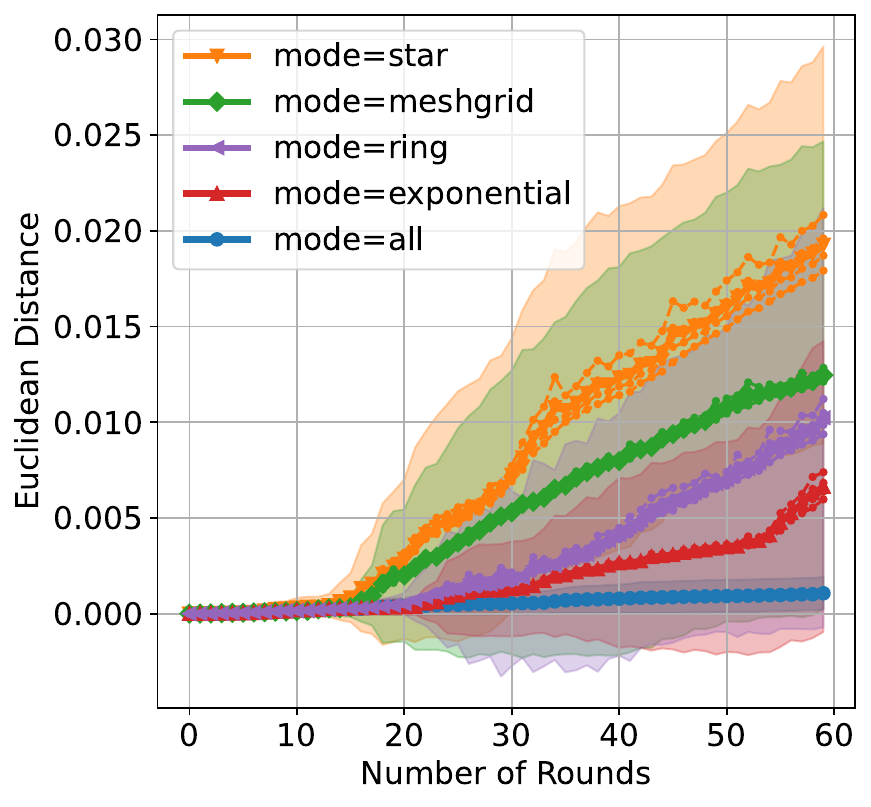}
\vspace{-0.5em}
\includegraphics[width=0.193\textwidth]{Figures/localdecentralized/different_lr/discriminator.pdf}
\includegraphics[width=0.193\textwidth]{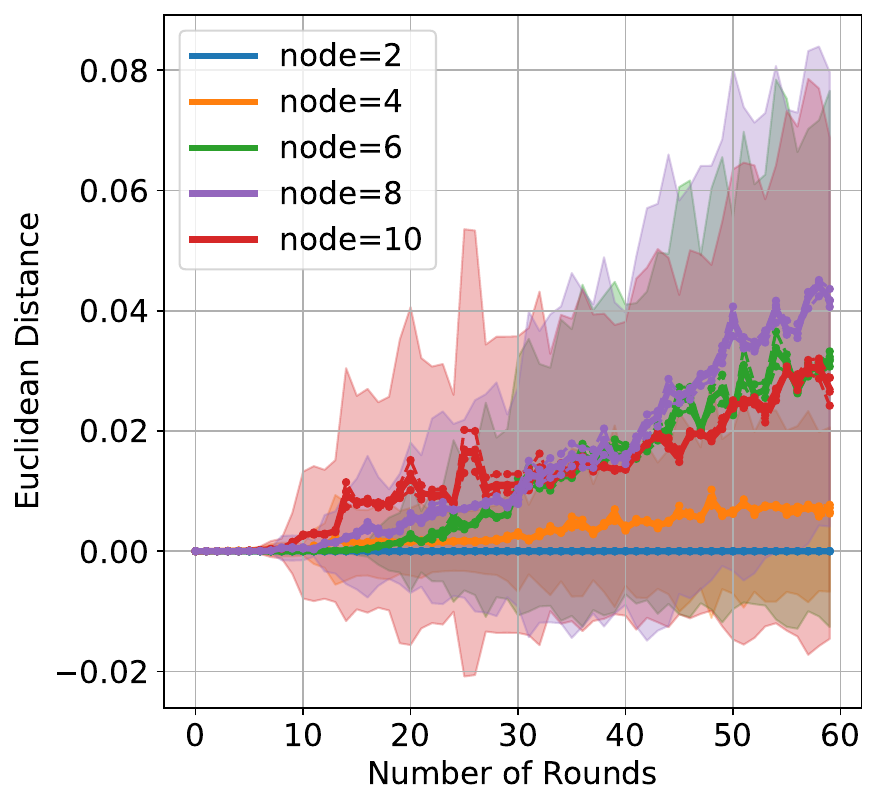}
\includegraphics[width=0.189\textwidth]{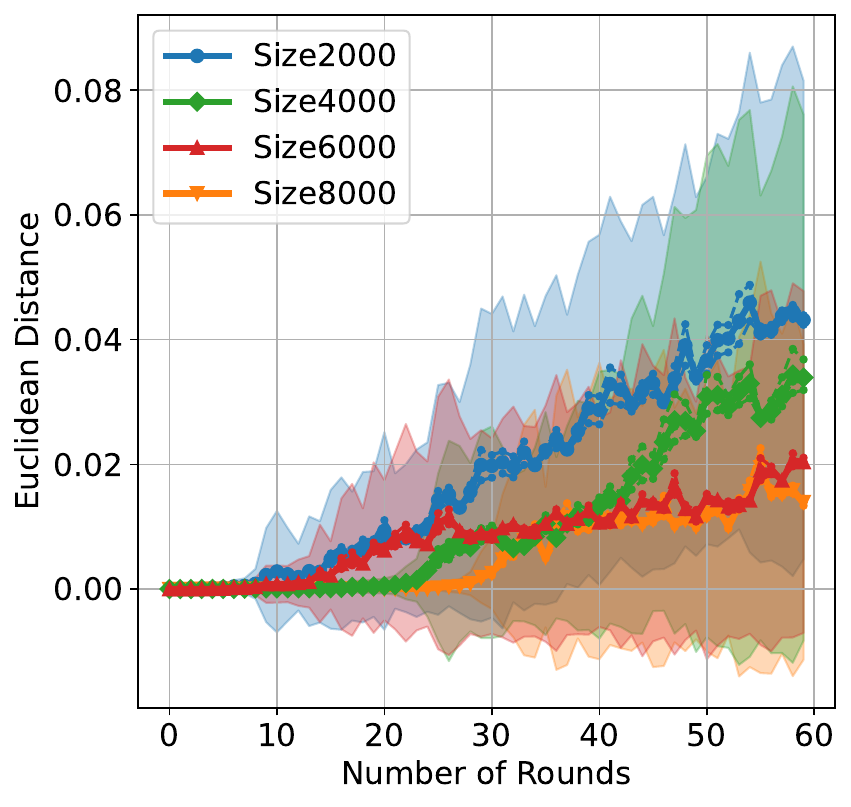}
\includegraphics[width=0.194\textwidth]{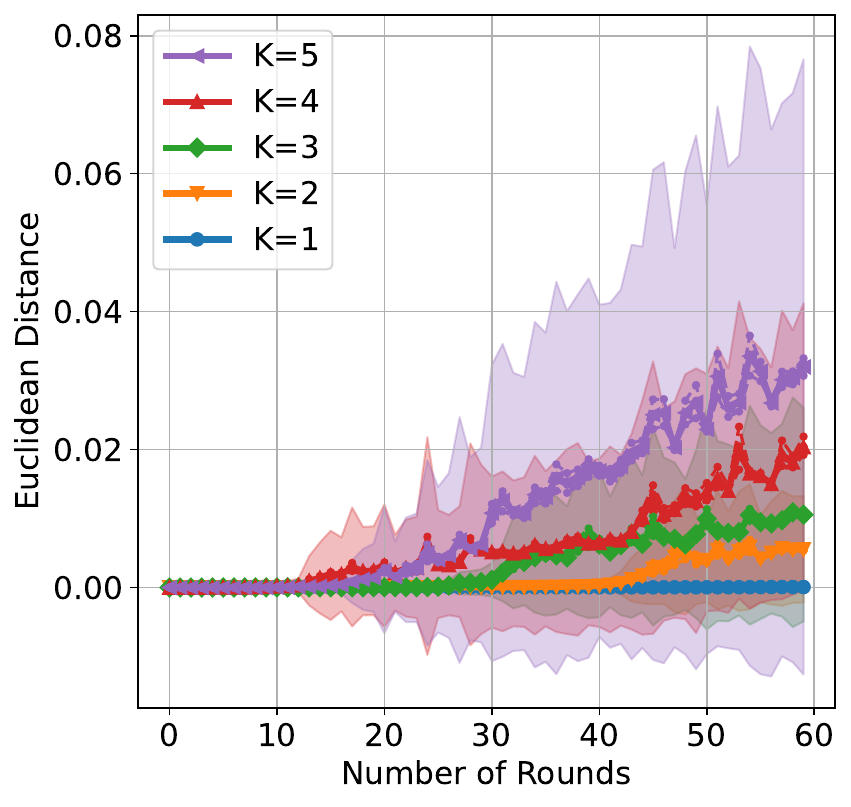}
\includegraphics[width=0.194\textwidth]{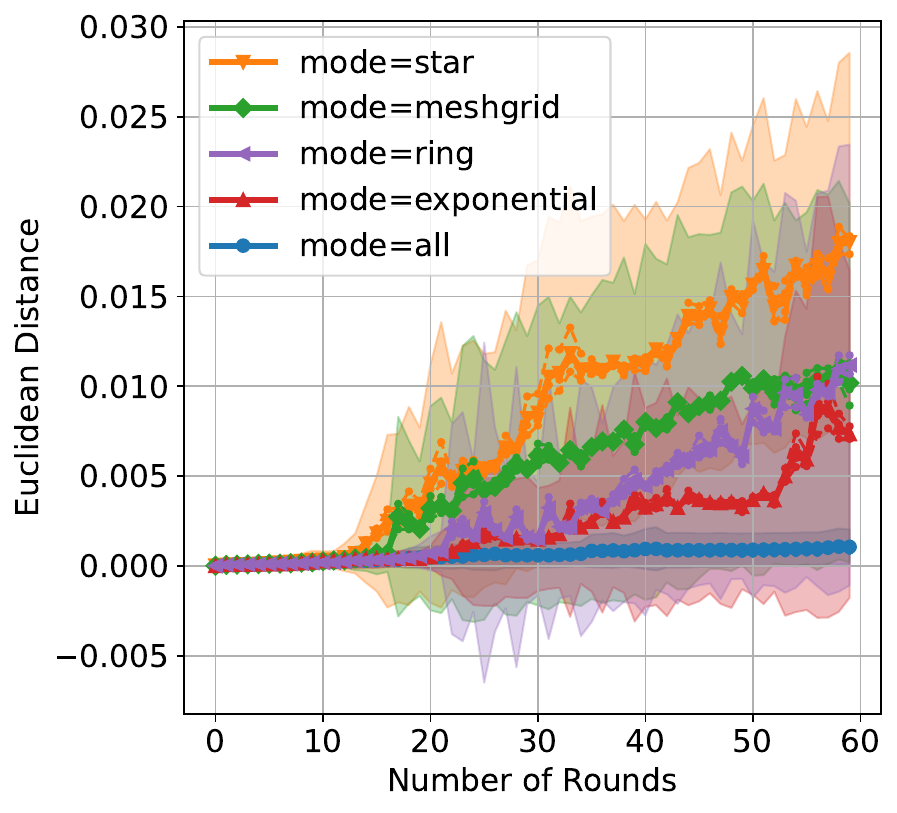}
\vspace{-0.9em}
\caption{Stability of generator model and discriminator model using Local-DSGDA method respectively. From left to right, the figures correspond to the varying learning rates, the number of nodes, the local dataset size, the number of local steps, and the topologies. Each layer is independently assessed and shown as the dashed lines.}
\label{fig:localdecentralized:dis}    
\vspace{-0.4cm}
\end{figure*}

\section{Experiments}

In this section, we mainly demonstrate the preliminary experiments to validate the theoretical analysis above. 
We first introduce the implementation details in section~\ref{tx:implementation} and then show the experiments and further discussions in section~\ref{tx:result}.

\begin{figure*}[t]
\centering
\includegraphics[width=0.24\textwidth]{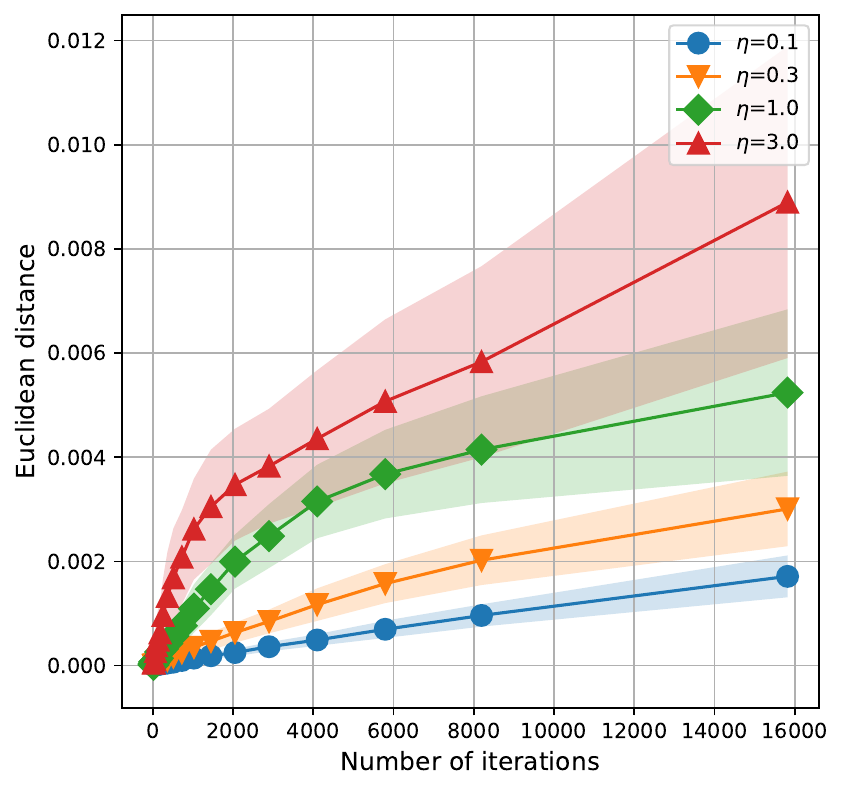}
\includegraphics[width=0.24\textwidth]{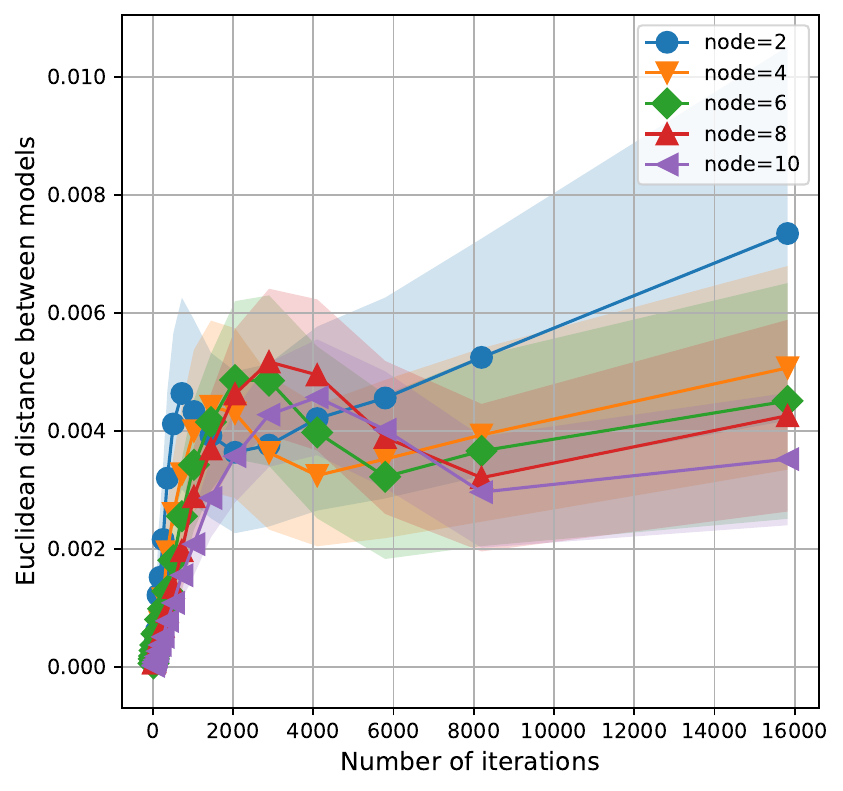}
\includegraphics[width=0.24\textwidth]{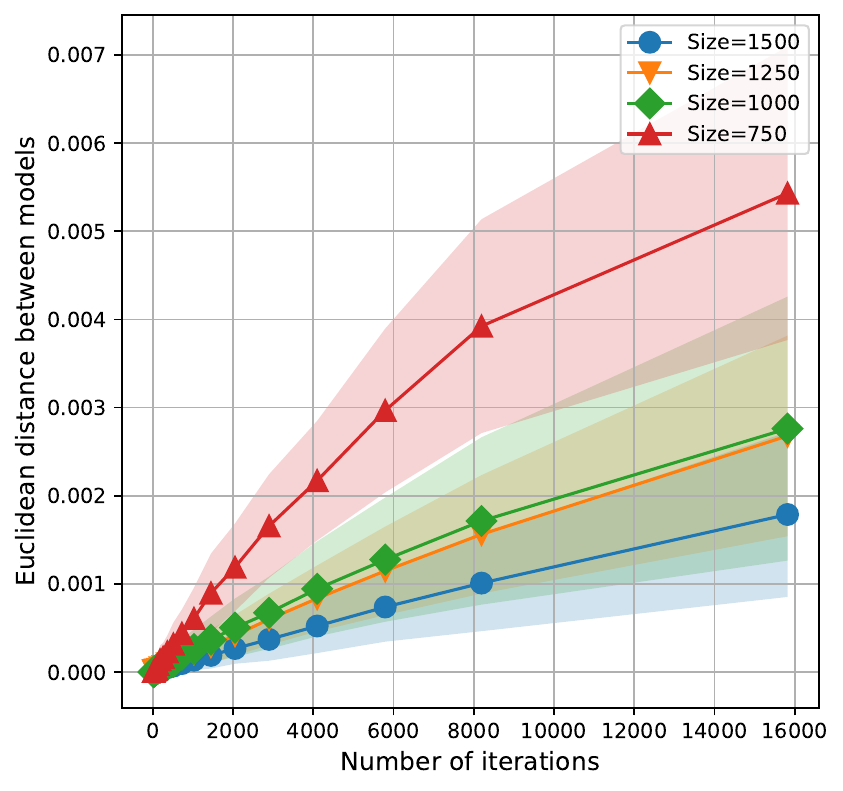}
\includegraphics[width=0.24\textwidth]{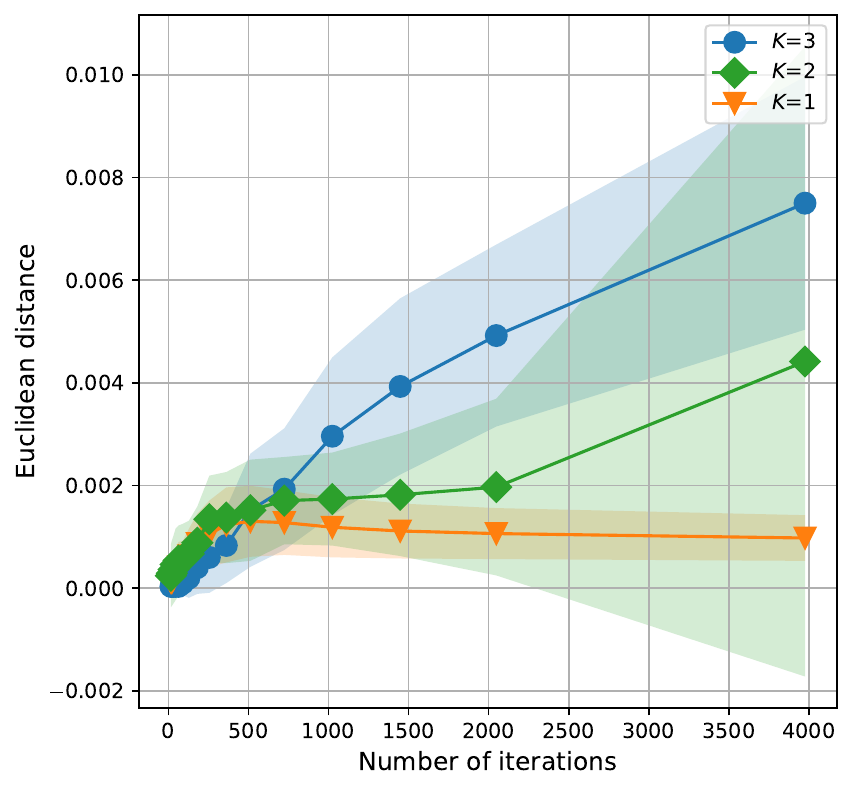} 
\vspace{-0.2em}
\includegraphics[width=0.24\textwidth]{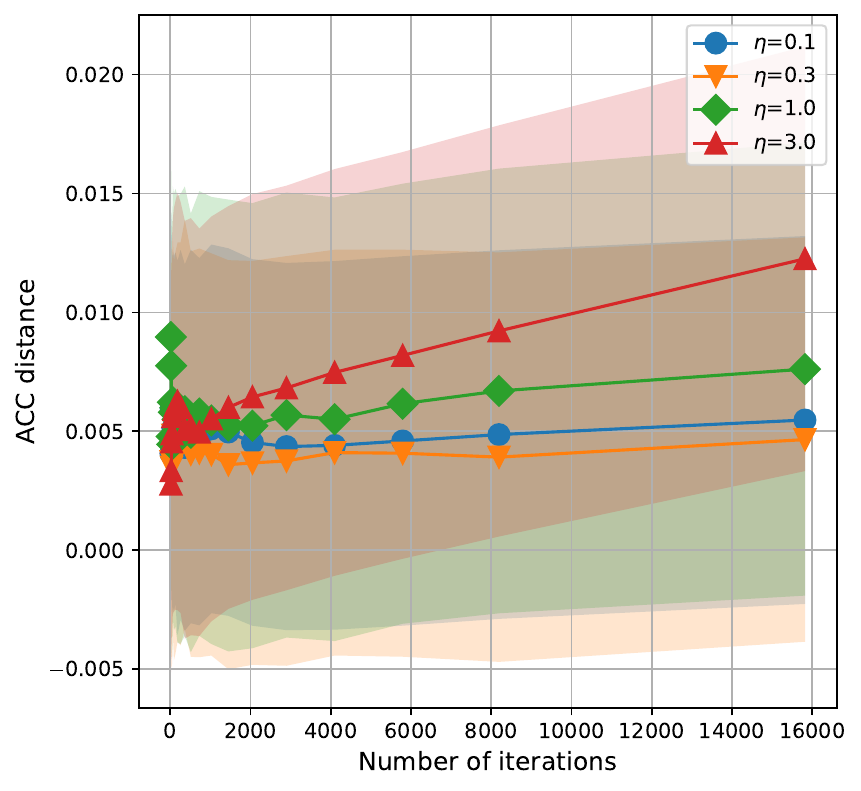}
\includegraphics[width=0.24\textwidth]{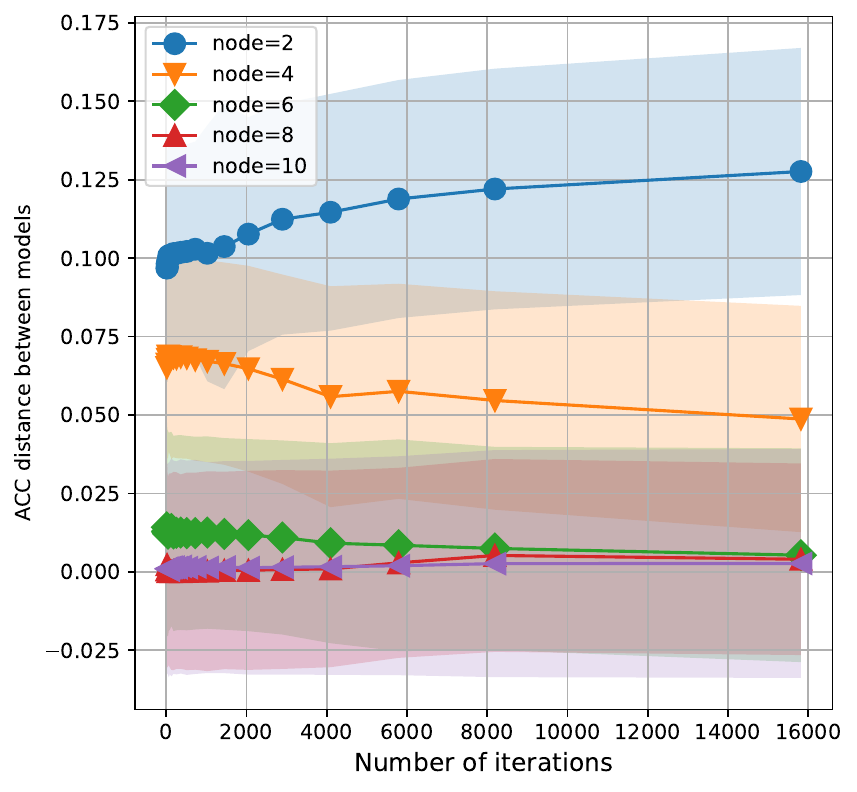}
\includegraphics[width=0.24\textwidth]{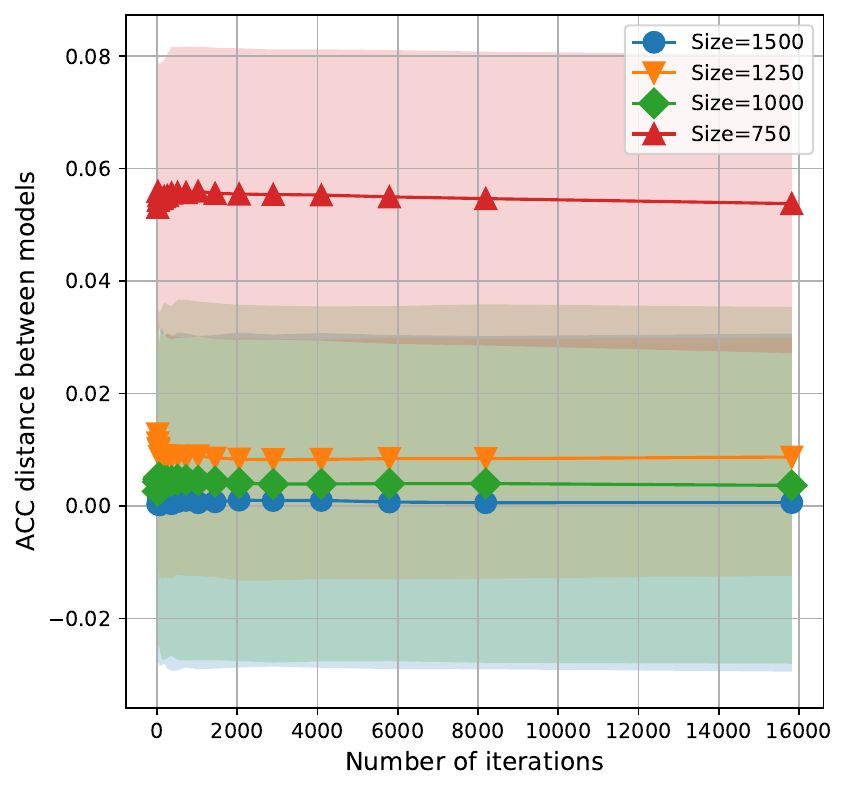}
\includegraphics[width=0.24\textwidth]{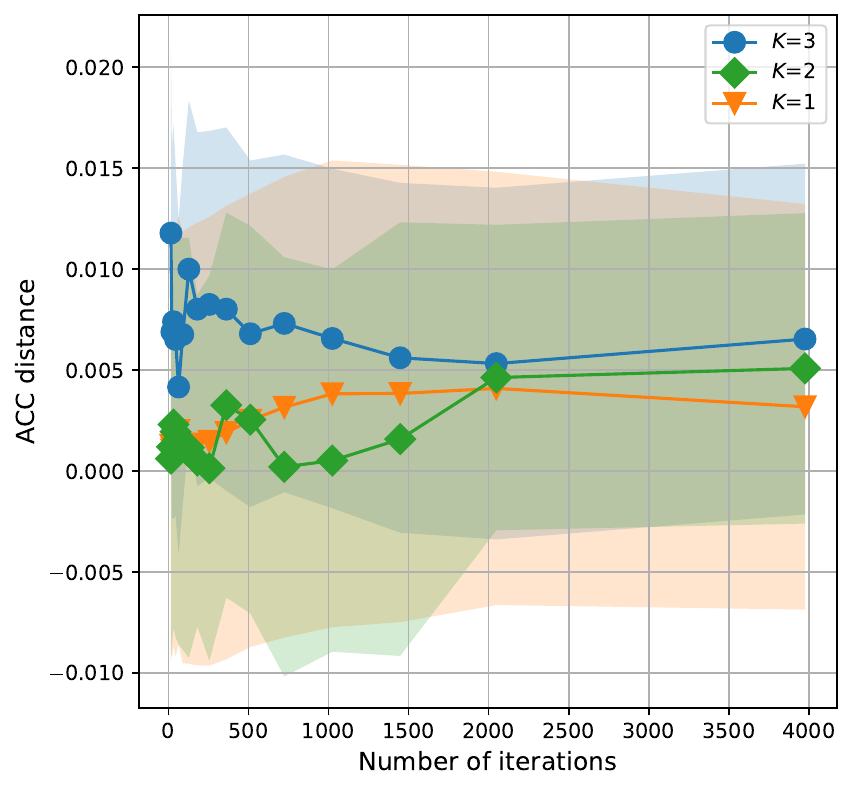} 
\vspace{-0.2em}
\caption{The first row shows the stability of Local-SGDA on AUC Maximization task, which is evaluated by the Euclidean distance between outputs of models trained on neighbouring dataset. The second row shows the generalization performance, evaluated by abs(training loss - test loss). From left to right, the figures correspond to the varying learning rates, the number of nodes, the local dataset size, and the number of local updates.}
\label{fig:local}   
\vspace{-0.2cm}
\end{figure*}

\subsection{Implementation}
\label{tx:implementation}

In experiments, we mainly investigate the stability of the distributed SGDA algorithm under the Nonconvex-Nonconcave setting and the stability and generalization of Local-SGDA under Convex-Concave setting. To evaluate our theoretical analysis, we follow \citet{lei2021stability,zhu2023stability} and solve a generative adversarial task on the MNIST dataset with a vanilla GAN structure to simulate the NC-NC setting, of which the generator and the discriminator comprise $4$ fully connected layers respectively. Each output layer connects to a leaky-ReLU activation. The same as \citet{lei2021stability}, we ignore all regularizations including weight decay, dropout, and data augmentations. We conduct the Local-SGDA on the task of AUC Maximization on dataset w5a to evaluate the stability and generalization under C-C setting. 

According to Definition~\ref{def:neighbouring:dataset}, we construct distributed neighboring dataset $\S=\{\S_1,...,\S_m\}$ and $\S'=\{\S'_1,...,\S'_m\}$, where each corresponding local dataset pair $\left(\S_i,\S'_i\right)$ only differs on one randomly selected data sample. Then we deploy the same initial model~$(\x,\y)$ with its local dataset pair $\left(\S_i,\S'_i\right)$ to the local client $i$. Each model will be trained $K$ iterations locally and communicate with its neighbors on a specific topology for exchanging and aggregating information. To trace the model gaps on the distributed neighboring dataset, after each communication, we aggregate all local models $(\x_i,\y_i)$ trained on $\S$ and perturbing models $(\x_i',\y_i')$ trained on $\S'$ respectively as the virtual global average $(\frac{1}{m}\sum_i\x_i,\frac{1}{m}\sum_i\y_i)$ and $(\frac{1}{m}\sum_i\x_i',\frac{1}{m}\sum_i\y_i')$, and then measure the Euclidean distance $\Vert\frac{1}{m}\sum_i(\x_i - \x_i')\Vert_2^2+\Vert\frac{1}{m}\sum_i(\y_i - \y_i')\Vert_2^2$ as the stability.


We mainly focus on the impacts of five factors, i) learning rates; ii) topologies; iii) node sizes; iv) local steps; and v) sample sizes. In each single study, we freeze selections of the other factors for fairness.
All the results are repeated on 4 different dataset constructions with 4 random seeds.

\subsection{Results}
\label{tx:result}

We validate the stability on both the generator model and discriminator model of Local-SGDA method (see Fig. \ref{fig:local:discriminator}) and of Local-DSGDA method (see Fig. \ref{fig:localdecentralized:dis}) respectively. We also validate the stability and generalization of Local-SGDA on AUC Maximization task (see Fig. \ref{fig:local}).
We can observe the consistent experiments that i) the larger learning rates always lead to a more unstable training process with larger gaps; ii) increasing the training scales will significantly increase instability; iii) increasing the training sample size helps maintain stability to enhance generalization ability; iv) the local training process introduces additional bias which significantly affects the stability of the model. And we can observe that the larger node sizes and sample sizes can help achieve better generalization performance. While increasing the learning rates and adding the number of local iterations will destroy generalization performance.
In addition, the topologies have a considerable impact on decentralized training. The stability ranking of training for different topologies is all $>$ exponential $>$ ring $>$ meshgrid $>$ star, which is closely related to the sparsity and spectral norm of the connections in the topology.

\section{Conclusion}

We propose a unified model \textit{Distributed-SGDA} for distributed minimax problems, analyzing the stability-based generalization gap and population risk for specific Local-SGDA and Local-DSGDA in various metrics across (S)C-(S)C, NC-SC, and NC-NC cases. Our novel theoretical results reveal the trade-off between generalization ability and optimization efficiency and suggest the hyperparameter choice for better performance on overall population risk, which is our \textit{real} target. Specifically, our theoretical results suggest:

\noindent\textbf{Decaying or fixed learning rate?} Choosing fixed learning rates outperforms the decaying one on perspectives of the argument stability ((S)C-(S)C) and the weak stability (NC-NC) and their corresponding generalization gap. Meanwhile, decaying learning rates excel the fixed learning rates in the empirical risks under all settings.

\noindent\textbf{How to balance?} Considering the population risk decomposed of generalization gap and empirical risk, generalization ability has to make a compromise for optimization in weak PD population risk under (S)C-(S)C case, while excess primal generalization gap is the dominant error in the excess primal population risk under NC-SC case. Overall, choosing hyperparameters to prioritize optimizing the dominant error will help lower population risk.

The preliminary experiments validate our theoretical findings that the smaller learning rates, the larger training scales and sample sizes and the less local training process will result in better stability and generalization.

\newpage
\appendix

The notations, basic technical tools, important lemmas including the bound analysis for the consensus term (Lemma \ref{le:delta:local},\ref{le:delta:localdecentra}) are put in Appendix \ref{sec:pre}. The proof of the connection between algorithmic stability and generalization gap (Theorem \ref{thm:connection}) is put in Appendix \ref{sec:connection}. Proof of the argument stability (Theorem \ref{thm:local:sc-sc-stability}), weak PD empirical risk (Theorem \ref{thm:weakpdempirical}) and the weak PD population risk (Theorem \ref{thm:sc-sc-population}) in (S)C-(S)C case are placed in Appendix \ref{sec:sc-sc}. Proof of the primal stability (Theorem \ref{thm:primal}), excess primal empirical risk (Theorem \ref{thm:ep:empirical}) and the excess primal population risk under NC-SC case are put in Appendix \ref{sec:NC-SC}. Proof of the weak stability (Theorem \ref{thm:weak}) under NC-NC case are placed in Appendix \ref{sec:NC-NC}.

\section{Preliminary}\label{sec:pre}
\subsection{Notation}\label{sub:denote}
To facilitate our proof, we first introduce important notations we will use in the proof.

\begin{itemize}
\setlength{\parindent}{0pt}
    \item All the vectors in our paper are column vectors by default, i.e., $\x^t_{i,k}\in\R^{d_{\x}},\y^t_{i,k}\in\R^{d_{\y}}$.
    \item We use $[\X,\Y]^t_k\triangleq\left[\left(\begin{array}{c}
     \x^t_{1,k}  \\
     \y^t_{1,k}
\end{array}\right),\left(\begin{array}{c}
     \x^t_{2,k}  \\
     \y^t_{2,k}
\end{array}\right),...,\left(\begin{array}{c}
     \x^t_{m,k}  \\
     \y^t_{m,k}
\end{array}\right)\right]^{\T}\in\R^{m\times (d_{\x}+d_{\y})}$ to denote the concatenation of all the local variables on the $k$-th local update of the $t$-th communication round and use $[\X,\Y]^t\triangleq\left[\left(\begin{array}{c}
     \x^t_1  \\
     \y^t_1
\end{array}\right),\left(\begin{array}{c}
     \x^t_2  \\
     \y^t_2
\end{array}\right),...,\left(\begin{array}{c}
     \x^t_m  \\
     \y^t_m
\end{array}\right)\right]^{\T}\in\R^{m\times (d_{\x}+d_{\y})}$ to denote the concatenation of all the "centered"-parameter in the t-th round after communication.
    \item Define $\nablaf^t_k\!\triangleq\!\left[\left(\!\!\begin{array}{c}
        \dx f_1(\x^t_{1,k},\y^t_{1,k};\xi_{1,j^t_k(1)})   \\
        -\dy f_1(\x^t_{1,k},\y^t_{1,k};\xi_{1,j^t_k(1)})
    \end{array}\!\!\right),...,\left(\!\!\begin{array}{c}
        \dx f_m(\x^t_{m,k},\y^t_{m,k};\xi_{m,j^t_k(m)})   \\
        -\dy f_m(\x^t_{m,k},\y^t_{m,k};\xi_{m,j^t_k(m)})
    \end{array}\!\!\right)\right]^{\T}\!\!\in\R^{m\times (d_{\x}+d_{\y})}$ to concatenate all the local stochastic gradient on the $t$-$k$-th iteration given on the sample $\bm{\xi}\triangleq\{\xi_{1,j^t_k(1)},...,\xi_{m,j^t_k(m)}\}$.
    \item $\I_m\in\R^{m\times m}$ denotes the $m$-dimension identity matrix and $\1_m\in\R^m$ denotes the all-one $m$-dimension vector. We define $\Pm\in\R^{m\times m}$ the matrix with elements uniformly set to $\frac{1}{m}$.
\end{itemize}

\subsection{Technical Tools}

\begin{property}\label{pro:smooth+}
    For a $L$-Lipschitz smooth function $g(x)$ (see Assumption~\ref{as:Lip:smmoth}) combined with convexity, we have:
    \[g(\y)\geq g(\x)+<\nabla g(\x),\y-\x>+\frac{1}{2L}\|\nabla g(\y)-\nabla g(\x)\|^2, \forall\x,\y.\]
\end{property}

\begin{defi}[\textbf{Frobenius inner product}]\label{def:inner}
    For any real-valued matrices $\bm{A}=(a_{ij}),\bm{B}=(b_{ij})\in\R^{m\times n}$, the Frobenius inner product is defined as:
    $$\langle\bm{A},\bm{B}\rangle_{\F}=\sum_{i=1}^m\sum_{j=1}^na_{ij}b_{ij}=\Tr(\bm{A}^{\T}\bm{B})=\Tr(\bm{B}^{\T}\bm{A}).$$
\end{defi}
\begin{defi}[\textbf{Matrix norm}]
    For real-valued matrix $\bm{A}=(a_{ij})\in\R^{m\times n}$, 
    \begin{enumerate}[wide=\parindent,label=(\roman*)]
        \item the Frobenius norm is defined as: $\|\bm{A}\|_{\F}=\sqrt{\sum_{i=1}^m\sum_{j=1}^na_{ij}^2}=\sqrt{\Tr(\bm{A}^{\T}\bm{A})};$
        \item the $\ell_2$-induced matrix norm is defined as: $\m\bm{A}\m_2=\max_{\|\bm{x}\|_2\neq0}\frac{\|\bm{A}\bm{x}\|_2}{\|\bm{x}\|_2}$, and it is exactly the spectral norm that $\m\bm{A}\m_2=\sqrt{\lambda_{\max}(\bm{A}^{\T}\bm{A})}$.
    \end{enumerate}
\end{defi}

\begin{rem}
    For simplicity, we omit the subscripts of the $\ell_2$ norm for vectors and the Frobenius norm for matrices when it does not lead to misunderstanding, i.e., $\|\x\|$ denotes $\|\x\|_2$ and $\|\bm{A}\|$ means $\|\bm{A}\|_{\F}$. The detailed proof of the equivalence is provided in Example 5.6.6. in~\citep{horn2012matrix}.
\end{rem}

\begin{property}[\citealt{horn2012matrix}]\label{pro:matrix}
    For any real-valued matrices $\bm{A}\in\R^{m\times n}$, $\bm{B}\in\R^{n\times d}$,
    \begin{enumerate}[wide=\parindent,label=(\roman*)]
        \item $|\langle\bm{A},\bm{B}\rangle_{\F}|\leq\|\bm{A}\|_{\F}\|\bm{B}\|_{\F};$
        \item $\|\bm{A}\|_{\F}\leq\m\bm{A}\m_2$, and $\|\bm{A}\bm{B}\|_{\F}\leq\m\bm{A}\m_2\|\bm{B}\|_{\F}\leq\|\bm{A}\|_{\F}\|\bm{B}\|_{\F}$.
    \end{enumerate}
\end{property}

\begin{lem}[\citealt{wang2021cooperative}]\label{le:spect}
    For the mixing matrix $\W$ satisfying the Assumption~\ref{as:mixing}, there holds:
    $$\m\W^k-\Pm\m_2=\lambdabm^k, \forall k\in\mathbb{N}.$$
\end{lem}

\begin{lem}[\citealt{zhu2023stability}]\label{le:nonexpan}
    For the gradient map $G_{g,\eta}$ where $g$ is $L$-Lipschitz smooth:
    $$G_{g,\eta}\left(\begin{array}{c}
        \x\\
        \y 
    \end{array}
    \right)=\left(\begin{array}{c}
        \x-\eta\nabla_{\x}g(\x,\y) \\
        \y+\eta\nabla_{\y}g(\x,\y)
    \end{array}\right),$$
    we have the following expansiveness under different conditions:
    \begin{enumerate}[wide=\parindent,label=(\roman*)]
        \item $G_{g,\eta}$ is $(1+\eta L)$-expansive, i.e., 
        $$\left\|G_{g,\eta}\left(\begin{array}{c}
            \x\\
            \y
        \end{array}\right)-G_{g,\eta}\left(\begin{array}{c}
            \x' \\
            \y' 
        \end{array}\right)\right\|\leq(1+\eta L)\left\|\left(\begin{array}{c}
            \x-\x' \\
            \y-\y'
        \end{array}\right)\right\|;$$
        \item When $g$ is $\mu$-strongly-convex-strongly-concave, and the learning rate satisfies $\eta\leq\frac{2}{L+\mu}$, then $G_{g,\eta}$ is $(1-\eta\frac{L\mu}{L+\mu})$-expansive, i.e., 
        $$\left\|G_{g,\eta}\left(\begin{array}{c}
            \x  \\
            \y 
        \end{array}\right)-G_{g,\eta}\left(\begin{array}{c}
            \x'  \\
            \y' 
        \end{array}\right)\right\|\leq(1-\eta\frac{L\mu}{L+\mu})\left\|\left(\begin{array}{c}
            \x-\x'  \\
            \y-\y'
        \end{array}\right)\right\|.$$
    \end{enumerate}
\end{lem}

\subsection{Important Lemmas}
\begin{lem}[\citet{lin2020gradient}]\label{le:RSlipschitz}
    Defining $\hat{\y}_{\S}(\x)\triangleq\arg\sup_{\y'\in\yset}F_{\S}(\x,\y')$, under Assumption \ref{as:Lip:smmoth}, when each local function $f_i(\x,\cdot)$ is $\mu$-strongly-concave, i.e., $F_{\S}(\x,\cdot)$ is $\mu$-strongly-concave as well, we have the conclusion that $\nabla R_{\S}(\x)=\dx F_{\S}(\x,\hat{\y}_{\S}(\x))$ and $R_{\S}(\x)$ is Lipschitz smooth with $L+\kappa L$. Besides, $\hat{\y}_{\S}(\cdot)$ is Lipschitz continuous with $\kappa$, where $\kappa=\frac{L}{\mu}$.
\end{lem}

\begin{proof}
    See Lemma 4.3 in \citet{lin2020gradient}.
\end{proof}

\begin{lem}\label{le:RSPL}
    Under Assumption \ref{as:Lip:smmoth}, when each local function $f_i(\cdot,\y)$ satisfies $\rho$-PL condition, i.e., $F_{\S}(\cdot,\y)$ satisfies $\rho$-PL condition, then the function $R_{\S}(\x)$ also satisfies $\rho$-PL condition. 
\end{lem}

\begin{proof}

    \begin{align*}
        \|\nabla R_{\S}(\x)\|^2=\|\dx F_{\S}(\x,\hat{\y}_{\S})\|^2&\geq2\rho(F_{\S}(\x,\hat{\y}_{\S}(\x))-\inf_{\x'\in\xset}F_{\S}(\x',\hat{\y}_{\S}(\x)))\\
        &\geq2\rho(R_{\S}(\x)-\inf_{\x'\in\xset}\sup_{\y'\in\yset}F_{\S}(\x',\y'))\\
        &\geq2\rho(R_{\S}(\x)-\inf_{\x'\in\xset}R_{\S}(\x')).
    \end{align*}
\end{proof}

\begin{lem}\label{le:delta:local}
    For LocalSGDA as a special case of Distributed-SGDA, under Assumption \ref{as:Lip:con}, \ref{as:Lip:smmoth}, the consensus term can be bounded by:
    \begin{equation*}
        \E[\Delta_k^t]\leq2G\sum_{k'=0}^{k-1}\eta^t_{k'}\prod_{k''=k'+1}^{k-1}(1-\eta^t_{k''}\frac{L\mu}{L+\mu}).
    \end{equation*}
\end{lem}

\begin{proof}[\textbf{Proof of Lemma~\ref{le:delta:local}}]
    According to the denotation of $(\bar{\x}^t_k,\bar{\y}^t_k)$, we have:
    \begin{equation*}
        \begin{aligned}
            \left(\begin{array}{c}
                \x_{i,k}^t-\bar{\x}_k^t   \\
                \y_{i,k}^t-\bar{\y}_k^t
            \end{array}\right)=\left(\begin{array}{c}
                \x_{i,k}^t-\frac{1}{m}\sum_{h=1}^m\x_{h,k}^t   \\
                \y_{i,k}^t-\frac{1}{m}\sum_{h=1}^m\y_{h,k}^t
            \end{array}\right)=\frac{1}{m}\sum_{h=1}^m\left(\begin{array}{c}
                \x_{i,k}^t-\x_{h,k}^t   \\
                \y_{i,k}^t-\y_{h,k}^t
            \end{array}\right).
        \end{aligned}
    \end{equation*}
    Then we concentrate on the update of $\left(\begin{array}{c}
        \x_{i,k}^t-\x_{h,k}^t   \\
        \y_{i,k}^t-\y_{h,k}^t 
    \end{array}\right)$ that:
    \begin{equation*}
        \begin{aligned}
            &\quad\;\left(\begin{array}{c}
                \x_{i,k+1}^t-\x_{h,k+1}^t \\
                \y_{i,k+1}^t-\y_{h,k+1}^t 
            \end{array}\right)\\
            &=\left(\begin{array}{c}
                \x_{i,k}^t-\eta^t_k\dx f_i(\x_{i,k}^t,\y_{i,k}^t;\xi_{i,j_k^t(i)})-\x_{h,k}^t+\eta^t_k\dx f_h(\x_{h,k}^t,\y_{h,k}^t;\xi_{h,j_k^t(h)})  \\
                \y_{i,k}^t+\eta^t_k\dy f_i(\x_{i,k}^t,\y_{i,k}^t;\xi_{i,j_k^t(i)})-\y_{h,k}^t-\eta^t_k\dy f_h(\x_{h,k}^t,\y_{h,k}^t;\xi_{h,j_k^t(h)})
            \end{array}\right)\\
            &=\left(\begin{array}{c}
                \x_{i,k}^t-\eta^t_k\dx f_i(\x_{i,k}^t,\y_{i,k}^t;\xi_{i,j_k^t(i)})-\x_{h,k}^t+\eta^t_k\dx f_i(\x_{h,k}^t,\y_{h,k}^t;\xi_{h,j_k^t(h)}) \\
                \y_{i,k}^t+\eta^t_k\dy f_i(\x_{i,k}^t,\y_{i,k}^t;\xi_{i,j_k^t(i)})-\y_{h,k}^t-\eta^t_k\dy f_i(\x_{h,k}^t,\y_{h,k}^t;\xi_{h,j_k^t(h)})
            \end{array}\right)\\
            &+\left(\begin{array}{c}
                \eta^t_k\left(\dx f_h(\x_{h,k}^t,\y_{h,k}^t;\xi_{h,j_k^t(h)})-\dx f_i(\x_{h,k}^t,\y_{h,k}^t;\xi_{h,j_k^t(h)})\right)  \\
                -\eta^t_k\left(\dy f_h(\x_{h,k}^t,\y_{h,k}^t;\xi_{h,j_k^t(h)})-\dy f_i(\x_{h,k}^t,\y_{h,k}^t;\xi_{h,j_k^t(h)})\right)
            \end{array}\right).
        \end{aligned}
    \end{equation*}
    
    Combining above equations and making use of Lemma \ref{le:nonexpan}, we can obtain:
    \begin{equation*}
        \begin{aligned}
            \E_{\A}\left\|\left(\begin{array}{c}
                \x_{i,k+1}^t-\x_{h,k+1}^t \\
                \y_{i,k+1}^t-\y_{h,k+1}^t 
            \end{array}\right)\right\|\leq(1-\eta^t_k\frac{L\mu}{L+\mu})\E_{\A}\left\|\left(\begin{array}{c}
                \x_{i,k}^t-\x_{h,k}^t \\
                \y_{i,k}^t-\y_{h,k}^t 
            \end{array}\right)\right\|+2\eta^t_kG.
        \end{aligned}
    \end{equation*}
    
    Recursively repeating above process from ($t$,$k+1$)-th iteration to ($t$,$0$)-th iteration, we can derive:
    \begin{equation*}
        \begin{aligned}
            &\quad\E_{\A}\left\|\left(\begin{array}{c}
                \x_{i,k+1}^t-\x_{h,k+1}^t \\
                \y_{i,k+1}^t-\y_{h,k+1}^t 
            \end{array}\right)\right\|\\
            &\leq(1-\eta^t_k\frac{L\mu}{L+\mu})\E_{\A}\left\|\left(\begin{array}{c}
                \x_{i,k}^t-\x_{h,k}^t \\
                \y_{i,k}^t-\y_{h,k}^t 
            \end{array}\right)\right\|+2\eta^t_kG\\
            &\leq\prod_{k'=k_0}^k\!(1-\eta^t_{k'}\frac{L\mu}{L+\mu})\E_{\A}\left\|\left(\begin{array}{c}
                \x_{i,k_0}^t\!-\x_{h,k_0}^t \\
                \y_{i,k_0}^t\!-\y_{h,k_0}^t \end{array}\right)\right\|+2G\sum_{k'=k_0}^k\eta^t_{k'}\prod_{k''=k'+1}^k\!(1-\eta^t_{k''}\frac{L\mu}{L+\mu})\\
            &\leq 2G\sum_{k'=0}^k\eta^t_{k'}\prod_{k''=k'+1}^k(1-\eta^t_{k''}\frac{L\mu}{L+\mu})
        \end{aligned}
    \end{equation*}
    where the last inequality is due to the fact that: $\x^t_{i,0}=\x^t_{h,0}=\x^t$ and $\y^t_{i,0}=\y^t_{h,0}=\y^t,\;\forall i,h$.
    
    Therefore we can bound the consensus term $\Delta^t_k$ for LocalSGDA as follows:
    \begin{equation*}
        \begin{aligned}
            \E_{\A}[\Delta^t_k]&\leq\frac{1}{m}\sum_{i=1}^m\frac{1}{m}\sum_{h=1}^m 2G\sum_{k'=0}^{k-1}\eta^t_{k'}\prod_{k''=k'+1}^{k-1}(1-\eta^t_{k''}\frac{L\mu}{L+\mu})\\
            &=2G\sum_{k'=0}^{k-1}\eta^t_{k'}\prod_{k''=k'+1}^{k-1}(1-\eta^t_{k''}\frac{L\mu}{L+\mu}).
        \end{aligned}
    \end{equation*}
\end{proof}

\begin{lem}\label{le:delta:localdecentra}
    For Distributed-SGDA $\A(T,K,\W)$, under Assumption \ref{as:mixing}, \ref{as:Lip:con}, \ref{as:Lip:smmoth}, we establish the following bound for the consensus term considering different values of the learning rates:
    \begin{enumerate}[wide=\parindent,label=(\roman*)]
        \item\label{fixed} for fixed learning rates:
        $$\E_{\A}[\Delta^t_k]\leq \eta G\sqrt{\frac{1}{1-\lambdabm}(k^2+\frac{2\lambdabm}{1-\lambdabm^2}K^2)};$$
        \item\label{decaying} for decaying learning rates that $\eta^t_k=\frac{1}{(k+1)^\alpha t^\beta}$:
        $$\E_{\A}[\Delta^t_k]\leq G\sqrt{\frac{1}{1-\lambdabm}\frac{k(1+\ln{k}\cdot\mathbf{1}_{\alpha=\frac{1}{2}}+\frac{1}{2\alpha-1}\cdot\mathbf{1}_{\alpha>\frac{1}{2}})}{t^{2\beta}}+(C_{\lambdabm^2}+\frac{C_{\lambdabm}}{1-\lambdabm})\frac{K(1+\ln{K}\cdot\mathbf{1}_{\alpha=\frac{1}{2}}+\frac{1}{2\alpha-1}\cdot\mathbf{1}_{\alpha>\frac{1}{2}})}{t^{2\beta}}}.$$
    \end{enumerate}
    
\end{lem}

\begin{proof}[\textbf{Proof of Lemma~\ref{le:delta:localdecentra}}]
    Employing the notations in matrix form as a concatenation of the varibles introduced in \ref{sub:denote}, we can conduct:
    \begin{equation*}
        \begin{aligned}
            &\quad\;\sum_{i=1}^m\left\|\left(\begin{array}{c}
                \x^t_{i,k}-\bar{\x}^t_k\\
                \y^t_{i,k}-\bar{\y}^t_k 
            \end{array}\right)\right\|^2\\
            &=\left\|[\X,\Y]^t_k-\Pm[\X,\Y]^t_k\right\|^2\\
            &=\left\|(\I_m-\Pm)[\X,\Y]^t_k\right\|^2\\
            &=\left\|(\I_m-\Pm)\left([\X,\Y]^t_{k-1}-\eta^t_{k-1}\nablaf^t_{k-1}\right)\right\|^2\\
            &=\left\|(\I_m-\Pm)\left([\X,\Y]^t_0-(\eta^t_{k-1}\nablaf^t_{k-1}+...+\eta^t_{k-1}\nablaf^t_0)\right)\right\|^2\\
            &=\left\|(\I_m-\Pm)\left([\X,\Y]^{t-1}-(\eta^t_{k-1}\nablaf^t_{k-1}+...+\eta^t_{k-1}\nablaf^t_0)\right)\right\|^2\\
            &=\left\|(\I_m-\Pm)\left(\W[\X,\Y]^{t-1}_K-(\eta^t_{k-1}\nablaf^t_{k-1}+...+\eta^t_{k-1}\nablaf^t_0)\right)\right\|^2\\
            &=\left\|(\I_m-\Pm)\left(\W\left([\X,\Y]^{t-1}_0-(\eta^{t-1}_{K-1}\nablaf^{t-1}_{K-1}+...+\eta^{t-1}_0\nablaf^{t-1}_0)\right)-(\eta^t_{k-1}\nablaf^t_{k-1}+...+\eta^t_0\nablaf^t_0)\right)\right\|^2\\
            &=\left\|(\I_m-\Pm)\left(\W[\X,\Y]^{t-2}-\W(\eta^{t-1}_{K-1}\nablaf^{t-1}_{K-1}+...+\eta^{t-1}_0\nablaf^{t-1}_0)-(\eta^t_{k-1}\nablaf^t_{k-1}+...+\eta^t_0\nablaf^t_0)\right)\right\|^2\\
            &=\left\|(\I_m-\Pm)\left(\W^{t-1}[\X,\Y]^0-\left((\eta^t_{k-1}\nablaf^t_{k-1}+...+\eta^t_0\nablaf^t_0)+\W(\eta^{t-1}_{K-1}\nablaf^{t-1}_{K-1}+...+\eta^{t-1}_0\nablaf^{t-1}_0)\right.\right.\right.\\
            &\left.\left.\left.\quad+...+\W^{t-1}(\eta^1_{K-1}\nablaf^1_{K-1}+...+\eta^1_0\nablaf^1_0)\right)\right)\right\|^2\\
            &=\left\|-(\I_m-\Pm)(\Xi^t_k+\W\Xi^{t-1}_K+...+\W^{t-1}\Xi^1_K)\right\|^2
        \end{aligned}
    \end{equation*}
    where we recursively use the update rule all above and define $\Xi^t_k\triangleq\sum_{h=0}^{k-1}\eta^t_h\nablaf^t_h$ in the last equation.

    Continuing with the evaluation, we can get:
    \begin{equation}\label{eq:1}
        \begin{aligned}
            &\quad\;\left\|(\I_m-\Pm)\sum_{\tau=0}^{t-1}\W^{\tau}\Xi^{t-\tau}_{k^t_{\tau}}\right\|^2\\
            &\overset{(a)}{=}\left\|\sum_{\tau=0}^{t-1}\W^{\tau}\Xi^{t-\tau}_{k^t_{\tau}}-\Pm\sum_{\tau=0}^{t-1}\Xi^{t-\tau}_{k^t_{\tau}}\right\|^2\\
            &=\left\|\sum_{\tau=0}^{t-1}(\W^\tau-\Pm)\Xi^{t-\tau}_{k^t_\tau}\right\|^2\\
            &\overset{(b)}{=}\sum_{\tau=0}^{t-1}\left\|(\W^\tau-\Pm)\Xi^{t-\tau}_{k^t_\tau}\right\|^2+\sum_{\tau\neq\tau'}^{t-1}\Tr\left(\left((\W^\tau-\Pm)\Xi^{t-\tau}_{k^t_\tau}\right)^{\T}\left((\W^{\tau'}-\Pm)\Xi^{t-\tau'}_{k^t_{\tau'}}\right)\right)
        \end{aligned}
    \end{equation}
    where $k^t_{\tau}=k\cdot\mathbf{1}_{\tau=0}+K\cdot\mathbf{1}_{\tau=1,...,t-1}$, $(a)$ follows from the property of the mixing matrix that $\Pm\W=\Pm=\W\Pm$. $(b)$ follows from the definition~\ref{def:inner} of Frobenius inner product.

    For the first term, we operate as follows:
    \begin{equation}\label{eq:2}
        \begin{aligned}
            \sum_{\tau=0}^{t-1}\left\|(\W^\tau-\Pm)\Xi^{t-\tau}_{k^t_\tau}\right\|^2\leq\sum_{\tau=0}^{t-1}\m\W^\tau-\Pm\m_2^2\|\Xi^{t-\tau}_{k^t_\tau}\|_{\F}^2=\sum_{\tau=0}^{t-1}\lambdabm^{2\tau}\|\Xi^{t-\tau}_{k^t_\tau}\|_{\F}^2
        \end{aligned}
    \end{equation}
    where the first inequality follows from Property~\ref{pro:matrix} and the second equality follows from Lemma~\ref{le:spect}.

    For the second term, we proceed as follows:
    \begin{equation}\label{eq:3}
        \begin{aligned}
            &\quad \sum_{\tau\neq\tau'}^{t-1}\Tr\left(\left((\W^\tau-\Pm)\Xi^{t-\tau}_{k^t_\tau}\right)^{\T}\left((\W^{\tau'}-\Pm)\Xi^{t-\tau'}_{k^t_{\tau'}}\right)\right)
            \\
            &\leq\sum_{\tau\neq\tau'}^{t-1}\|(\W^\tau-\Pm)\Xi^{t-\tau}_{k^t_\tau}\|_{\F}\|(\W^{\tau'}-\Pm)\Xi^{t-\tau'}_{k^t_{\tau'}}\|_{\F}\\
            &\leq\sum_{\tau\neq\tau'}^{t-1}\lambdabm^{\tau}\|\Xi^{t-\tau}_{k^t_\tau}\|_{\F}\cdot\lambdabm^{\tau'}\|\Xi^{t-\tau'}_{k^t_{\tau'}}\|_{\F}\\
            &\leq\frac{1}{2}\sum_{\tau\neq\tau'}^{t-1}\lambdabm^{\tau+\tau'}\left(\|\Xi^{t-\tau}_{k^t_\tau}\|_{\F}^2+\|\Xi^{t-\tau'}_{k^t_{\tau'}}\|_{\F}^2\right)\\
            &=\frac{1}{2}\sum_{\tau\neq\tau'}^{t-1}\lambdabm^{\tau+\tau'}\|\Xi^{t-\tau}_{k^t_\tau}\|_{\F}^2+\frac{1}{2}\sum_{\tau\neq\tau'}^{t-1}\lambdabm^{\tau+\tau'}\|\Xi^{t-\tau'}_{k^t_{\tau'}}\|_{\F}^2\\
            &=\sum_{\tau\neq\tau'}^{t-1}\lambdabm^{\tau+\tau'}\|\Xi^{t-\tau}_{k^t_\tau}\|_{\F}^2.
        \end{aligned}
    \end{equation}

    Substituting \eqref{eq:2} and \eqref{eq:3} back into inequality~\eqref{eq:1}, we can derive:
    \begin{equation*}
        \begin{aligned}
            &\quad\;\left\|(\I_m-\Pm)\sum_{\tau=0}^{t-1}\W^{\tau}\Xi^{t-\tau}_{k^t_{\tau}}\right\|^2\\
            &\leq\sum_{\tau=0}^{t-1}\lambdabm^{2\tau}\|\Xi^{t-\tau}_{k^t_\tau}\|_{\F}^2+\sum_{\tau=0}^{t-1}\lambdabm^\tau\|\Xi^{t-\tau}_{k^t_\tau}\|_{\F}^2\sum_{\tau'\neq\tau}^{t-1}\lambdabm^{\tau'}\\
            &=\|\Xi^t_k\|^2_{\F}+\sum_{\tau=1}^{t-1}\lambdabm^{2\tau}\|\Xi^{t-\tau}_K\|^2_{\F}+\|\Xi^t_k\|_{\F}^2\sum_{\tau'=1}^{t-1}\lambdabm^{\tau'}+\sum_{\tau=1}^{t-1}\lambdabm^\tau\|\Xi^{t-\tau}_K\|_{\F}^2\sum_{\tau'\neq\tau}^{t-1}\lambdabm^{\tau'}\\
            &\leq\frac{1}{1-\lambdabm}\|\Xi^t_k\|^2_{\F}+\sum_{\tau=1}^{t-1}\lambdabm^{2\tau}\|\Xi^{t-\tau}_K\|^2_{\F}+\frac{1}{1-\lambdabm}\sum_{\tau=1}^{t-1}\lambdabm^\tau\|\Xi^{t-\tau}_K\|_{\F}^2
        \end{aligned}
    \end{equation*}
    where the second last equality is due to the definition of $k^t_\tau$ below Eq~\eqref{eq:1}.

    Then, we will discuss different conditions when learning rates are fixed or decaying during the $t$-th communication round or the $k$-th local update iteration.
    
    \begin{enumerate}[wide=\parindent,label=(\roman*)]
        \item \textbf{Learning rates are fixed on the $t$ -th round and on the $k$ -th iteration:}\label{con1}
        \begin{gather*}
            \|\Xi^t_k\|^2_{\F}=\left\|\sum_{h=0}^{k-1}\eta^t_h\nablaf^t_h\right\|^2_{\F}\leq mk^2\eta^2G^2,\\
            \sum_{\tau=1}^{t-1}\lambdabm^{2\tau}\|\Xi^{t-\tau}_K\|^2_{\F}\leq \sum_{\tau=1}^{t-1}\lambdabm^{2\tau}mK^2\eta^2G^2\leq\frac{\lambdabm^2}{1-\lambdabm^2}\cdot mK^2\eta^2G^2,\\
            \sum_{\tau=1}^{t-1}\lambdabm^\tau\|\Xi^{t-\tau}_K\|_{\F}^2\leq\sum_{\tau=1}^{t-1}\lambdabm^\tau mK^2\eta^2G^2\leq\frac{\lambdabm}{1-\lambdabm}\cdot mK^2\eta^2G^2.
        \end{gather*}
        According to the Cauchy-Schwarz inequality, we have:
        \begin{equation*}
            \begin{aligned}
                \Delta^t_k=\frac{1}{m}\sum_{i=1}^m\left\|\left(\begin{array}{cc}
                    \x^t_{i,k}-\bar{\x}^t_k  \\
                    \y^t_{i,k}-\bar{\y}^t_k 
                \end{array}\right)\right\|&\leq\frac{1}{m}\sqrt{m}\sqrt{\sum_{i=1}^m\left\|\left(\begin{array}{c}
                    \x^t_{i,k}-\bar{\x}^t_k\\
                    \y^t_{i,k}-\bar{\y}^t_k 
                \end{array}\right)\right\|^2}\\&\leq\eta G\sqrt{\frac{1}{1-\lambdabm}(k^2+\frac{2\lambdabm}{1-\lambdabm^2}K^2)}.
            \end{aligned}
        \end{equation*}
        \item \textbf{Learning rates are fixed on the $t$-th round while decaying on the $k$-th iteration, i.e., $\eta^t_k=\frac{1}{(k+1)^\alpha}$:}\label{con2}
        \begin{gather*}
            \|\Xi^t_k\|^2_{\F}=\left\|\sum_{h=0}^{k-1}\eta^t_h\nablaf^t_h\right\|^2_{\F}\leq mkG^2\sum_{h=0}^{k-1}\frac{1}{(h+1)^{2\alpha}}\leq mkG^2\left(1+\ln{k}\cdot\mathbf{1}_{\alpha=\frac{1}{2}}+\frac{1}{2\alpha-1}\cdot\mathbf{1}_{\alpha>\frac{1}{2}}\right),\\
            \sum_{\tau=1}^{t-1}\lambdabm^{2\tau}\|\Xi^{t-\tau}_K\|^2_{\F}\leq\sum_{\tau=1}^{t-1}\lambdabm^{2\tau}mKG^2\sum_{h=0}^{K-1}\frac{1}{(h+1)^{2\alpha}}\leq\frac{\lambdabm^2}{1-\lambdabm^2}mKG^2\left(1+\ln{K}\cdot\mathbf{1}_{\alpha=\frac{1}{2}}+\frac{1}{2\alpha-1}\cdot\mathbf{1}_{\alpha>\frac{1}{2}}\right),\\
            \sum_{\tau=1}^{t-1}\lambdabm^\tau\|\Xi^{t-\tau}_K\|_{\F}^2\leq\sum_{\tau=1}^{t-1}\lambdabm^\tau mKG^2\sum_{h=0}^{K-1}\frac{1}{(h+1)^{2\alpha}}\leq\frac{\lambdabm}{1-\lambdabm}mKG^2\left(1+\ln{K}\cdot\mathbf{1}_{\alpha=\frac{1}{2}}+\frac{1}{2\alpha-1}\cdot\mathbf{1}_{\alpha>\frac{1}{2}}\right),\\
            \Delta^t_k\leq G\sqrt{\frac{1}{1-\lambdabm}\left(k(1+\ln{k}\cdot\mathbf{1}_{\alpha=\frac{1}{2}}+\frac{1}{2\alpha-1}\cdot\mathbf{1}_{\alpha>\frac{1}{2}})+\frac{2\lambdabm}{1-\lambdabm^2}K(1+\ln{K}\cdot\mathbf{1}_{\alpha=\frac{1}{2}}+\frac{1}{2\alpha-1}\cdot\mathbf{1}_{\alpha>\frac{1}{2}})\right)}.
        \end{gather*}
        \item \textbf{Learning rates are decaying on the $t$-th round while fixed on the $k$-th iteration, i.e., $\eta^t_k=\frac{1}{t^\beta}$:}\label{con3}
        \begin{gather*}
            \|\Xi^t_k\|^2_{\F}=\left\|\sum_{h=0}^{k-1}\eta^t_h\nablaf^t_h\right\|^2_{\F}\leq mk^2G^2\frac{1}{t^{2\beta}},\\
            \sum_{\tau=1}^{t-1}\lambdabm^{2\tau}\|\Xi^{t-\tau}_K\|^2_{\F}\leq\sum_{\tau=1}^{t-1}\lambdabm^{2\tau}mK^2G^2\frac{1}{(t-\tau)^{2\beta}}\leq mK^2G^2\frac{C_{\lambdabm^2}}{t^{2\beta}},\\
            \sum_{\tau=1}^{t-1}\lambdabm^\tau\|\Xi^{t-\tau}_K\|_{\F}^2\leq\sum_{\tau=1}^{t-1}\lambdabm^{\tau}mK^2G^2\frac{1}{(t-\tau)^{2\beta}}\leq mK^2G^2\frac{C_{\lambdabm}}{t^{2\beta}},\\
            \Delta^t_k\leq G\sqrt{\frac{1}{1-\lambdabm}\frac{k^2}{t^{2\beta}}+C_{\lambdabm^2}\frac{K^2}{t^{2\beta}}+\frac{C_{\lambdabm}}{1-\lambdabm}\frac{K^2}{t^{2\beta}}}.
        \end{gather*}
        \item \textbf{Learning rates are decaying on the $t$-th round and on the $k$-th iteration, i.e., $\eta^t_k=\frac{1}{(k+1)^\alpha t^\beta}$:}\label{con4}
        \begin{gather*}
            \|\Xi^t_k\|^2_{\F}=\left\|\sum_{h=0}^{k-1}\eta^t_h\nablaf^t_h\right\|^2_{\F}\leq mkG^2\frac{1}{t^{2\beta}}\left(1+\ln{k}\cdot\mathbf{1}_{\alpha=\frac{1}{2}}+\frac{1}{2\alpha-1}\cdot\mathbf{1}_{\alpha>\frac{1}{2}}\right),\\
            \sum_{\tau=1}^{t-1}\lambdabm^{2\tau}\|\Xi^{t-\tau}_K\|^2_{\F}\leq mKG^2\frac{C_{\lambdabm^2}}{t^{2\beta}}\left(1+\ln{K}\cdot\mathbf{1}_{\alpha=\frac{1}{2}}+\frac{1}{2\alpha-1}\cdot\mathbf{1}_{\alpha>\frac{1}{2}}\right),\\
            \sum_{\tau=1}^{t-1}\lambdabm^\tau\|\Xi^{t-\tau}_K\|_{\F}^2\leq mKG^2\frac{C_{\lambdabm}}{t^{2\beta}}\left(1+\ln{K}\cdot\mathbf{1}_{\alpha=\frac{1}{2}}+\frac{1}{2\alpha-1}\cdot\mathbf{1}_{\alpha>\frac{1}{2}}\right),\\
            \Delta^t_k\leq G\sqrt{\frac{1}{1-\lambdabm}\frac{k(1+\ln{k}\cdot\mathbf{1}_{\alpha=\frac{1}{2}}+\frac{1}{2\alpha-1}\cdot\mathbf{1}_{\alpha>\frac{1}{2}})}{t^{2\beta}}+(C_{\lambdabm^2}+\frac{C_{\lambdabm}}{1-\lambdabm})\frac{K(1+\ln{K}\cdot\mathbf{1}_{\alpha=\frac{1}{2}}+\frac{1}{2\alpha-1}\cdot\mathbf{1}_{\alpha>\frac{1}{2}})}{t^{2\beta}}}.
        \end{gather*}
    \end{enumerate}
    where the notation of $C_{\lambdabm^2}$ as well as $C_{\lambdabm}$ refers to Lemma 3 in~\cite{zhu2023stability}.

    Since condition~\ref{con4} involves both the decaying learning rates on the $t$-th round and on the $k$-th iteration which are discussed in condition~\ref{con2} and condition~\ref{con3} respectively, we will only consider two situations whether learning rates are fixed or decaying in the subsequent texts.

    Besides, we can set $\lambdabm=0, C_{\lambdabm}=C_{\lambdabm^2}=0$ which symbolizes the LocalSGDA, and our result here coincides with the counterpart of LocalSGDA as proved in Lemma \ref{le:delta:local}. 
\end{proof}

\section{Proof of the Connection}\label{sec:connection}
We make some adjustments to the proof process presented in \citep{zhu2023stability} for better understanding and extend existing results to the excess primal generalization gap.

\begin{proof}[\textbf{Proof of Theorem~\ref{thm:connection}}]\label{pf:connection}
    Here we set up the distributed neighbouring dataset $\S,\S'$ following these stages. For different datasets $\S=\{\S_1,...,\S_m\}$ and $\S'=\{\S'_1,...,\S'_m\}$ with local datasets $\S_i=\{\xi_{i,1},...,\xi_{i,n}\}$ and $\S'_i=\{\xi'_{i,1},...,\xi'_{i,n}\}$ respectively. Let $\bm{r}\triangleq\{r_1,...,r_m\}$ and $\S^{(\bm{r})}\triangleq\{\S^{(r_1)}_1,...,\S^{(r_m)}_m\}$, where $\S^{(r_i)}_i=\{\xi_{i,1},...,\xi_{i,r_i-1},\xi'_{i,r_i},\xi_{i,r_i+1},...,\xi_{i,n}\}$ and $1\leq r_i\leq n,\forall i=1,...,m$. Then the distributed neighbouring dataset $\S,\S^{(\bm{r})}$ has been constructed. And in the following sections, we adopt this method to construct the distributed neighbouring dataset without misunderstanding.

    \textit{\textbf{Weak PD generalization gap~\ref{connection:weak}.}} Rearrange the order of the weak PD generaliztaion gap:
    \begin{equation*}
        \begin{aligned}
            &\quad\zeta^{\mathrm{w}}_{\mathrm{gen}}(\A,\S)\\
            &=\!\!(\!\sup_{\y'\!\in\yset}\!\E[F(\A_{\x}(\S),\y')]\!\!-\!\!\!\inf_{\x'\!\in\xset}\!\!\E[F(\x'\!,\A_{\y}(\S))])\!\!-\!\!(\!\sup_{\y'\!\in\yset}\!\E[\FS(\A_{\x}(\S)),\y')]\!\!-\!\!\inf_{\x'\!\in\xset}\!\! 
           \E[\FS(\x'\!,\A_{\y}(\S))])\\
            &=\!\!\sup_{\y'\!\in\yset}\!\E[F(\A_{\x}(\S),\y')]\!-\!\sup_{\y'\!\in\yset}\!\E[\FS(\A_{\x}(\S)),\y')]\!+\!\inf_{\x'\!\in\xset}\!\!\E[\FS(\x'\!,\A_{\y}(\S))]\!-\!\inf_{\x'\!\in\xset}\!\!\E[F(\x'\!,\A_{\y}(\S))]\\
            &\leq\!\sup_{\y'\in\yset}\!\left(\E_{\A,\S}[F(\A_{\x}(\S),\y')\!-\!\FS(\A_{\x}(\S),\y')]\right)\!+\!\sup_{\x'\in\xset}\!\left(\E_{\A,\S}[\FS(\x'\!,\A_{\y}(\S))\!-\!F(\x'\!,\A_{\y}(\S))]\right).
        \end{aligned}
    \end{equation*}

    For the first term, we have:
    \begin{equation}\label{eq:00}
        \begin{aligned}
            &\quad\E_{\A,\S}[F(\A_{\x}(\S),\y')\!-\!\FS(\A_{\x}(\S),\y')]\\
            &=\!\frac{1}{n^m}\!\!\sum_{r_1=1}^n\!\!\cdots\!\!\!\sum_{r_m=1}^n\!\!\E_{\A,\S,\S'}[F(\A_{\x}(\S^{(\bm{r})}),\y')]-\E_{\A,\S}[\FS(\A_{\x}(\S),\y')]\\
            &=\!\frac{1}{n^m}\!\!\sum_{\bm{r}}\E_{\A,\S,\S'}[\frac{1}{m}\!\sum_{i=1}^m\! f_i(\A_{\x}\!(\S^{(\bm{r})}),\y';\xi_{i,r_i})\!-\!\frac{1}{m}\!\sum_{i=1}^m\!f_i(\A_{\x}\!(\S),\y';\xi_{i,r_i})]\\
            &\leq\!\frac{1}{n^m}\!\!\sum_{\bm{r}}\E_{\A,\S,\S'}[\frac{1}{m}\sum_{i=1}^m G\|\A_{\x}(\S^{(\bm{r})})-\A_{\x}(\S)\|]\\
            &=\!G\E_{\A,\S,\S'}\|\A_{\x}(\S^{(\bm{r})})-\A_{\x}(\S)\|
        \end{aligned}
    \end{equation}
    where the first equation is due to the symmetric data distribution of $\xi_{i,l_i}\in\S_i\sim\D_i$ and $\xi'_{i,l_i}\in\S'_i\sim\D_i$ and there are $n^m$ permutations of $\bm{r}=\{r_1,...,r_m\}$. The second equation follows from the fact that $\xi_{i,r_i}\notin\S^{(\bm{r})}$ so that $\xi_{i,r_i}$ is independent from $\S^{\bm{r}}$ and we abbreviate $\sum_{r_1=1}^n\cdots\sum_{r_m=1}^n$ as $\sum_{\bm{r}}$ for brevity. [Remark. The independence is shown in this way: $\E_{\A,\S}[f(\A(\S);\xi)]=\E_{\xi}\E_{\A,\S|\xi}[f(\A(\S);\xi)|\xi]=\E_{\xi}\E_{\A,\S}[f(\A(\S);\xi)]=\E_{\A,\S}[F(\A(\S))].$]

    The second term can be valued in the same way:
    \begin{equation*}
        \E_{\A,\S}[\FS(\x'\!,\A_{\y}(\S))\!-\!F(\x'\!,\A_{\y}(\S))]\leq G\E_{\A,\S,\S'}\|\A_{\y}(\S^{(\bm{r})})-\A_{\y}(\S)\|.
    \end{equation*}

    Therefore we can derive the weak PD generalization gap:
    \begin{equation*}
        \zeta^{\mathrm{w}}_{\mathrm{gen}}(\A,\S)\leq\sqrt{2}G\sup_{\S,\S'}\E_{\A}\left\|\left(\begin{array}{c}
        \A_{\x}(\S)-\A_{\x}(\S^{(\bm{r})})   \\
        \A_{\y}(\S)-\A_{\y}(\S^{(\bm{r})})
    \end{array}\right)\right\|\leq\sqrt{2}G\epsilon.
    \end{equation*}

    Besides, from the second equation in \eqref{eq:00}, we can conduct the weak PD generalization gap directly from weak stability:
    \begin{align*}
        \zeta^{\mathrm{w}}_{\mathrm{gen}}(\A,\S)&\leq\!\sup_{\y'\in\yset}\!\left(\E_{\A,\S}[F(\A_{\x}(\S),\y')\!-\!\FS(\A_{\x}(\S),\y')]\right)\!+\!\sup_{\x'\in\xset}\!\left(\E_{\A,\S}[\FS(\x'\!,\A_{\y}(\S))\!-\!F(\x'\!,\A_{\y}(\S))]\right)\\
        &=\!\sup_{\y'\in\yset}\!\left(\!\frac{1}{n^m}\!\!\sum_{\bm{r}}\E_{\A,\S,\S'}[\frac{1}{m}\!\sum_{i=1}^m\! f_i(\A_{\x}\!(\S^{(\bm{r})}),\y';\xi_{i,r_i})\!-\!\frac{1}{m}\!\sum_{i=1}^m\!f_i(\A_{\x}\!(\S),\y';\xi_{i,r_i})]\right)\!\\
        &\quad+\!\sup_{\x'\in\xset}\!\left(\!\frac{1}{n^m}\!\!\sum_{\bm{r}}\E_{\A,\S,\S'}[\frac{1}{m}\!\sum_{i=1}^m\! f_i(\x'\!,\A_{\y}(\S^{(\bm{r})});\xi_{i,r_i})\!-\!\frac{1}{m}\!\sum_{i=1}^m\!f_i(\x'\!,\A_{\y}(\S^{(\bm{r})});\xi_{i,r_i})]\right)\\
        &\leq\sup_{\xi_{i,l_{i}}}\big[\sup_{\y'\in\yset}\E_{\A}[\frac{1}{m}\sum_{i=1}^m\big(f_i(\A_{\x}(\S),\y';\xi_{i,l_i})-f_i(\A_{\x}(\S'),\y';\xi_{i,l_i})\\
        &\quad+f_i(\x',\A_{\y}(\S);\xi_{i,l_i})-f_i(\x',\A_{\y}(\S');\xi_{i,l_i})\big)]\big].
    \end{align*}

    \textit{\textbf{Excess primal generalization gap~\ref{connection:ex}.}}
    Observing the structure of the excess primal generalization gap:
    \begin{equation}\label{eq:41}
        \begin{aligned}
            \zeta^{\mathrm{ep}}_{\mathrm{gen}}(\A,\S)&=\Delta^{\!\mathrm{ep}}(\A_{\x}(\S))-\Delta^{\!\mathrm{ep}}_{\S}(\A_{\x}(\S))\\
            &=\E[R(\A_{\x}(\S))-\inf_{\x'\in\xset}R(\x')]-\left(\E[R_{\S}(\A_{\x}(\S))-\inf_{\x'\in\xset}R_{\S}(\x')]\right)\\
            &=\left(\E_{\A,\S}[R(\A_{\x}(\S))-R_{\S}(\A_{\x}(\S))]\right)+\left(\E_{\A,\S}[\inf_{\x'\in\xset}R_{\S}(\x')-\inf_{\x'\in\xset}R(\x')]\right).
        \end{aligned}
    \end{equation}

    Let $(\x^*,\y^*)$ be the saddle point of $F$, i.e., $\forall \x\in\xset,\y\in\yset$, we have: $F(\x^*,\y)\leq F(\x^*,\y^*)\leq F(\x,\y^*)$. Regarding the primal risk $R(\x)=\sup_{\y'\in\yset}F(\x,\y')$ and the primal empirical risk $R_{\S}(\x)=\sup_{\y'\in\yset}F_{\S}(\x,\y')$, let $\y^*_{\S}=\arg\sup_{\y'\in\yset}F(\A_{\x}(\S),\y')$ and $\hat{\y}^*_{\S}=\arg\sup_{\y'\in\yset} F_{\S}(\x^*,\y')$. Proceeding with the first term:
    \begin{equation}\label{eq:42}
        \begin{aligned}
            &\quad\;\E_{\A,\S}[R(\A_{\x}(\S))-R_{\S}(\A_{\x}(\S))]\\
            &=\E_{\A,\S}[\sup_{\y'\in\yset}F(\A_{\x}(\S),\y')-\sup_{\y'\in\yset}F_{\S}(\A_{\x}(\S),\y')]\\
            &\overset{(a)}{=}\frac{1}{n^m}\sum_{r_1=1}^n\cdots\sum_{r_m=1}^n\E_{\A,\S,\S'}[F(\A_{\x}(\S^{(\bm{r})}),\y^*_{\S^{(\bm{r})}})]-\E_{\A,\S}[\sup_{\y'\in\yset}F_{\S}(\A_{\x}(\S),\y')]\\
            &\overset{(b)}{\leq}\frac{1}{n^m}\sum_{\bm{r}}\E_{\A,\S,\S'}[\frac{1}{m}\sum_{i=1}^m f_i(\A_{\x}(\S^{(\bm{r})}),\y^*_{\S^{(\bm{r})}};\xi_{i,r_i})-\frac{1}{m}\sum_{i=1}^mf_i(\A_{\x}(\S),\y^*_{\S};\xi_{i,r_i})]\\
            &\leq G\E_{\A}[\sup_{\S,\S'}\left\|\left(\begin{array}{c}
                \A_{\x}(\S^{(\bm{r})})-\A_{\x}(\S) \\
                \y^*_{\S^{(\bm{r})}}-\y^*_{\S} 
            \end{array}\right)\right\|]\\
            &\overset{(c)}{\leq}G\sqrt{1+\frac{L^2}{\mu^2}}\E_{\A}[\sup_{\S,\S'}\|\A_{\x}(\S^{(\bm{r})})-\A_{\x}(\S)\|]\\
            &\leq G\sqrt{1+\frac{L^2}{\mu^2}}\epsilon
        \end{aligned}
    \end{equation}
    where equation $(a)$ owes to the symmetric data distribution of $\S$ and $\S^{(\bm{r})}$. And the inequality $(b)$ results from the fact that $\xi_{i,r_i}$ is independent from $\A_{\x}(\S^{(\bm{r})})$ and $F_{\S}(\A_{\x}(\S),\y^*_{\S})\leq\sup_{\y'\in\yset}F_{\S}(\A_{\x}(\S),\y')$. The inequality $(c)$ references the conclusion of Lemma 4.3 in \citet{lin2020gradient} that $\|\y^*_{\S^{(\bm{r})}}-\y^*_{\S}\|\leq\frac{L}{\mu}\|\A_{\x}(\S^{(\bm{r})})-\A_{\x}(\S)\|$ due to the strong concavity of $F(\x,\cdot)$. 

    Continuing with the second term, we have:
    \begin{equation}\label{eq:43}
        \E[\inf_{\x'\in\xset}R_{\S}(\x')-\inf_{\x'\in\xset}R(\x')]=\E[R_{\S}(\x^*)-R(\x^*)]+\E[\inf_{\x'\in\xset}R_{\S}(\x')-R_{\S}(\x^*)]\leq\E[R_{\S}(\x^*)-R(\x^*)]
    \end{equation}
    where the inequality owing to the fact that $\inf_{\x'\in\xset}R_{\S}(\x')\leq R_{\S}(\x^*)$.

    \begin{equation}\label{eq:44}
        \E[R_{\S}(\x^*)-R(\x^*)]=\E[F_{\S}(\x^*,\hat{\y}^*_{\S})-F(\x^*,\hat{\y}^*_{\S})]+\E[F(\x^*,\hat{\y}^*_{\S})-F(\x^*,\y^*)]\leq \E[F_{\S}(\x^*,\hat{\y}^*_{\S})-F(\x^*,\hat{\y}^*_{\S})]
    \end{equation}
    where the inequality is due to the property of saddle point $(\x^*,\y^*)$ that $F(\x^*,\hat{\y}^*_{\S})\leq F(\x^*,\y^*)$.

    Due to the symmetric data distribution of $\S,\S^{(\bm{r})}$ and the fact tha $\xi_{i,r_i}$ is independent from $\S^{\bm{r}}$, we have:
    \begin{equation}\label{eq:45}
        \begin{aligned}
            \E_{\A,\S}[F_{\S}(\x^*,\hat{\y}^*_{\S})-F(\x^*,\hat{\y}^*_{\S})]&=\frac{1}{n^m}\sum_{r_1=1}^n\cdots\sum_{r_m=1}^n \E_{\A,\S,\S'}[\frac{1}{m}\sum_{i=1}^mf_i(\x^*,\hat{\y}^*_{\S};\xi_{i,r_i})-f_i(\x^*,\hat{\y}^*_{\S^{(\bm{l})}};\xi_{i,r_i})]\\
            &\leq\frac{1}{n^m}\sum_{\bm{r}}\frac{1}{m}\sum_{i=1}^m G\E_{\A,\S,\S'}\|\hat{\y}^*_{\S}-\hat{\y}^*_{\S^{(\bm{r})}}\|.
        \end{aligned}
    \end{equation}

    Since $F_{\S}(\x,\cdot)$ is $\mu$-strongly concave, then for any given $\x$ we have:
    \begin{equation}\label{eq:46}
        \frac{\mu}{2}\|\hat{\y}^*_{\S}-\hat{\y}^*_{\S^{(\bm{r})}}\|^2\leq F_{\S}(\x^*,\hat{\y}^*_{\S})-F_{\S}(\x^*,\hat{\y}^*_{\S^{(\bm{r})}})+<\dy F_{\S}(\x^*,\hat{\y}^*_{\S}),\hat{\y}^*_{\S^{(\bm{r})}}-\hat{\y}^*_{\S}>\leq F_{\S}(\x^*,\hat{\y}^*_{\S})-F_{\S}(\x^*,\hat{\y}^*_{\S^{(\bm{r})}})
    \end{equation}
    where the last inequality follows from the fact that $\hat{\y}^*_{\S}$ is the maximizer of $F_{\S}(\x^*,\cdot)$ that $<\dy F_{\S}(\x^*,\hat{\y}^*_{\S}),\y-\hat{\y}^*_{\S}>\leq0,\forall \y$.

    On the other side, for a specific $\bm{r}=\{r_1,...,r_m\}$,
    \begin{equation}\label{eq:47}
        \begin{aligned}
            &\quad\;F_{\S}(\x^*,\hat{\y}^*_{\S})-F_{\S}(\x^*,\hat{\y}^*_{\S^{(\bm{r})}})\\
            &=\frac{1}{m}\sum_{i=1}^m\frac{1}{n}\sum_{l_i=1}^n\left(f_i(\x^*,\hat{\y}^*_{\S};\xi_{i,l_i})-f_i(\x^*,\hat{\y}^*_{\S^{(\bm{r})}};\xi_{i,l_i})\right)\\
            &=\frac{1}{mn}f_i(\x^*,\hat{\y}^*_{\S};\xi_{i,r_i})-f_i(\x^*,\hat{\y}^*_{\S^{(\bm{r})}};\xi_{i,r_i})+\frac{1}{m}\sum_{i=1}^m\frac{1}{n}\sum_{l_i\neq r_i}\left(f_i(\x^*,\hat{\y}^*_{\S};\xi_{i,l_i})-f_i(\x^*,\hat{\y}^*_{\S^{(\bm{r})}};\xi_{i,l_i})\right)\\
            &=\frac{1}{mn}\left(f_i(\x^*,\hat{\y}^*_{\S};\xi_{i,r_i})-f_i(\x^*,\hat{\y}^*_{\S^{(\bm{r})}};\xi_{i,r_i})\right)+\frac{1}{m}\sum_{i=1}^m\frac{1}{n}\sum_{l_i\neq r_i}\left(f_i(\x^*,\hat{\y}^*_{\S};\xi_{i,l_i})-f_i(\x^*,\hat{\y}^*_{\S^{(\bm{r})}};\xi_{i,l_i})\right)\\
            &\quad+\frac{1}{mn}\left(f_i(\x^*,\hat{\y}^*_{\S};\xi'_{i,r_i})-f_i(\x^*,\hat{\y}^*_{\S^{(\bm{r})}};\xi'_{i,r_i})\right)-\frac{1}{mn}\left(f_i(\x^*,\hat{\y}^*_{\S};\xi'_{i,r_i})-f_i(\x^*,\hat{\y}^*_{\S^{(\bm{r})}};\xi'_{i,r_i})\right)\\
            &=\frac{1}{mn}\!\left(f_i(\x^*\!\!,\hat{\y}^*_{\S};\xi_{i,r_i}\!)\!-\!f_i(\x^*\!\!,\hat{\y}^*_{\S^{(\bm{r})}};\xi_{i,r_i}\!)\right)\!+\!F_{\S^{(\bm{r})}}\!(\x^*\!\!,\hat{\y}^*_{\S})\!-\!F_{\S^{(\bm{r})}}(\x^*\!\!,\hat{\y}^*_{\S^{(\bm{r})}}\!)\\
            &\quad-\frac{1}{mn}\!\left(f_i(\x^*\!\!,\hat{\y}^*_{\S};\xi'_{i,r_i}\!)\!-\!f_i(\x^*\!\!,\hat{\y}^*_{\S^{(\bm{r})}};\xi'_{i,r_i})\right)\\
            &\leq\frac{2G}{mn}\|\hat{\y}^*_{\S}-\hat{\y}^*_{\S^{(\bm{r})}}\|
        \end{aligned}
    \end{equation}
    where the last inequality derives from the fact that $\hat{\y}^*_{\S^{(\bm{r})}}$ maximizes $F_{\S^{(\bm{r})}}(\x^*,\cdot)$.

    Combining inequality~\eqref{eq:46} and \eqref{eq:47}, we can obtain:
    \begin{equation}\label{eq:48}
        \|\hat{\y}^*_{\S}-\hat{\y}^*_{\S^{(\bm{r})}}\|\leq\frac{4G}{\mu mn}.
    \end{equation}

    Plugging inequality~\eqref{eq:48} into inequality~\eqref{eq:45}, we can derive:
    \begin{equation}\label{eq:49}
        \E_{\A,\S}[F_{\S}(\x^*,\hat{\y}^*_{\S})-F(\x^*,\hat{\y}^*_{\S})]\leq\frac{4G^2}{\mu mn}.
    \end{equation}

    Integrating inequalities~\eqref{eq:41},\eqref{eq:42},\eqref{eq:43},\eqref{eq:44},\eqref{eq:49}, we can acquire the excess primal generalization gap:
    \begin{equation*}
        \zeta^{\mathrm{ep}}_{\mathrm{gen}}(\A,\S)\leq G\sqrt{1+\frac{L^2}{\mu^2}}\epsilon+\frac{4G^2}{\mu mn}.
    \end{equation*}

    \textit{\textbf{Strong PD generalization gap~\ref{connection:strong}.}} Analysing the composition of the strong PD generalization gap:
    \begin{equation*}
        \begin{aligned}
            &\quad\zeta^{\mathrm{s}}_{\mathrm{gen}}(\A,\S)\\
            &=\!\!\E[(\!\sup_{\y'\!\in\yset}\!F(\A_{\x}(\S),\y')\!\!-\!\!\!\inf_{\x'\!\in\xset}\!\!F(\x'\!,\A_{\y}(\S)))]\!\!-\!\!\E[(\!\sup_{\y'\!\in\yset}\!\FS(\A_{\x}(\S)),\y')\!\!-\!\!\inf_{\x'\!\in\xset}\!\!\FS(\x'\!,\A_{\y}(\S)))]\\
            &=\!\!\E_{\A,\S}[\sup_{\y'\!\in\yset}\!F(\A_{\x}(\S),\y')\!-\!\sup_{\y'\!\in\yset}\!\FS(\A_{\x}(\S)),\y')]\!+\!\E_{\A,\S}[\inf_{\x'\!\in\xset}\!\!\FS(\x'\!,\A_{\y}(\S))\!-\!\inf_{\x'\!\in\xset}\!\!F(\x'\!,\A_{\y}(\S))].
        \end{aligned}
    \end{equation*}

    The first term has been proved in inequality~\eqref{eq:42} that:
    \begin{equation*}
        \E_{\A,\S}[\sup_{\y'\in\yset}F(\A_{\x}(\S),\y')-\sup_{\y'\in\yset}F_{\S}(\A_{\x}(\S),\y')]\leq G\sqrt{1+\frac{L^2}{\mu^2}}\E_{\A}[\sup_{\S,\S'}\|\A_{\x}(\S^{(\bm{r})})-\A_{\x}(\S)\|].
    \end{equation*}

    Analogously, we can get similar results on the other counterpart due to the strong convexity of $F$:
    \begin{equation*}
        \E_{\A,\S}[\inf_{\y'\in\yset}F(\A_{\x}(\S),\y')-\inf_{\y'\in\yset}F_{\S}(\A_{\x}(\S),\y')]\leq G\sqrt{1+\frac{L^2}{\mu^2}}\E_{\A}[\sup_{\S,\S'}\|\A_{\y}(\S^{(\bm{r})})-\A_{\y}(\S)\|].
    \end{equation*}

    Merge above inequalities, and we have the strong PD generalization gap as follows:
    \begin{equation*}
        \zeta^{\mathrm{s}}_{\mathrm{gen}}(\A,\S)\leq G\sqrt{2+\frac{2L^2}{\mu^2}}\epsilon.
    \end{equation*}

\end{proof}
\section{Proof in the (Strongly)Convex-(Strongly)Concave Case}\label{sec:sc-sc}
\subsection{Proof of Stability}
We use notations $(\x^t_{i,k},\y^t_{i,k})$ and $(\xdot^t_{i,k},\ydot^t_{i,k})$ to represent the $i$-th local model parameters of the distributed algorithm $\A(T,K,\W)$ traininig on the distributed neighboring dataset $\S$ and $\S'$ (see Definition~\ref{def:neighbouring:dataset}) on the $k$-th local update iteration during the $t$-th communication round respectively.

There are several steps to prove the stability of the distributed algorithm $\A(T,K,\W)$. First we concentrate on the single local update during the ($t$,$k$)-th iteration, as shown in Lemma~\ref{le:iteration}. Then we are going to bound the critical deviation term $\Delta^t_k$ as indicated in Lemma~\ref{le:delta:local} and Lemma~\ref{le:delta:localdecentra}. Finally we will evaluate the stability error where we use $(\A_{\S}(\x),\A_{\S}(\y))$ to represent the output model of algorithm $\A$ training on the dataset $\S$.

\begin{lem}\label{le:iteration}
    For distributed algorithm $\A(T,K,\W)$, under Assumption \ref{as:mixing}, \ref{as:Lip:con}, \ref{as:Lip:smmoth}, and each local function $f_i$ satisfies $\mu$-SC-SC. Then from the perspective of the averaged parameters, the iterative error from distributed neighboring datasets is shown as:
    $$\E_{\A}\left\|\left(\begin{array}{c}
                \bar{\x}^t_{k+1}-\bar{\xdot}^t_{k+1}   \\
                \bar{\y}^t_{k+1}-\bar{\ydot}^t_{k+1}
            \end{array}\right)\right\|\leq(1-\eta^t_k\frac{L\mu}{L+\mu})\E_{\A}\left\|\left(\begin{array}{c}
                \bar{\x}^t_k-\bar{\xdot}^t_k\\
                \bar{\y}^t_k-\bar{\ydot}^t_k
            \end{array}\right)\right\|+ 2\eta^t_kL\E_{\A}[\Delta^t_k]+\frac{2\eta^t_kG}{n}.$$
\end{lem}

\begin{proof}[\textbf{Proof of Lemma~\ref{le:iteration}}]
    Recalling the local update rule, we have the corresponding update rule for the global average on both algorithms:
    \begin{equation*}
        \begin{aligned}
            \left(\begin{array}{c}
               \bar{\x}^t_{k+1}\\
               \bar{\y}^t_{k+1} 
            \end{array}\right)&=\frac{1}{m}\sum_{i=1}^m\left(\begin{array}{c}
                \x^t_{i,k+1}\\
                \y^t_{i,k+1}
            \end{array}\right)\\&=\frac{1}{m}\sum_{i=1}^m\left(\begin{array}{c}
                \x^t_{i,k}-\eta^t_k\dx f_i(\x^t_{i,k},\y^t_{i,k};\xi_{i,j_k^t(i)})\\
                \y^t_{i,k}+\eta^t_k\dy f_i(\x^t_{i,k},\y^t_{i,k};\xi_{i,j_k^t(i)})
            \end{array}\right)\\
            &=\left(\begin{array}{cc}
                \bar{\x}^t_k\\
                \bar{\y}^t_k 
            \end{array}\right)-\frac{1}{m}\sum_{i=1}^m\left(\begin{array}{c}
                \eta^t_k\dx f_i(\x^t_{i,k},\y^t_{i,k};\xi_{i,j^t_k(i)})\\
                -\eta^t_k\dy f_i(\x^t_{i,k},\y^t_{i,k};\xi_{i,j^t_k(i)})
            \end{array}\right).
        \end{aligned}
    \end{equation*}
    
    On the $k$-th iteration of the $t$-th round,
    \begin{equation*}
        \begin{aligned}
            &\quad\;\left(\begin{array}{c}
        \bar{\x}^t_{k+1}-\bar{\xdot}^t_{k+1}\\
        \bar{\y}^t_{k+1}-\bar{\ydot}^t_{k+1}
        \end{array}
        \right)\\
        &=\left(\begin{array}{c}
            \bar{\x}^t_k-\bar{\xdot}^t_k-\frac{\eta^t_k}{m}\sum_{i=1}^m\left(\dx f_i(\x^t_{i,k},\y^t_{i,k};\xi_{i,j^t_{k}(i)})-\dx f_i(\xdot^t_{i,k},\ydot^t_{i,k};\xi_{i,j^t_k(i)})\right)\\
            \bar{\y}^t_k-\bar{\ydot}^t_k+\frac{\eta^t_k}{m}\sum_{i=1}^m\left(\dy f_i(\x^t_{i,k},\y^t_{i,k};\xi_{i,j^t_{k}(i)})-\dy f_i(\xdot^t_{i,k},\ydot^t_{i,k};\xi_{i,j^t_k(i)})\right)
        \end{array}\right).
        \end{aligned}
    \end{equation*}
    
    Assuming there will be $m'$ agents encountering the \textit{different} sample among the whole $m$ agents in every single iteration, we can suppose they are the latter $m'$ agents whereas the former $m-m'$ agents select the same samples without loss of generality. Then we can proceed as follows:
    \begin{equation*}
        \begin{aligned}
            &\quad\;\left(\begin{array}{c}
        \bar{\x}^t_{k+1}-\bar{\xdot}^t_{k+1}\\
        \bar{\y}^t_{k+1}-\bar{\ydot}^t_{k+1}
        \end{array}
        \right)\\
        &=\frac{1}{m}\sum_{i=1}^{m-m'}\left(\begin{array}{c}
            \bar{\x}^t_k-\eta^t_k\dx f_i(\bar{\x}^t_k,\bar{\y}^t_k;\xi_{i,j^t_k(i)})-\left(\bar{\xdot}^t_k-\eta^t_k\dx f_i(\bar{\xdot}^t_k,\bar{\ydot}^t_k;\xi_{i,j^t_k(i)})\right)\\
            \bar{\y}^t_k+\eta^t_k\dy f_i(\bar{\x}^t_k,\bar{\y}^t_k;\xi_{i,j^t_k(i)})-\left(\bar{\ydot}^t_k+\eta^t_k\dy f_i(\bar{\xdot}^t,\bar{\ydot}^t;\xi_{i,j^t_k(i)})\right)
        \end{array}
        \right)\\
        &+\frac{1}{m}\sum_{i=m-m'+1}^m\left(\begin{array}{c}
            \bar{\x}^t_k-\eta^t_k\dx f_i(\bar{\x}^t_k,\bar{\y}^t_k;\xi_{i,n})-\left(\bar{\xdot}^t_k-\eta^t_k\dx f_i(\bar{\xdot}^t_k,\bar{\ydot}^t_k;\xi'_{i,n})\right)\\
            \bar{\y}^t_k+\eta^t_k\dy f_i(\bar{\x}^t_k,\bar{\y}^t_k;\xi_{i,n})-\left(\bar{\ydot}^t_k+\eta^t_k\dy f_i(\bar{\xdot}^t,\bar{\ydot}^t;\xi'_{i,n})\right)
        \end{array}
        \right)\\
        &+\frac{1}{m}\sum_{i=1}^{m-m'}\left(\begin{array}{c}
            -\eta^t_k\left(\dx f_i(\x^t_{i,k},\y^t_{i,k};\xi_{i,j^t_{k}(i)})-\dx f_i(\bar{\x}^t_k,\bar{\y}^t_k;\xi_{i,j^t_k(i)})\right)\\
            \eta^t_k\left(\dy f_i(\x^t_{i,k},\y^t_{i,k};\xi_{i,j^t_{k}(i)})-\dy f_i(\bar{\x}^t_k,\bar{\y}^t_k;\xi_{i,j^t_k(i)})\right)
        \end{array}
        \right)\\
        &+\frac{1}{m}\sum_{i=1}^{m-m'}\left(\begin{array}{c}
            \eta^t_k\left(\dx f_i(\xdot^t_{i,k},\ydot^t_{i,k};\xi_{i,j^t_{k}(i)})-\dx f_i(\bar{\xdot}^t_k,\bar{\ydot}^t_k;\xi_{i,j^t_k(i)})\right)\\
            -\eta^t_k\left(\dy f_i(\xdot^t_{i,k},\ydot^t_{i,k};\xi_{i,j^t_{k}(i)})-\dy f_i(\bar{\xdot}^t_k,\bar{\ydot}^t_k;\xi_{i,j^t_k(i)})\right)
        \end{array}
        \right)\\
        &+\frac{1}{m}\sum_{i=m-m'+1}^m\left(\begin{array}{c}
            -\eta^t_k\left(\dx f_i(\x^t_{i,k},\y^t_{i,k};\xi_{i,n})-\dx f_i(\bar{\x}^t_k,\bar{\y}^t_k;\xi_{i,n})\right)\\
            \eta^t_k\left(\dy f_i(\x^t_{i,k},\y^t_{i,k};\xi_{i,n})-\dy f_i(\bar{\x}^t_k,\bar{\y}^t_k;\xi_{i,n})\right)
        \end{array}
        \right)\\
        &+\frac{1}{m}\sum_{i=m-m'+1}^m\left(\begin{array}{c}
            \eta^t_k\left(\dx f_i(\xdot^t_{i,k},\ydot^t_{i,k};\xi'_{i,n})-\dx f_i(\bar{\xdot}^t_k,\bar{\ydot}^t_k;\xi'_{i,n})\right)\\
            -\eta^t_k\left(\dy f_i(\xdot^t_{i,k},\ydot^t_{i,k};\xi'_{i,n})-\dy f_i(\bar{\xdot}^t_k,\bar{\ydot}^t_k;\xi'_{i,n})\right)
        \end{array}
        \right).
        \end{aligned}
    \end{equation*}

    All the terms above can be bounded according to Lemma~\ref{le:nonexpan} along with the Lipschitz continuous assumption~\ref{as:Lip:con}, and it turns out to be:
    \begin{equation*}
        \begin{aligned}
            &\quad\;\left\|\left(\begin{array}{c}
        \bar{\x}^t_{k+1}-\bar{\xdot}^t_{k+1}\\
        \bar{\y}^t_{k+1}-\bar{\ydot}^t_{k+1}
        \end{array}
        \right)\right\|\\
            &\leq(1-\eta^t_k\frac{L\mu}{L+\mu})\left\|\left(\begin{array}{c}
                \bar{\x}^t_k-\bar{\xdot}^t_k\\
                \bar{\y}^t_k-\bar{\ydot}^t_k
            \end{array}\right)\right\|+\frac{2\eta^t_kGm'}{m}+ \frac{\eta^t_kL}{m}\sum_{i=1}^m\left(\left\|\left(\begin{array}{c}
                \x^t_{i,k}-\bar{\x}^t_k\\
                \y^t_{i,k}-\bar{\y}^t_k
            \end{array}\right)\right\|+\left\|\left(\begin{array}{c}
                \xdot^t_{i,k}-\bar{\xdot}^t_k\\
                \ydot^t_{i,k}-\bar{\ydot}^t_k
            \end{array}\right)\right\|\right).
        \end{aligned}
    \end{equation*}

    Taking expectation on the randomness of samples chosen during each iteration, we can derive:
    \begin{equation*}
        \begin{aligned}
            &\quad\;\E_{\A}\left\|\left(\begin{array}{c}
                \bar{\x}^t_{k+1}-\bar{\xdot}^t_{k+1}   \\
                \bar{\y}^t_{k+1}-\bar{\ydot}^t_{k+1}
            \end{array}\right)\right\|\\
            &\leq \sum_{m'=0}^m \binom{m}{m'}(\frac{1}{n})^{m'}(1-\frac{1}{n})^{m-m'}\left\{(1-\eta^t_k\frac{L\mu}{L+\mu})\E_{\A}\left\|\left(\begin{array}{c}
                \bar{\x}^t_k-\bar{\xdot}^t_k\\
                \bar{\y}^t_k-\bar{\ydot}^t_k
            \end{array}\right)\right\|+\frac{2\eta^t_kGm'}{m}+ 2\eta^t_kL\E_{\A}[\Delta^t_k]\right\}\\
            &= (1-\eta^t_k\frac{L\mu}{L+\mu})\E_{\A}\left\|\left(\begin{array}{c}
                \bar{\x}^t_k-\bar{\xdot}^t_k\\
                \bar{\y}^t_k-\bar{\ydot}^t_k
            \end{array}\right)\right\|+ 2\eta^t_kL\E_{\A}[\Delta^t_k]+\frac{2\eta^t_kG}{n}
        \end{aligned}
    \end{equation*}
    where each node encounters the \textit{different} sample with a probability of $\frac{1}{n}$ and the last equation is due to the fact that: $\binom{m}{m'}\cdot\frac{m'}{m}=\binom{m-1}{m'-1}$.
\end{proof}

\begin{proof}[\textbf{Proof of Theorem~\ref{thm:local:sc-sc-stability}}]
    For the output of the Distributed SGDA (see Algorithm~\ref{alg:distributedsgda}):
    $$\left(\begin{array}{c}
        \x^T \\
        \y^T 
    \end{array}\right)=\left(\begin{array}{c}
        \frac{1}{m}\sum_{i=1}^m\x_i^T \\
        \frac{1}{m}\sum_{i=1}^m\y_i^T
    \end{array}\right)=\left(\begin{array}{c}
        \frac{1}{m}\sum_{i=1}^m\sum_{h=1}^m\omega_{ih}\x^T_{h,K}   \\
        \frac{1}{m}\sum_{i=1}^m\sum_{h=1}^m\omega_{ih}\y^T_{h,K} 
    \end{array}\right)=\left(\begin{array}{c}
        \frac{1}{m}\sum_{h=1}^m\x^T_{h,K}  \\
        \frac{1}{m}\sum_{h=1}^m\y^T_{h,K}
    \end{array}\right)=\left(\begin{array}{c}
        \bar{\x}^T_K \\
        \bar{\y}^T_K
    \end{array}\right).$$
    
    The identification between the beginning of a communication round and the end of the last communication round is shown as:
    $$\left(\begin{array}{c}
        \bar{\x}^t_0 \\
        \bar{\y}^t_0
    \end{array}\right)=\left(\begin{array}{c}
        \frac{1}{m}\sum_{i=1}^m\x^t_{i,0} \\
        \frac{1}{m}\sum_{i=1}^m\y^t_{i,0}
    \end{array}\right)=\left(\begin{array}{c}
        \frac{1}{m}\sum_{i=1}^m\x^{t-1}_i \\
        \frac{1}{m}\sum_{i=1}^m\y^{t-1}_i 
    \end{array}\right)=\left(\begin{array}{c}
        \frac{1}{m}\sum_{i=1}^m\sum_{h=1}^m\omega_{ih}\x^{t-1}_{h,K}\\
        \frac{1}{m}\sum_{i=1}^m\sum_{h=1}^m\omega_{ih}\y^{t-1}_{h,K}
    \end{array}\right)=\left(\begin{array}{c}
        \bar{\x}^{t-1}_K \\
        \bar{\y}^{t-1}_K 
    \end{array}\right).$$
    
    Recursively repeating the inequality in Lemma~\ref{le:iteration} from $T$ to $0$:
    \begin{equation}\label{eq:11}
        \begin{aligned}
            &\quad\;\E_{\A}\left\|\left(\begin{array}{c}
                \bar{\x}^T-\bar{\xdot}^T \\
                \bar{\y}^T-\bar{\ydot}^T
            \end{array}\right)\right\|\\
            &=\E_{\A}\left\|\left(\begin{array}{c}
                \bar{\x}^T_K-\bar{\xdot}^T_K \\
                \bar{\y}^T_K-\bar{\ydot}^T_K
            \end{array}\right)\right\|\\
            &\leq (1-\eta^T_{K-1}\frac{L\mu}{L+\mu})\E_{\A}\left\|\left(\begin{array}{c}
                \bar{\x}^T_{K-1}-\bar{\xdot}^T_{K-1}\\
                \bar{\y}^T_{K-1}-\bar{\ydot}^T_{K-1}
            \end{array}\right)\right\|+ 2\eta^T_{K-1}L\E_{\A}[\Delta^T_{K-1}]+\frac{2\eta^T_{K-1}G}{n}\\
            &\leq\cdots\\
            &\leq\prod_{k=0}^{K-1}(1-\eta^T_k\frac{L\mu}{L+\mu})\E_{\A}\left\|\left(\begin{array}{c}
                \bar{\x}^T_0-\bar{\xdot}^T_0\\
                \bar{\y}^T_0-\bar{\ydot}^T_0
            \end{array}\right)\right\|+2\sum_{k=0}^{K-1}\eta^T_k(L\E_{\A}[\Delta^T_k]+\frac{G}{n})\prod_{k'=k+1}^{K-1}(1-\eta^T_{k'}\frac{L\mu}{L+\mu})\\
            &=\prod_{k=0}^{K-1}(1-\eta^T_k\frac{L\mu}{L+\mu})\E_{\A}\left\|\left(\begin{array}{c}
                \bar{\x}^{T-1}_K-\bar{\xdot}^{T-1}_K\\
                \bar{\y}^{T-1}_K-\bar{\ydot}^{T-1}_K
            \end{array}\right)\right\|+2\sum_{k=0}^{K-1}\eta^T_k(L\E_{\A}[\Delta^T_k]+\frac{G}{n})\prod_{k'=k+1}^{K-1}(1-\eta^T_{k'}\frac{L\mu}{L+\mu})\\
            &\leq\cdots\\
            &\leq\prod_{t=1}^T\prod_{k=0}^{K-1}(1-\eta^t_k\frac{L\mu}{L+\mu})\E_{\A}\left\|\left(\begin{array}{c}
                \bar{\x}^0_K-\bar{\xdot}^0_K\\
                \bar{\y}^0_K-\bar{\ydot}^0_K
            \end{array}\right)\right\|\\
            &\quad\;+2\sum_{t=1}^T\prod_{t'=t+1}^T\prod_{k_1=0}^{K-1}(1-\eta^{t'}_{k_1}\frac{L\mu}{L+\mu})\sum_{k=0}^{K-1}\eta^t_k(L\E_{\A}[\Delta^t_k]+\frac{G}{n})\prod_{k'=k+1}^{K-1}(1-\eta^t_{k'}\frac{L\mu}{L+\mu})\\
            &=2\sum_{t=1}^T\prod_{t'=t+1}^T\prod_{k_1=0}^{K-1}(1-\eta^{t'}_{k_1}\frac{L\mu}{L+\mu})\sum_{k=0}^{K-1}\eta^t_k(L\E_{\A}[\Delta^t_k]+\frac{G}{n})\prod_{k'=k+1}^{K-1}(1-\eta^t_{k'}\frac{L\mu}{L+\mu})
        \end{aligned}
    \end{equation}
    where the initialization of parameters accounts for the last equation.

    \begin{enumerate}[wide=\parindent,label=(\roman*)]
        \item \textbf{When learning rates are fixed,} from condition~\ref{fixed} in Lemma~\ref{le:delta:localdecentra}, it follows that:
        $$\E_{\A}[\Delta^t_k]\leq\eta G\left(\sqrt{\frac{1}{1-\lambdabm}}k+\sqrt{\frac{2\lambdabm}{(1-\lambdabm)(1-\lambdabm^2)}}K\right).$$
        
        Therefore we have the argument stability as follows:
        \begin{equation*}
            \begin{aligned}
                &\quad\;\E_{\A}\left\|\left(\begin{array}{c}
                    \A_{\x}(\S)-\A_{\x}(\S')  \\
                    \A_{\y}(\S)-\A_{\y}(\S') 
                \end{array}\right)\right\|\\
                &\leq2\sum_{t=1}^T(1-\eta\frac{L\mu}{L+\mu})^{K(T-t)}\sum_{k=0}^{K-1}\eta(L\E_{\A}[\Delta^t_k]+\frac{G}{n})\\
                &\leq 2\sum_{t=1}^T(1-\eta\frac{L\mu}{L+\mu})^{K(T-t)}\eta \sum_{k=0}^{K-1}\left(\eta GL\sqrt{\frac{1}{1-\lambdabm}}k+\eta GL\sqrt{\frac{2\lambdabm}{(1-\lambdabm)(1-\lambdabm^2)}}K+\frac{G}{n}\right)\\
                &\leq2\sum_{t=1}^T(1-\eta\frac{L\mu}{L+\mu})^{K(T-t)}\eta\left(\eta GL\sqrt{\frac{1}{1-\lambdabm}}\frac{K^2}{2}+\eta GL\sqrt{\frac{2\lambdabm}{(1-\lambdabm)(1-\lambdabm^2)}}K^2+\frac{G}{n}K\right)\\
                &\leq\frac{2}{1-(1-\eta\frac{L\mu}{L+\mu})^K}\eta\left(\eta GL\sqrt{\frac{1}{1-\lambdabm}}\frac{K^2}{2}+\eta GL\sqrt{\frac{2\lambdabm}{(1-\lambdabm)(1-\lambdabm^2)}}K^2+\frac{G}{n}K\right)\\
                &\leq\frac{2G(L+\mu)}{L\mu}\left(\eta L\sqrt{\frac{1}{1-\lambdabm}}\frac{K^2}{2}+\eta L\sqrt{\frac{2\lambdabm}{(1-\lambdabm)(1-\lambdabm^2)}}K^2+\frac{K}{n}\right).
            \end{aligned}
        \end{equation*}

        \item \textbf{When learning rates are decaying that $\eta^t_k=\frac{1}{(k+1)^\alpha t^\beta}$,} condition~\ref{decaying} implies that:
        $$\E_{\A}[\Delta^t_k]\leq\frac{G}{t^\beta}\left(\sqrt{\frac{k(1\!+\!\ln{k}\cdot\mathbf{1}_{\alpha=\frac{1}{2}}\!+\!\frac{1}{2\alpha-1}\cdot\mathbf{1}_{\alpha>\frac{1}{2}})}{1-\lambdabm}}\!+\!\sqrt{(C_{\lambdabm^2}\!+\!\frac{C_{\lambdabm}}{1-\lambdabm})K(1\!+\!\ln{K}\cdot\mathbf{1}_{\alpha=\frac{1}{2}}\!+\!\frac{1}{2\alpha-1}\cdot\mathbf{1}_{\alpha>\frac{1}{2}})}\right).$$
        
        The summation is shown as:
        \begin{equation*}
            \sum_{k=0}^{K-1}\frac{1}{(k+1)^\alpha}\leq1+\left\{
            \begin{array}{rc}
                \frac{K^{1-\alpha}}{1-\alpha} &  0<\alpha<1\\
                \ln{K}& \alpha=1\\
                \frac{1}{\alpha-1} & \alpha>1
            \end{array}\right. .
        \end{equation*}
        
        And the product is calculated as:
        \begin{equation*}
            \begin{aligned}
                \prod_{t'=t+1}^T\prod_{k_1=0}^{K-1}(1-\frac{1}{(k_1+1)^\alpha t'^\beta}\frac{L\mu}{L+\mu})&\leq\prod_{t'=t+1}^T\prod_{k_1=0}^{K-1}e^{-\frac{1}{(k_1+1)^\alpha t'^\beta}\frac{L\mu}{L+\mu}}\\
                &=e^{-\frac{L\mu}{L+\mu}\sum_{t'=t+1}^T\sum_{k_1=0}^{K-1}\frac{1}{(k_1+1)^\alpha t'^\beta}}\\
                &=e^{-\frac{L\mu}{L+\mu}\sum_{t'=t+1}^T\frac{1}{t'^\beta}\sum_{k_1=0}^{K-1}\frac{1}{(k_1+1)^\alpha}}\\
                &\overset{\beta=1}{\leq}e^{-\frac{L\mu}{L+\mu}\ln{\frac{T}{t}}\sum_{k_1=0}^{K-1}\frac{1}{(k_1+1)^\alpha}}\\
                &\leq(\frac{t}{T})^{\frac{L\mu}{L+\mu}\sum_{k_1=0}^{K-1}\frac{1}{(k_1+1)^\alpha}}.
            \end{aligned}
        \end{equation*}
        
        Substituting all above terms into inequality~\eqref{eq:11}, we can derive:
        \begin{equation*}
            \begin{aligned}
                &\quad\;\E_{\A}\left\|\left(\begin{array}{c}
                    \A_{\x}(\S)-\A_{\x}(\S')  \\
                    \A_{\y}(\S)-\A_{\y}(\S') 
                \end{array}\right)\right\|\\
                &\overset{\frac{1}{2}<\alpha<1}{\leq}2\sum_{t=1}^T(\frac{t}{T})^{\frac{L\mu}{L+\mu}(1+\frac{K^{1-\alpha}}{1-\alpha})}\sum_{k=0}^{K-1}\frac{1}{(k+1)^\alpha t}\left(\frac{GL}{t}\left(\sqrt{\frac{\frac{2\alpha}{2\alpha-1}k}{1-\lambdabm}}+\sqrt{(C_{\lambdabm^2}+\frac{C_{\lambdabm}}{1-\lambdabm})\frac{2\alpha}{2\alpha-1}K}\right)+\frac{G}{n}\right)\\
                &=\frac{2}{T^{\frac{L\mu}{L+\mu}(1+\frac{K^{1-\alpha}}{1-\alpha})}}\left\{\sum_{t=1}^T\sum_{k=0}^{K-1}\frac{t^{\frac{L\mu}{L+\mu}(1+\frac{K^{1-\alpha}}{1-\alpha})-2}}{(k+1)^{\alpha-\frac{1}{2}}}GL\sqrt{\frac{\frac{2\alpha}{2\alpha-1}}{1-\lambdabm}}\right.\\
                &\left.+\sum_{t=1}^T\sum_{k=0}^{K-1}\frac{t^{\frac{L\mu}{L+\mu}(1+\frac{K^{1-\alpha}}{1-\alpha})-2}}{(k+1)^\alpha}GL\sqrt{(C_{\lambdabm^2}+\frac{C_{\lambdabm}}{1-\lambdabm})\frac{2\alpha}{2\alpha-1}K}+\sum_{t=1}^T\sum_{k=0}^{K-1}\frac{t^{\frac{L\mu}{L+\mu}(1+\frac{K^{1-\alpha}}{1-\alpha})-1}}{(k+1)^\alpha}\frac{G}{n}\right\}.
            \end{aligned}
        \end{equation*}
        
        We further evaluate the summation terms when $\frac{L\mu}{L+\mu}(1+\frac{K^{1-\alpha}}{1-\alpha})>1$:
        \begin{gather*}
            \sum_{t=1}^T\sum_{k=0}^{K-1}\frac{t^{\frac{L\mu}{L+\mu}(1+\frac{K^{1-\alpha}}{1-\alpha})-2}}{(k+1)^{\alpha-\frac{1}{2}}}=\sum_{t=1}^T t^{\frac{L\mu}{L+\mu}(1+\frac{K^{1-\alpha}}{1-\alpha})-2}\sum_{k=0}^{K-1}\frac{1}{(k+1)^{\alpha-\frac{1}{2}}}\leq\frac{(T+1)^{\frac{L\mu}{L+\mu}(1+\frac{K^{1-\alpha}}{1-\alpha})-1}}{\frac{L\mu}{L+\mu}(1+\frac{K^{1-\alpha}}{1-\alpha})-1}\left(1+\frac{K^{\frac{3}{2}-\alpha}}{\frac{3}{2}-\alpha}\right),\\
            \sum_{t=1}^T\sum_{k=0}^{K-1}\frac{t^{\frac{L\mu}{L+\mu}(1+\frac{K^{1-\alpha}}{1-\alpha})-2}}{(k+1)^\alpha}\leq\frac{(T+1)^{\frac{L\mu}{L+\mu}(1+\frac{K^{1-\alpha}}{1-\alpha})-1}}{\frac{L\mu}{L+\mu}(1+\frac{K^{1-\alpha}}{1-\alpha})-1}\left(1+\frac{K^{1-\alpha}}{1-\alpha}\right),\\
            \sum_{t=1}^T\sum_{k=0}^{K-1}\frac{t^{\frac{L\mu}{L+\mu}(1+\frac{K^{1-\alpha}}{1-\alpha})-1}}{(k+1)^\alpha}\leq\frac{(T+1)^{\frac{L\mu}{L+\mu}(1+\frac{K^{1-\alpha}}{1-\alpha})}}{\frac{L\mu}{L+\mu}(1+\frac{K^{1-\alpha}}{1-\alpha})}\left(1+\frac{K^{1-\alpha}}{1-\alpha}\right)=\frac{L+\mu}{L\mu}(T+1)^{\frac{L\mu}{L+\mu}(1+\frac{K^{1-\alpha}}{1-\alpha})}.
        \end{gather*}
        
        Eventually, we can get:
        \begin{equation*}
            \begin{aligned}
                &\quad\;\E_{\A}\left\|\left(\begin{array}{c}
                    \A_{\x}(\S)-\A_{\x}(\S')  \\
                    \A_{\y}(\S)-\A_{\y}(\S') 
                \end{array}\right)\right\|\\
                &\leq\frac{1}{T+1}\bigg(\frac{K^{\frac{3}{2}-\alpha}+\frac{3}{2}-\alpha}{\frac{L\mu}{L+\mu}(K^{1-\alpha}+1-\alpha)+\alpha-1}\sqrt{\frac{\frac{2\alpha}{2\alpha-1}}{1-\lambdabm}}\frac{2GL(1-\alpha)}{(\frac{3}{2}-\alpha)}\\
                &+\frac{K^{1-\alpha}+1-\alpha}{\frac{L\mu}{L+\mu}(K^{1-\alpha}+1-\alpha)+\alpha-1}2GL\sqrt{(C_{\lambdabm^2}+\frac{C_{\lambdabm}}{1-\lambdabm})\frac{2\alpha}{2\alpha-1}K}\bigg)+\frac{G}{n}\frac{L+\mu}{L\mu}.
            \end{aligned}
        \end{equation*}
    \end{enumerate}
\end{proof}

\subsection{Proof of Weak PD Population Risk}
\subsubsection{Proof of Optimization Error}
\begin{thm}[\textbf{Weak PD empirical risk}]\label{thm:weakpdempirical}
    Assuming the model parameters satisfying that $\sup_{\x\in\xset}\|\x||\leq B_{\x},\sup_{\y\in\yset}\|\y\|\leq B_{\y}$, defining $(\widetilde{\x}^T_K,\widetilde{\y}^T_K)\triangleq(\frac{1}{TK}\sum_{t=1}^T\sum_{k=0}^{K-1}\bar{\x}^t_k,\frac{1}{TK}\sum_{t=1}^T\sum_{k=0}^{K-1}\bar{\y}^t_k)$ as the averaged output of the algorithm $\A(T,K,\W)$, under Assumption \ref{as:mixing}, \ref{as:Lip:con}, \ref{as:Lip:smmoth}, when each local function $f_i$ is $\mu$-SC-SC, then we have the following weak PD empirical risk $\Delta^{\!\mathrm{w}}_{\S}(\x,\y)=\sup_{y'}\E[F_{\S}(\x,\y')]-\inf_{\x'}\E[F_{\S}(\x',\y)]$ regarding different learning rates :
    \begin{enumerate}[wide=\parindent,label=(\roman*)]
        \item\label{op:fixed} for fixed learning rates,
        \begin{equation*}
            \begin{aligned}
                &\quad\;\Delta^{\!\mathrm{w}}_{\S}(\widetilde{\x}^T_K,\widetilde{\y}^T_K)\\
                &\leq\frac{1-\eta}{2\eta TK}(B_{\x}^2+B_{\y}^2)+\eta G^2+(B_{\x}+B_{\y})\frac{2G}{\sqrt{TK}}+\eta GK(B_{\x}+B_{\y})(1+\frac{L}{2})(\frac{1}{2}\sqrt{\frac{1}{1-\lambdabm}}+\sqrt{\frac{2\lambdabm}{(1-\lambdabm)(1-\lambdabm^2)}});
            \end{aligned}
        \end{equation*}
        \item\label{op:decaying} for decaying learning rates that $\eta^t_k=\frac{1}{(k+1)^\alpha t^\beta}$, requiring $\frac{1}{2}<\alpha\leq1,\beta\leq1$,
        \begin{equation*}
            \begin{aligned}
                &\quad\;\Delta^{\!\mathrm{w}}_{\S}(\widetilde{\x}^T_K,\widetilde{\y}^T_K)\\
                &\leq\frac{G^2}{TK}(\frac{T^{1-\beta}}{1-\beta}\1_{\beta<1}+(1+\ln{T})\1_{\beta=1})(\frac{K^{1-\alpha}}{1-\alpha}\1_{\alpha<1}+(1+\ln{K})\1_{\alpha=1})\\
                &+\frac{(B_{\x}\!+\!B_{\y})G(L+2)(\sqrt{\frac{1}{1-\lambdabm}\frac{2\alpha}{2\alpha-1}}\!+\!\sqrt{(C_{\lambdabm^2}\!+\!\frac{C_{\lambdabm}}{1-\lambdabm})\frac{2\alpha}{2\alpha-1}})}{2T}K^{\frac{1}{2}}(\frac{T^{1-\beta}}{1-\beta}\1_{\beta<1}+(1+\ln{T})\1_{\beta=1})+\frac{2G(B_{\x}\!+\!B_{\y})}{\sqrt{TK}}.
            \end{aligned}
        \end{equation*}
    \end{enumerate}
\end{thm}

\begin{proof}[\textbf{Proof of Theorem~\ref{thm:weakpdempirical}}]
    For arbitrary $\x$, we have:
    \begin{equation*}
        \begin{aligned}
            &\quad\;\|\bar{\x}^t_{k+1}-\x\|^2\\
            &=\|\bar{\x}^t_k-\frac{\eta^t_k}{m}\sum_{i=1}^m\dx f_i(\x^t_{i,k},\y^t_{i,k};\xi_{i,j^t_k(i)})-\x\|^2\\
            &=\|\bar{\x}^t_k-\x\|^2+2<\x-\bar{\x}^t_k,\frac{\eta^t_k}{m}\sum_{i=1}^m\dx f_i(\x^t_{i,k},\y^t_{i,k};\xi_{i,j^t_k(i)})>+\eta^{t\,2}_k\|\frac{1}{m}\sum_{i=1}^m\dx f_i(\x^t_{i,k},\y^t_{i,k};\xi_{i,j^t_k(i)})\|^2\\
            &\leq\|\bar{\x}^t_k-\x\|^2+2\eta^t_k<\x-\bar{\x}^t_k,\dx\FS(\bar{\x}^t_k,\bar{\y}^t_k)>\\
            &\quad+2\eta^t_k<\x-\bar{\x}^t_k,\frac{1}{m}\sum_{i=1}^m\dx f_i(\x^t_{i,k},\y^t_{i,k};\xi_{i,j^t_k(i)})-\dx\FS(\bar{\x}^t_k,\bar{\y}^t_k)>+\eta^{t\,2}_kG^2\\
            &\leq\|\bar{\x}^t_k-\x\|^2+2\eta^t_k\left(\FS(\x,\bar{\y}^t_k)-\FS(\bar{\x}^t_k,\bar{\y}^t_k)-\frac{\mu}{2}\|\x-\bar{\x}^t_k\|^2\right)\\
            &\quad+2\eta^t_k<\x-\bar{\x}^t_k,\frac{1}{m}\sum_{i=1}^m\dx f_i(\x^t_{i,k},\y^t_{i,k};\xi_{i,j^t_k(i)})-\dx\FS(\bar{\x}^t_k,\bar{\y}^t_k)>+\eta^{t\,2}_kG^2
        \end{aligned}
    \end{equation*}
    where the second-to-last inequality is due to the Lipschitz continuity of $f_i$ (see Assumption~\ref{as:Lip:con}) and we make use of the property of convexity on the function $\FS(\cdot,\y)=\frac{1}{m}\sum_{i=1}^m\frac{1}{n}\sum_{l_i=1}^nf_i(\cdot,\y;\xi_{i,l_i})$ (see Definition~\ref{def:convexconcave}) in the last inequality.

    Then we rearrange the above inequality and sum up from $t=1,k=0$ to $t=T,k=K-1$, based on the fact that $\bar{\x}^t_0=\bar{\x}^{t-1}_K$ it follows that:
    \begin{equation*}
        \begin{aligned}
            &\quad\;\sum_{t=1}^T\sum_{k=0}^{K-1}2\eta^t_k(\FS(\bar{\x}^t_k,\bar{\y}^t_k)-\FS(\x,\bar{\y}^t_k))\\
            &\leq\|\bar{\x}^1_0-\x\|^2-\|\bar{\x}^T_K-\x\|^2+2\sum_{t=1}^T\sum_{k=0}^{K-1}\eta^t_k<\x-\bar{\x}^t_k,\frac{1}{m}\sum_{i=1}^m\dx f_i(\x^t_{i,k},\y^t_{i,k};\xi_{i,j^t_k(i)})-\dx\FS(\bar{\x}^t_k,\bar{\y}^t_k)>\\
            &\quad+\sum_{t=1}^T\sum_{k=0}^{K-1}\eta^{t\,2}_kG^2-\sum_{t=1}^T\sum_{k=0}^{K-1}\mu\eta^t_k\|\x-\bar{\x}^t_k\|^2.
        \end{aligned}
    \end{equation*}

    Taking expectation on the randomness of the algorithm $\A$, we can obtain:
    \begin{equation}\label{eq:21}
        \begin{aligned}
            &\quad\;2\sum_{t=1}^T\sum_{k=0}^{K-1}\eta^t_k\E_{\A}[\FS(\bar{\x}^t_k,\bar{\y}^t_k)-\FS(\x,\bar{\y}^t_k)]\\
            &\leq(1-\eta^1_0)B_{\x}^2+\sum_{t=1}^T\sum_{k=0}^{K-1}\eta^{t\,2}_kG^2+2B_{\x}\E_{\A}\|\sum_{t=1}^T\sum_{k=0}^{K-1} \eta^t_k\left(\frac{1}{m}\sum_{i=1}^m\dx f_i(\x^t_{i,k},\y^t_{i,k};\xi_{i,j^t_k(i)})-\dx\FS(\bar{\x}^t_k,\bar{\y}^t_k)\right)\|\\
            &\quad+\sum_{t=1}^T\sum_{k=0}^{K-1}\eta^t_k\E_{\A}<\bar{\x}^t_k,\dx\FS(\bar{\x}^t_k,\bar{\y}^t_k)-\frac{1}{m}\sum_{i=1}^m\dx f_i(\x^t_{i,k},\y^t_{i,k};\xi_{i,j^t_k(i)})>.
        \end{aligned}
    \end{equation}

    For the cross-product term, we evaluate it as:
    \begin{equation}\label{eq:22}
        \begin{aligned}
            &\quad\;\E_{\A}<\bar{\x}^t_k,\dx\FS(\bar{\x}^t_k,\bar{\y}^t_k)-\frac{1}{m}\sum_{i=1}^m\dx f_i(\x^t_{i,k},\y^t_{i,k};\xi_{i,j^t_k(i)})>\\
            &=\E_{\A}<\bar{\x}^t_k,\frac{1}{m}\sum_{i=1}^m\frac{1}{n}\sum_{l_i=1}^n\dx f_i(\bar{\x}^t_k,\bar{\y}^t_k;\xi_{i,l_i})-\frac{1}{m}\sum_{i=1}^m\dx f_i(\bar{\x}^t_k,\bar{\y}^t_k;\xi_{i,j^t_k(i)})>\\
            &\quad+\E_{\A}<\bar{\x}^t_k,\frac{1}{m}\sum_{i=1}^m\dx f_i(\bar{\x}^t_k,\bar{\y}^t_k;\xi_{i,j^t_k(i)})-\frac{1}{m}\sum_{i=1}^m\dx f_i(\x^t_{i,k},\y^t_{i,k};\xi_{i,j^t_k(i)})>\\
            &\leq0+B_{\x}\E_{\A}\|\frac{1}{m}\sum_{i=1}^m\left(\dx f_i(\bar{\x}^t_k,\bar{\y}^t_k;\xi_{i,j^t_k(i)})-\dx f_i(\x^t_{i,k},\y^t_{i,k};\xi_{i,j^t_k(i)})\right)\|\\
            &\leq B_{\x}\frac{L}{m}\sum_{i=1}^m\E_{\A}\left\|\left(\begin{array}{c}
                \bar{\x}^t_k-\x^t_{i,k} \\
                \bar{\y}^t_k-\y^t_{i,k}
            \end{array}\right)\right\|.
        \end{aligned}
    \end{equation}

    For the other term, we proceed with its quadratic one:
    \begin{equation}\label{eq:23}
        \begin{aligned}
        &\quad\;\E_{\A}\|\sum_{t=1}^T\sum_{k=0}^{K-1} \eta^t_k\left(\frac{1}{m}\sum_{i=1}^m\dx f_i(\x^t_{i,k},\y^t_{i,k};\xi_{i,j^t_k(i)})-\dx\FS(\bar{\x}^t_k,\bar{\y}^t_k)\right)\|^2\\
        &=\sum_{t=1}^T\sum_{k=0}^{K-1}\eta^{t\,2}_k\E_{\A}\|\frac{1}{m}\sum_{i=1}^m\dx f_i(\x^t_{i,k},\y^t_{i,k};\xi_{i,j^t_k(i)})-\dx\FS(\bar{\x}^t_k,\bar{\y}^t_k)\|^2+\sum_{(t,k)\neq(t',k')}\eta^t_k\eta^{t'}_{k'}\\
        &\E_{\A}\!\!<\!\!\frac{1}{m}\!\sum_{i=1}^m\dx f_i(\x^t_{i,k},\y^t_{i,k};\xi_{i,j^t_k(i)})\!-\!\dx\FS(\bar{\x}^t_k,\bar{\y}^t_k),\frac{1}{m}\sum_{i=1}^m\!\dx f_i(\x^{t'}_{i,k'},\y^{t'}_{i,k'};\xi_{i,j^{t'}_{k'}(i)})\!-\!\dx\FS(\bar{\x}^{t'}_{k'},\bar{\y}^{t'}_{k'})\!>\\
        &\leq 4G^2\sum_{t=1}^T\sum_{k=0}^{K-1}\eta^{t\,2}_k+\sum_{(t,k)\neq(t',k')}\eta^t_k\eta^{t'}_{k'}\\
        &\;\E_{\A}\!\!<\!\!\frac{1}{m}\!\!\sum_{i=1}^m\!\dx f_i(\!\x^t_{i,k},\!\y^t_{i,k};\!\xi_{i,j^t_k\!(i)}\!)\!\!-\!\!\frac{1}{m}\!\!\sum_{i=1}^m\!\dx f_i(\!\bar{\x}^t_k,\!\bar{\y}^t_k;\xi_{i,j^t_k\!(i)}\!)\!\!+\!\!\frac{1}{m}\!\!\sum_{i=1}^m\!\dx f_i(\!\bar{\x}^t_k,\bar{\y}^t_k;\xi_{i,j^t_k\!(i)}\!)\!\!-\!\!\frac{1}{m}\!\!\sum_{i=1}^m\!\!\frac{1}{n}\!\!\sum_{l_i=1}^n\!\dx\! f_i(\!\bar{\x}^t_k,\!\bar{\y}^t_k;\xi_{i,l_i}\!),\\
        &\frac{1}{m}\!\!\sum_{i=1}^m\!\dx f_i(\!\x^{t'}_{i,k'},\!\y^{t'}_{i,k'};\!\xi_{i,j^{t'}_{k'}\!(i)}\!)\!\!-\!\!\frac{1}{m}\!\!\sum_{i=1}^m\!\dx f_i(\!\bar{\x}^{t'}_{k'},\!\bar{\y}^{t'}_{k'};\!\xi_{i,j^{t'}_{k'}\!(i)}\!)\!\!+\!\!\frac{1}{m}\!\!\sum_{i=1}^m\!\dx f_i(\!\bar{\x}^{t'}_{k'},\bar{\y}^{t'}_{k'};\!\xi_{i,j^{t'}_{k'}\!(i)}\!)\!\!-\!\!\frac{1}{m}\!\!\sum_{i=1}^m\!\!\frac{1}{n}\!\!\sum_{l_i=1}^n\!\dx\! f_i(\!\bar{\x}^{t'}_{k'},\!\bar{\y}^{t'}_{k'};\!\xi_{i,l_i}\!)\!\!>\\
        &=4G^2\sum_{t=1}^T\sum_{k=0}^{K-1}\eta^{t\,2}_k+\sum_{(t,k)\neq(t',k')}\eta^t_k\eta^{t'}_{k'}\\
        &\;\E_{\!\A}\!\!<\!\!\!\frac{1}{m}\!\!\sum_{i=1}^m\!\dx f_i(\!\x^t_{i,k},\!\y^t_{i,k};\!\xi_{i,j^t_k\!(i)}\!)\!\!-\!\!\frac{1}{m}\!\!\sum_{i=1}^m\!\dx f_i(\!\bar{\x}^t_k,\!\bar{\y}^t_k;\xi_{i,j^t_k\!(i)}\!)\!,\!\frac{1}{m}\!\!\sum_{i=1}^m\!\dx f_i(\!\x^{t'}_{i,k'},\!\y^{t'}_{i,k'};\!\xi_{i,j^{t'}_{k'}\!(i)}\!)\!\!-\!\!\frac{1}{m}\!\!\sum_{i=1}^m\!\dx\! f_i\!(\!\bar{\x}^{t'}_{k'},\!\bar{\y}^{t'}_{k'};\!\xi_{i,j^{t'}_{k'}\!(i)}\!)\!\!>\\
        &\leq4G^2\sum_{t=1}^T\sum_{k=0}^{K-1}\eta^{t\,2}_k+\sum_{(t,k)\neq(t',k')}\eta^t_k\eta^{t'}_{k'}\\
        &\;\E_{\!\A}\!\|\!\frac{1}{m}\!\!\sum_{i=1}^m\!\dx\! f_i(\!\x^t_{i,k},\!\y^t_{i,k};\!\xi_{i,j^t_k\!(i)}\!)\!\!-\!\!\frac{1}{m}\!\!\sum_{i=1}^m\!\dx f_i(\!\bar{\x}^t_k,\!\bar{\y}^t_k;\xi_{i,j^t_k\!(i)}\!)\!\|\!\cdot\!\|\!\frac{1}{m}\!\!\sum_{i=1}^m\!\dx f_i(\!\x^{t'}_{i,k'},\!\y^{t'}_{i,k'};\!\xi_{i,j^{t'}_{k'}\!(i)}\!)\!\!-\!\!\frac{1}{m}\!\!\sum_{i=1}^m\!\dx\! f_i\!(\!\bar{\x}^{t'}_{k'},\!\bar{\y}^{t'}_{k'};\!\xi_{i,j^{t'}_{k'}\!(i)}\!)\!\|\\
        &\leq4G^2\sum_{t=1}^T\sum_{k=0}^{K-1}\eta^{t\,2}_k+\sum_{(t,k)\neq(t',k')}\eta^t_k\eta^{t'}_{k'}(\frac{1}{m}\sum_{i=1}^m\E_{\A}\left\|\left(\begin{array}{c}
            \x^t_{i,k}-\bar{\x}^t_k\\
            \y^t_{i,k}-\bar{\y}^t_k
        \end{array}\right)\right\|)(\frac{1}{m}\sum_{i=1}^m\E_{\A}\left\|\left(\begin{array}{c}
            \x^{t'}_{i,k'}-\bar{\x}^{t'}_{k'}\\
            \y^{t'}_{i,k'}-\bar{\y}^{t'}_{k'}
        \end{array}\right)\right\|).
        \end{aligned}
    \end{equation}

    For fixed learning rates \ref{op:fixed}, combining inequalities~\eqref{eq:22} and \eqref{eq:23} into \eqref{eq:21}, we can derive:
    \begin{equation*}
        \begin{aligned}
            &\quad\;2\eta\sum_{t=1}^T\sum_{k=0}^{K-1}\E_{\A}[\FS(\bar{\x}^t_k,\bar{\y}^t_k)-\FS(\x,\bar{\y}^t_k)]\\
            &\leq(1-\eta)B_{\x}^2+\eta^2G^2TK+2B_{\x}\sqrt{4\eta^2G^2TK+\eta^2\sum_{(t,k)\neq(t',k')}\E_{\A}[\Delta^t_k]\E_{\A}[\Delta^{t'}_{k'}]}+B_{\x}\eta L\sum_{t=1}^T\sum_{k=0}^{K-1}\E_{\A}[\Delta^t_k].
        \end{aligned}
    \end{equation*}

    Dividing both sides of the above inequality by $2\eta TK$ and exploiting the concavity of $\FS(\x,\cdot)$ that $\FS(\x,\frac{1}{TK}\sum_{t=1}^T\sum_{k=0}^{K-1}\bar{\y}^t_k)\geq\frac{1}{TK}\sum_{t=1}^T\sum_{k=0}^{K-1}\FS(\x,\bar{\y}^t_k)$, we arrive:
    \begin{equation}\label{eq:31}
        \begin{aligned}
            &\quad\;\frac{1}{TK}\sum_{t=1}^T\sum_{k=0}^{K-1}\E_{\A}[\FS(\bar{\x}^t_k,\bar{\y}^t_k)]-\inf_{\x}\E_{\A}[\FS(\x,\widetilde{\y}^T_K)]\\
            &\leq\frac{1}{TK}\sum_{t=1}^T\sum_{k=0}^{K-1}\E_{\A}[\FS(\bar{\x}^t_k,\bar{\y}^t_k)]-\E_{\A}[\FS(\x,\widetilde{\y}^T_K)]\\
            &\leq\frac{1-\eta}{2\eta TK}B_{\x}^2+\frac{\eta G^2}{2}+B_{\x}\frac{\sqrt{4G^2TK+\sum_{(t,k)\neq(t',k')}\E_{\A}[\Delta^t_k]\E_{\A}[\Delta^{t'}_{k'}]}}{TK}+B_{\x}L\frac{\sum_{t=1}^T\sum_{k=0}^{K-1}\E_{\A}[\Delta^t_k]}{2TK}.
        \end{aligned}
    \end{equation}

    On the other hand, we can obtain similar result that:
    \begin{equation}\label{eq:32}
        \begin{aligned}
            &\quad\;\sup_{\y}\E_{\A}[\FS(\widetilde{\x}^T_K,\y)]-\frac{1}{TK}\sum_{t=1}^T\sum_{k=0}^{K-1}\E_{\A}[\FS(\bar{\x}^t_k,\bar{\y}^t_k)]\\
            &\leq\frac{1-\eta}{2\eta TK}B_{\y}^2+\frac{\eta G^2}{2}+B_{\y}\frac{\sqrt{4G^2TK+\sum_{(t,k)\neq(t',k')}\E_{\A}[\Delta^t_k]\E_{\A}[\Delta^{t'}_{k'}]}}{TK}+B_{\y}L\frac{\sum_{t=1}^T\sum_{k=0}^{K-1}\E_{\A}[\Delta^t_k]}{2TK}.
        \end{aligned}
    \end{equation}

    Summing up inequalities~\eqref{eq:31} and \eqref{eq:32}, we can bound the weak PD empirical risk according to condition~\ref{fixed}:
    \begin{equation*}
        \begin{aligned}
            &\quad\;\Delta^{\!\mathrm{w}}_{\S}(\widetilde{\x}^T_K,\widetilde{\y}^T_K)\\
            &\leq\frac{1-\eta}{2\eta TK}(B_{\x}^2\!\!+\!\!B_{\y}^2)\!+\!\eta G^2\!\!+\!(B_{\x}\!\!+\!\!B_{\y})\frac{\sqrt{4G^2TK\!+\!\sum_{(t,k)\neq(t',k')}\E_{\A}[\Delta^t_k]\E_{\A}[\Delta^{t'}_{k'}]}}{TK}\\
            &\quad+(B_{\x}\!+\!B_{\y})L\frac{\sum_{t=1}^T\!\sum_{k=0}^{K-1}\E_{\A}[\Delta^t_k]}{2TK}\\
            &\leq\frac{1-\eta}{2\eta TK}(B_{\x}^2+B_{\y}^2)+\eta G^2+(B_{\x}+B_{\y})\frac{2G}{\sqrt{TK}}+\eta GK(B_{\x}+B_{\y})(1+\frac{L}{2})(\frac{1}{2}\sqrt{\frac{1}{1-\lambdabm}}+\sqrt{\frac{2\lambdabm}{(1-\lambdabm)(1-\lambdabm^2)}}).
        \end{aligned}
    \end{equation*}

    For decaying learning rates \ref{op:decaying}, plugging inequalities~\eqref{eq:22} and \eqref{eq:23} into \eqref{eq:21}, we have:
    \begin{equation*}
        \begin{aligned}
            &\quad\;\sum_{t=1}^T\sum_{k=0}^{K-1}\E_{\A}[\FS(\bar{\x}^t_k,\bar{\y}^t_k)-\FS(\x,\bar{\y}^t_k)]\\
            &\leq\frac{G^2}{2}\sum_{t=1}^T\sum_{k=0}^{K-1}\frac{1}{(k+1)^\alpha t^\beta}+\frac{B_{\x}L}{2}\sum_{t=1}^T\sum_{k=0}^{K-1}\E_{\A}[\Delta^t_k]+B_{\x}\sqrt{4G^2TK+\sum_{(t,k)\neq(t',k')}\E_{\A}[\Delta^t_k]\E_{\A}[\Delta^{t'}_{k'}]}.
        \end{aligned}
    \end{equation*}

    Analogously, dividing both sides by $TK$ and making use of the concavity of $\FS(\x,\cdot)$, we can get:
    \begin{equation*}
        \begin{aligned}
            &\quad\;\frac{1}{TK}\sum_{t=1}^T\sum_{k=0}^{K-1}\E_{\A}[\FS(\bar{\x}^t_k,\bar{\y}^t_k)]-\inf_{\x}\E_{\A}[\FS(\x,\widetilde{\y}^T_K)]\\
            &\leq\frac{G^2}{2TK}\sum_{t=1}^T\sum_{k=0}^{K-1}\frac{1}{(k+1)^\alpha t^\beta}+\frac{B_{\x}L}{2TK}\sum_{t=1}^T\sum_{k=0}^{K-1}\E_{\A}[\Delta^t_k]+\frac{2GB_{\x}}{\sqrt{TK}}+\frac{B_{\x}}{TK}\sqrt{\sum_{(t,k)\neq(t',k')}\E_{\A}[\Delta^t_k]\E_{\A}[\Delta^{t'}_{k'}]}.
        \end{aligned}
    \end{equation*}

    Combining with the symmetric result on the other parameter and referring to condition~\ref{decaying}, we have the following bound for the weak PD empirical risk:
    \begin{equation*}
        \begin{aligned}
            &\quad\;\Delta^{\mathrm{w}}_{\S}(\widetilde{\x}^T_K,\widetilde{\y}^T_K)\\
            &\leq\frac{G^2}{TK}\sum_{t=1}^T\sum_{k=0}^{K-1}\frac{1}{(k+1)^\alpha t^\beta}+\frac{(B_{\x}+B_{\y})L}{2TK}\sum_{t=1}^T\sum_{k=0}^{K-1}\E_{\A}[\Delta^t_k]\\
            &\quad+\frac{2G(B_{\x}+B_{\y})}{\sqrt{TK}}\frac{(B_{\x}+B_{\y})}{TK}\sqrt{\sum_{(t,k)\neq(t',k')}\E_{\A}[\Delta^t_k]\E_{\A}[\Delta^{t'}_{k'}]}\\
            &\overset{\alpha>\frac{1}{2}}{\leq} \frac{G^2}{TK}(\frac{T^{1-\beta}}{1-\beta}\1_{\beta<1}+(1+\ln{T})\1_{\beta=1})(\frac{K^{1-\alpha}}{1-\alpha}\1_{\alpha<1}+(1+\ln{K})\1_{\alpha=1})+\frac{2G(B_{\x}\!+\!B_{\y})}{\sqrt{TK}}\\
            &\quad+\frac{(B_{\x}\!+\!B_{\y})G(L+2)(\sqrt{\frac{1}{1-\lambdabm}\frac{2\alpha}{2\alpha-1}}+\sqrt{(C_{\lambdabm^2}+\frac{C_{\lambdabm}}{1-\lambdabm})\frac{2\alpha}{2\alpha-1}})}{2T}K^{\frac{1}{2}}(\frac{T^{1-\beta}}{1-\beta}\1_{\beta<1}+(1+\ln{T})\1_{\beta=1}).
        \end{aligned}
    \end{equation*}
\end{proof}

\subsubsection{Proof of Weak PD Population Risk}
\begin{proof}[\textbf{Proof of Theorem~\ref{thm:sc-sc-population}}]
    For the weak PD population risk, we have the following decomposition upon the corresponding generalization gap and empirical risk:
    \begin{equation*}
        \Delta^{\!\mathrm{w}}(\widetilde{\x}^T_K,\widetilde{\y}^T_K)=\underbrace{\Delta^{\!\mathrm{w}}(\widetilde{\x}^T_K,\widetilde{\y}^T_K)-\Delta^{\!\mathrm{w}}_{\S}(\widetilde{\x}^T_K,\widetilde{\y}^T_K)}_{\mathrm{weak\; PD\; generalization\; gap}}+\underbrace{\Delta^{\!\mathrm{w}}_{\S}(\widetilde{\x}^T_K,\widetilde{\y}^T_K)}_{\mathrm{weak\; PD\; empirical\; risk}}
    \end{equation*}

    We have the argument stability error from Theorem \ref{thm:local:sc-sc-stability}. According to Theorem \ref{thm:connection}$(i)$, which implies the weak PD generalization gap. Since the generalization error holds for any ($t$,$k$)-th model, it can also suit the averaged version. Combined with the weak PD empirical risk on the averaged output $(\widetilde{\x}^T_K,\widetilde{\y}^T_K)$, we can derive the weak PD population risk as follows:
    \begin{enumerate}[wide=\parindent,label=(\roman*)]
        \item for fixed learning rates,
        \begin{align*}
            &\quad\;\Delta^{\!\mathrm{w}}(\widetilde{\x}^T_K,\widetilde{\y}^T_K)\\
            &\leq\frac{2\sqrt{2}G^2(L+\mu)}{L\mu}\left(\eta L\sqrt{\frac{1}{1-\lambdabm}}\frac{K^2}{2}+\eta L\sqrt{\frac{2\lambdabm}{(1-\lambdabm)(1-\lambdabm^2)}}K^2+\frac{K}{n}\right)\\
            &+\frac{1-\eta}{2\eta TK}(B_{\x}^2+B_{\y}^2)+\eta G^2+(B_{\x}+B_{\y})\frac{2G}{\sqrt{TK}}+\eta GK(B_{\x}+B_{\y})(1+\frac{L}{2})(\frac{1}{2}\sqrt{\frac{1}{1-\lambdabm}}+\sqrt{\frac{2\lambdabm}{(1-\lambdabm)(1-\lambdabm^2)}})\\
            &\leq\frac{2\sqrt{2}G^2(L+\mu)}{L\mu}\bigg(\frac{\mathcal{O}(\sqrt{\lambdabm_1})L}{2}\eta K^2+\mathcal{O}(\sqrt{\lambdabm_2})L\eta K^2+\frac{K}{n}\bigg)+\eta G^2\\
            &+\frac{(1-\eta)(B_{\x}^2+B_{\y}^2\!)}{2\eta TK}+\frac{2G(B_{\x}+B_{\y})}{\sqrt{TK}}+\frac{\eta GKL(B_{\x}+B_{\y})(\mathcal{O}(\sqrt{\lambdabm_1})+\mathcal{O}(\sqrt{\lambdabm_2}))}{2}.
        \end{align*}

        \item for decaying learning rates $\eta^t_k=\frac{1}{(k+1)\alpha t}(\frac{1}{2}<\alpha<1)$,
        \begin{align*}
            &\quad\;\Delta^{\!\mathrm{w}}(\widetilde{\x}^T_K,\widetilde{\y}^T_K)\\
            &\leq\frac{\sqrt{2}G}{T+1}\bigg(\frac{(K^{\frac{3}{2}-\alpha}+\frac{3}{2}-\alpha)\sqrt{\frac{\frac{2\alpha}{2\alpha-1}}{1-\lambdabm}}\frac{2GL(1-\alpha)}{(\frac{3}{2}-\alpha)}}{\frac{L\mu}{L+\mu}(K^{1-\alpha}+1-\alpha)+\alpha-1}+\frac{(K^{1-\alpha}+1-\alpha)2GL\sqrt{(C_{\lambdabm^2}+\frac{C_{\lambdabm}}{1-\lambdabm})\frac{2\alpha}{2\alpha-1}K}}{\frac{L\mu}{L+\mu}(K^{1-\alpha}+1-\alpha)+\alpha-1}\bigg)\\
            &\quad+\frac{\sqrt{2}G^2}{n}\frac{L+\mu}{L\mu}+\frac{G^2(1+\ln{T})}{(1-\alpha)TK^\alpha}+\frac{(B_{\x}\!+\!B_{\y})G(L+2)(\sqrt{\frac{1}{1-\lambdabm}\frac{2\alpha}{2\alpha-1}}+\sqrt{(C_{\lambdabm^2}+\frac{C_{\lambdabm}}{1-\lambdabm})\frac{2\alpha}{2\alpha-1}})}{2T}K^{\frac{1}{2}}(1+\ln{T})\\
            &\quad+\frac{2G(B_{\x}+B_{\y})}{\sqrt{TK}}\\
            &\leq\frac{2G(B_{\x}\!+\!B_{\y})}{\sqrt{TK}}+\frac{\sqrt{2}G^2}{n}\frac{L+\mu}{L\mu}+\frac{\sqrt{2}G(L+\mu)}{L\mu}\frac{K^\frac{1}{2}}{T+1}(\mathcal{O}(\sqrt{\lambdabm_1})+GL\mathcal{O}(\sqrt{\lambdabm_2}))\\
            &+\frac{+\ln{T}}{T}\bigg(\frac{G^2}{(1-\alpha)K^\alpha}+(B_{\x}+B_{\y})GL(\mathcal{O}(\sqrt{\lambdabm_1})+\mathcal{O}(\sqrt{\lambdabm_2}))K^{\frac{1}{2}}\bigg).
        \end{align*}
    \end{enumerate}
    In both conditions, we define $\mathcal{O}(\sqrt{\lambdabm_1})$ and $\mathcal{O}(\sqrt{\lambdabm_2})$ as constants relevant to $\sqrt{\frac{1}{1-\lambdabm}}$ and $\sqrt{C_{\lambdabm^2}+\frac{C_{\lambdabm}}{1-\lambdabm}}$ (or $\sqrt{\frac{2\lambdabm}{(1-\lambdabm)(1-\lambdabm^2)}}$) respectively where the first figure equals to $1$ while the second figure equals to $0$ when $\lambdabm=0$.
    
\end{proof}
\section{Proof in Nonconvex-Strongly-Concave Case}\label{sec:NC-SC}
Firstly, we provide some technical lemmas to facilitate our subsequent proof process.

\subsection{Proof of Stability}

\begin{proof}[\textbf{Proof of Theorem~\ref{thm:primal}}]
    For distributed neighboring dataset $\S$ and $\S'$, where we assume the each local dataset $\S_i$ and $\S'_i$ differs at the last sample without loss of generality, i.e., $\xi_{i,n}$ and $\xi'_{i,n}$ respectively. We use $\A_{\x}(\S)$ and $\A_{\x}(\S')$ to represent the output of the algorithm $\A$ training on the dataset $\S$ and $\S'$, i.e., $\x^T$ and $\xdot^T$ respectively. According to Lemma 14 in \citet{yang2022differentially}, we hold the conclusion that:
    \begin{equation}\label{eq:71}
        \begin{aligned}
            &\quad\; 4\rho\|\pi_{\S}(A_{\x}(\S))-\pi_{\S'}(A_{\x}(\S'))\|^2+\mu\|\hat{\y}_{\S}-\hat{\y}_{\S'}\|^2\\
            &\leq F_{\S}(\pi_{\S'}(\A_{\x}(\S')),\hat{\y}_{\S})-F_{\S'}(\pi_{\S'}(\A_{\x}(\S')),\hat{\y}_{\S})+F_{\S'}(\pi_{\S}(\A_{\x}(\S)),\hat{\y}_{\S'})-F_{\S}(\pi_{\S}(\A_{\x}(\S)),\hat{\y}_{\S'})
        \end{aligned}   
    \end{equation}
    where we define $\Theta_{\x(\S)}=\{\hat{\x}_{\S}|(\hat{\x}_{\S},\hat{\y}_{\S})\in\arg\min_{\x\in\xset}\max_{\y\in\yset} F_{\S}(\x,\y)\}$ and $\pi_{\S}(\x)\triangleq\arg\min_{\hat{\x}_{\S}\in\Theta_{\x(\S)}}\|\x-\hat{\x}_{\S}\|^2$.

    Continuing with inequality~\eqref{eq:71}, we can get:
    \begin{equation}\label{eq:72}
        \begin{aligned}
            &\quad\;F_{\S}(\pi_{\S'}(\A_{\x}(\S')),\hat{\y}_{\S})-F_{\S'}(\pi_{\S'}(\A_{\x}(\S')),\hat{\y}_{\S})+F_{\S'}(\pi_{\S}(\A_{\x}(\S)),\hat{\y}_{\S'})-F_{\S}(\pi_{\S}(\A_{\x}(\S)),\hat{\y}_{\S'})\\
            &=\frac{1}{m}\sum_{i=1}^m\frac{1}{n}\\
            &\quad\left(f_i(\pi_{\S'}(\A_{\x}(\S')),\hat{\y}_{\S};\xi_{i,n})\!\!-\!\!f_i(\pi_{\S'}(\A_{\x}(\S')),\hat{\y}_{\S};\xi'_{i,n})\!\!+\!\!f_i(\pi_{\S}(\A_{\x}(\S)),\hat{\y}_{\S'};\xi'_{i,n})\!\!-\!\!f_i(\pi_{\S}(\A_{\x}(\S)),\hat{\y}_{\S'};\xi_{i,n})\right)\\
            &=\frac{1}{m}\sum_{i=1}^m\frac{1}{n}\\
            &\quad\left(f_i(\pi_{\S'}(\A_{\x}(\S')),\hat{\y}_{\S};\xi_{i,n})\!\!-\!\!f_i(\pi_{\S}(\A_{\x}(\S)),\hat{\y}_{\S'};\xi_{i,n})\!\!+\!\!f_i(\pi_{\S}(\A_{\x}(\S)),\hat{\y}_{\S'};\xi'_{i,n})\!\!-\!\!f_i(\pi_{\S'}(\A_{\x}(\S')),\hat{\y}_{\S};\xi'_{i,n})\right)\\
            &\leq\frac{2G}{n}\left\|\left(\begin{array}{c}
                \pi_{\S}(\A_{\x}(\S))-\pi_{\S'}(\A_{\x}(\S'))   \\
                \hat{\y}_{\S'}-\hat{\y}_{\S} 
            \end{array}\right)\right\|\\
            &\leq\frac{2G}{n}\sqrt{\frac{1}{4\rho}+\frac{1}{\mu}}\sqrt{4\rho\|\pi_{\S}(A_{\x}(\S))-\pi_{\S'}(A_{\x}(\S'))\|^2+\mu\|\hat{\y}_{\S}-\hat{\y}_{\S'}\|^2}.
        \end{aligned}
    \end{equation}

    Combining inequality~\eqref{eq:71} and \eqref{eq:72}, we can acquire:
    \begin{equation}
        \begin{aligned}
            2\sqrt{\rho}\|\pi_{\S}(A_{\x}(\S))-\pi_{\S'}(A_{\x}(\S'))\|_2&\leq\sqrt{4\rho\|\pi_{\S}(A_{\x}(\S))-\pi_{\S'}(A_{\x}(\S'))\|^2+\mu\|\hat{\y}_{\S}-\hat{\y}_{\S'}\|^2}\\
        &\leq\frac{2G}{n}\sqrt{\frac{1}{4\rho}+\frac{1}{\mu}}.
        \end{aligned}
    \end{equation}

    While on the other hand, PL condition implies quadratic growth with the same constant $\rho$ as shown in \cite{karimi2016linear}, then we can obtain the following for any given dataset $\S$:
    \begin{equation}
        F_{\S}(\x,\y)-F_{\S}(\pi_{\S}(\x),\y)\leq\frac{\rho}{2}\|\x-\pi_{\S}(\x)\|^2, \forall \y\in\yset.
    \end{equation}

    Therefore there holds that, for any given dataset $\S$, the divergence between $\A_{\x}(\S)$ to the nearest optimal solution $\pi_{\S}(\x)$ can be bounded by the corresponding excess primal empirical risk: $\E_{\A}\|\A_{\x}(\S)-\pi_{\S}(\A_{\x}(\S))\|\leq\sqrt{\E_{\A}\|\x^T-\pi_{\S}(\x^T)\|^2}\leq\sqrt{\frac{2}{\rho}\E_{\A}[R_{\S}(\x^T)-R^*_{\S}]} $. 

    And the primal stability can be bounded by:
    \begin{equation*}
        \begin{aligned}
            &\quad\;\E_{\A}\|\A_{\x}(\S)-\A_{\x}(\S')\|\\
            &\leq\E_{\A}\|\A_{\x}(\S)-\pi_{\S}(\A_{\x}(\S))\|+\E_{\A}\|\pi_{\S}(\A_{\x}(\S))-\pi_{\S'}(\A_{\x}(\S'))\|+\E_{\A}\|\A_{\x}(\S')-\pi_{\S'}(\A_{\x}(\S'))\|\\
            &\leq\E_{\A}\|\pi_{\S}(\A_{\x}(\S))-\pi_{\S'}(\A_{\x}(\S'))\|+\sqrt{\frac{2}{\rho}\Delta^{\mathrm{ep}}_{\S}(\A_{\x}(\S))}+\sqrt{\frac{2}{\rho}\Delta^{\mathrm{ep}}_{\S'}(\A_{\x}(\S'))}\\
            &\leq\frac{2G}{n}\sqrt{\frac{1}{4\rho}+\frac{1}{\mu}}+2\sqrt{\frac{2}{\rho}\Delta^{\mathrm{rp}}_{\S}}.
        \end{aligned}
    \end{equation*}
\end{proof}

\subsection{Proof of Population Risk}

\subsubsection{Proof of Optimization Error}

\begin{thm}[\textbf{Excess Primal Empirical Risk}]\label{thm:ep:empirical}
    When each $f_i$ satisfies $\rho$PL-$\mu$SC condition, under Assumption~\ref{as:mixing}, \ref{as:Lip:con}, \ref{as:Lip:smmoth}, choosing learning rates as $\eta^t_k=\frac{4}{\mu(k+1)^\alpha t}$, where $1/2<\alpha<1$, letting $\mathbf{b}=\frac{2\mu^3L^2}{\mu^4-32L^2}\leq\frac{(\rho-\mu/4)\mu^3}{16L^2(L+L^2/\mu)}$, and requiring $\eta^t_k(L+L^2/\mu)\leq1$, we have the excess primal empirical risk is bounded by:
    \begin{equation*}
        \begin{aligned}
            &\quad\;\E_{\A}[R_{\S}(\x^T)-R^*_{\S}]\\
            &\leq\frac{1}{T}\Big\{\frac{16G}{\mu^2}(\frac{4B_{\y}\mathbf{b} L}{1-\mu L}+2G)\bigg(\frac{K^{\frac{3}{2}}+\frac{3}{2}-\alpha}{K^{1-\alpha}}\frac{1-\alpha}{\frac{3}{2}-\alpha}\sqrt{\frac{\frac{2\alpha}{2\alpha-1}}{1-\lambdabm}}+\frac{K^{1-\alpha}+1-\alpha}{K^{1-\alpha}}\sqrt{K}\sqrt{(C_{\lambdabm^2}+\frac{C_{\lambdabm}}{1-\lambdabm})\frac{2\alpha}{2\alpha-1}}\bigg)\\
            &+\frac{64G^2\mathbf{b}}{\mu^4}\frac{(1-\alpha)K^2}{K^{1-\alpha}}\bigg(\frac{\frac{2\alpha}{2\alpha-1}}{1-\lambdabm}+(C_{\lambdabm^2}+\frac{C_{\lambdabm}}{1-\lambdabm})\frac{2\alpha}{2\alpha-1}\bigg)+\frac{1}{K^{1-\alpha}}\frac{2\alpha(1-\alpha)}{2\alpha-1}\frac{16G^2}{\mu^2}(\frac{\rho}{1-\mu L}+2(L+L^2/\mu))\Big\}.
        \end{aligned}
    \end{equation*}
\end{thm}

\begin{proof}[\textbf{Proof of Theorem~\ref{thm:ep:empirical}}]
    Applying the Lipschitz smoothness property on $R_{\S}(\x)$ with constant $L+L^2/\mu$ as mentioned in Lemma~\ref{le:RSlipschitz}, we have:
    \begin{equation*}
        R_{\S}(\bar{\x}^t_{k+1})\leq R_{\S}(\bar{\x}^t_k)+<\nabla R_{\S}(\bar{\x}^t_k),\bar{\x}^t_{k+1}-\bar{\x}^t_k>+\frac{L+L^2/\mu}{2}\|\bar{\x}^t_{k+1}-\bar{\x}^t_k\|^2.
    \end{equation*}

    Subtracting $R_{\S}^*$ from both sides of above inequality and utilizing the update rule of $\bar{\x}^t_{k+1}$, we can obtain:
    \begin{equation}\label{eq:51}
        \begin{aligned}
            &\quad\; R_{\S}(\bar{\x}^t_{k+1})-R_{\S}^*\\
            &\leq R_{\S}(\bar{\x}^t_k)-R_{\S}^*-\eta^t_k\!<\!\nabla R_{\S}(\bar{\x}^t_k),\frac{1}{m}\sum_{i=1}^m\dx f_i(\x^t_{i,k},\y^t_{i,k};\xi_{i,j^t_k(i)})\!>\\
            &\quad+\frac{L+L^2/\mu}{2}\eta^{t\,2}_k\|\frac{1}{m}\sum_{i=1}^m\dx f_i(\x^t_{i,k},\y^t_{i,k};\xi_{i,j^t_k(i)})\|^2\\
            &\leq R_{\S}(\bar{\x}^t_k)\!\!-\!\!R_{\S}^*\!\!-\!\!\eta^t_k\!<\!\nabla R_{\S}(\bar{\x}^t_k),\dx F_{\S}(\bar{\x}^t_k,\bar{\y}^t_k)\!\!>\!+\eta^t_k\!<\!\nabla R_{\S}(\bar{\x}^t_k),\dx F_{\S}(\bar{\x}^t_k,\bar{\y}^t_k)\!-\!\frac{1}{m}\!\sum_{i=1}^m\!\dx f_i(\x^t_{i,k},\y^t_{i,k};\xi_{i,j^t_k(i)})\!\!>\!\\
            &\quad+\frac{L+L^2/\mu}{2}\eta^{t\,2}_k\|\dx F_{\S}(\bar{\x}^t_k,\bar{\y}^t_k)\|^2+\frac{L+L^2/\mu}{2}\eta^{t\,2}_k\|\frac{1}{m}\sum_{i=1}^m\!\dx f_i(\x^t_{i,k},\y^t_{i,k};\xi_{i,j^t_k(i)})-\dx F_{\S}(\bar{\x}^t_k,\bar{\y}^t_k)\|^2\\
            &\quad+(L+L^2/\mu)\eta^{t\,2}_k<\dx F_{\S}(\bar{\x}^t_k,\bar{\y}^t_k),\frac{1}{m}\sum_{i=1}^m\!\dx f_i(\x^t_{i,k},\y^t_{i,k};\xi_{i,j^t_k(i)})-\dx F_{\S}(\bar{\x}^t_k,\bar{\y}^t_k)>.
        \end{aligned}
    \end{equation}

    Taking expectation on the randomness of the algorithm, we can derive for the second cross term:
    \begin{equation}\label{eq:52}
        \begin{aligned}
            &\quad\;\E_{\A}\!<\!\nabla R_{\S}(\bar{\x}^t_k),\dx F_{\S}(\bar{\x}^t_k,\bar{\y}^t_k)\!-\!\frac{1}{m}\!\sum_{i=1}^m\!\dx f_i(\x^t_{i,k},\y^t_{i,k};\xi_{i,j^t_k(i)})\!\!>\!\\
            &=\underbrace{\E_{\A}\!<\!\nabla R_{\S}(\bar{\x}^t_k),\frac{1}{m}\sum_{i=1}^m\frac{1}{n}\sum_{l_i=1}^n\dx f_i(\bar{\x}^t_k,\bar{\y}^t_k;\xi_{i,l_i})\!-\!\frac{1}{m}\!\sum_{i=1}^m\!\dx f_i(\bar{\x}^t_k,\bar{\y}^t_k;\xi_{i,j^t_k(i)})\!\!>\!}_{=0}\\
            &\quad+\E_{\A}\!<\!\nabla R_{\S}(\bar{\x}^t_k),\frac{1}{m}\!\sum_{i=1}^m\!\dx f_i(\bar{\x}^t_k,\bar{\y}^t_k;\xi_{i,j^t_k(i)})\!-\!\frac{1}{m}\!\sum_{i=1}^m\!\dx f_i(\x^t_{i,k},\y^t_{i,k};\xi_{i,j^t_k(i)})\!\!>\!.
        \end{aligned}
    \end{equation}

    Analogously, we can get similar results for the last cross term:
    \begin{equation}\label{eq:53}
        \begin{aligned}
            &\quad\;\E_{\A}\!<\!\dx F_{\S}(\bar{\x}^t_k,\bar{\y}^t_k),\!\frac{1}{m}\!\sum_{i=1}^m\!\dx f_i(\x^t_{i,k},\y^t_{i,k};\xi_{i,j^t_k(i)})-\!\dx F_{\S}(\bar{\x}^t_k,\bar{\y}^t_k)\!\!>\!\\
            &=\E_{\A}\!<\!-\dx F_{\S}(\bar{\x}^t_k,\bar{\y}^t_k),\frac{1}{m}\!\sum_{i=1}^m\!\dx f_i(\bar{\x}^t_k,\bar{\y}^t_k;\xi_{i,j^t_k(i)})\!-\!\frac{1}{m}\!\sum_{i=1}^m\!\dx f_i(\x^t_{i,k},\y^t_{i,k};\xi_{i,j^t_k(i)})\!\!>\!.
        \end{aligned}
    \end{equation}

    Combining equations \eqref{eq:52} and \eqref{eq:53} with their coefficients under the requirement that $\eta^t_k(L+L^2/\mu)\leq1$, we can get:
    \begin{equation}\label{eq:54}
        \begin{aligned}
            &\quad\;\eta^t_k\eqref{eq:52}+(L+L^2/\mu)\eta^{t\,2}_k\eqref{eq:53}\\
            &\leq\eta^t_k\E_{\A}\!<\!\nabla R_{\S}(\bar{\x}^t_k)-\dx F_{\S}(\bar{\x}^t_k,\bar{\y}^t_k),\frac{1}{m}\!\sum_{i=1}^m\!\dx f_i(\bar{\x}^t_k,\bar{\y}^t_k;\xi_{i,j^t_k(i)})\!-\!\frac{1}{m}\!\sum_{i=1}^m\!\dx f_i(\x^t_{i,k},\y^t_{i,k};\xi_{i,j^t_k(i)})\!>\!\\
            &\leq2G\eta^t_k\E_{\A}[\Delta^t_k].
        \end{aligned}
    \end{equation}

    In a similar way as we have operated in inequality~\eqref{eq:23}, we can conclude that:
    \begin{equation}\label{eq:55}
        \E_{\A}\|\frac{1}{m}\sum_{i=1}^m\!\dx f_i(\x^t_{i,k},\y^t_{i,k};\xi_{i,j^t_k(i)})-\dx F_{\S}(\bar{\x}^t_k,\bar{\y}^t_k)\|^2\leq4G^2.
    \end{equation}

    Besides, under the requirement that $\eta^t_k(L+L^2/\mu)\leq1$, we have the following recombination:
    \begin{equation}\label{eq:56}
        \begin{aligned}
            &\quad\;-\eta^t_k\E_{\A}\!<\!\nabla R_{\S}(\bar{\x}^t_k),\dx F_{\S}(\bar{\x}^t_k,\bar{\y}^t_k)\!\!>\!+\frac{L+L^2/\mu}{2}\eta^{t\,2}_k\E_{\A}\|\dx F_{\S}(\bar{\x}^t_k,\bar{\y}^t_k)\|^2\\
            &=\frac{\eta^t_k}{2}\E_{\A}\|\nabla R_{\S}(\bar{\x}^t_k)-\dx F_{\S}(\bar{\x}^t_k,\bar{\y}^t_k)\|^2-\frac{\eta^t_k}{2}\E_{\A}\|\nabla R_{\S}(\bar{\x}^t_k)\|^2\\
            &\leq\frac{\eta^t_k}{2}L^2\E_{\A}\|\bar{\y}^t_k-\hat{\y}_{\S}(\bar{\x}^t_k)\|^2-\rho\eta^t_k(R_{\S}(\bar{\x}^t_k)-R^*_{\S})
        \end{aligned}
    \end{equation}
    where the last inequality results from Lemma~\ref{le:RSlipschitz} that $\nabla R_{\S}(\bar{\x}^t_k)=\dx F_{\S}(\bar{\x}^t_k,\hat{\y}_{\S}(\bar{\x}^t_k))$ combined with the Lipschitz smoothness and PL condition of $R_{\S}$, as noted in Lemma~\ref{le:RSPL}.

    Overall, plugging equations~\eqref{eq:54},\eqref{eq:55},\eqref{eq:56} into the inequality~\eqref{eq:51} with expectation on the algorithm, we can obtain:
    \begin{equation}\label{eq:60}
        \E_{\A}[R_{\S}(\bar{\x}^t_{k+1})-R_{\S}^*]\leq(1-\rho\eta^t_k)\E_{\A}[R_{\S}(\bar{\x}^t_k)-R_{\S}^*]+\frac{\eta^t_kL^2}{2}\E_{\A}\|\bar{\y}^t_k-\hat{\y}_{\S}(\bar{\x}^t_k)\|^2+2G\eta^t_k\E_{\A}[\Delta^t_k]+2(L+L^2/\mu)\eta^{t\,2}_kG^2.
    \end{equation}

    Next, we will focus on the $\E_{\A}\|\bar{\y}^t_k-\hat{\y}_{\S}(\bar{\x}^t_k)\|^2$ term, for a given factor $\mathbf{a}$, we proceed with the following decomposition:
    \begin{equation}\label{eq:57}
        \E_{\A}\|\bar{\y}^t_{k+1}-\hat{\y}_{\S}(\bar{\x}^t_{k+1})\|^2\leq(1+\mathbf{a})\E_{\A}\|\hat{\y}_{\S}(\bar{\x}^t_{k+1})-\hat{\y}_{\S}(\bar{\x}^t_k)\|^2+(1+\frac{1}{\mathbf{a}})\E_{\A}\|\bar{\y}^t_{k+1}-\hat{\y}_{\S}(\bar{\x}^t_k)\|^2.
    \end{equation}

    For the first component, according to Lemma~\ref{le:RSlipschitz} and \ref{le:RSPL}, we have:
    \begin{equation}\label{eq:58}
        \begin{aligned}
            &\quad\E_{\A}\|\hat{\y}_{\S}(\bar{\x}^t_{k+1})-\hat{\y}_{\S}(\bar{\x}^t_k)\|^2\\
            &\leq\frac{L^2}{\mu^2}\E_{\A}\|\bar{\x}^t_{k+1}-\bar{\x}^t_k\|^2\\
            &\leq\frac{L^2}{\mu^2}\eta^{t\,2}_k\E_{\A}\|\frac{1}{m}\sum_{i=1}^m\dx f_i(\x^t_{i,k},\y^t_{i,k};\xi_{i,j^t_k(i)})\|^2\\
            &\leq\frac{2L^2}{\mu^2}\eta^{t\,2}_k(\E_{\A}[\Delta^t_k])^2+\frac{2L^2}{\mu^2}\eta^{t\,2}_k\E_{\A}\|\dx F_{\S}(\bar{\x}^t_k,\bar{\y}^t_k)\|^2\\
            &\leq\frac{2}{\mu}\eta^t_k(\E_{\A}[\Delta^t_k])^2\!\!+\!\!\frac{4L^2}{\mu^2}\eta^{t\,2}_k\E_{\A}\|\nabla R_{\S}(\bar{\x}^t_k)\!\!-\!\!\dx F_{\S}(\bar{\x}^t_k,\bar{\y}^t_k)\|^2\!\!+\!\!\frac{4L^2}{\mu^2}\eta^{t\,2}_k\E_{\A}\|\nabla R_{\S}(\bar{\x}^t_k)\|^2\\
            &\leq\frac{2}{\mu}\eta^t_k(\E_{\A}[\Delta^t_k])^2\!\!+\!\!\frac{4L^4}{\mu^2}\eta^{t\,2}_k\E_{\A}\|\hat{\y}_{\S}(\bar{\x}^t_k)\!\!-\!\!\bar{\y}^t_k\|^2\!\!+\!\!\frac{8L^2(L\!+\!L^2/\mu)}{\mu^2}\eta^{t\,2}_k\E_{\A}[R_{\S}(\bar{\x}^t_k)-R^*_{\S}].
        \end{aligned}
    \end{equation}

    For the second component, applying the update rule we can derive:
    \begin{small}
        \begin{equation}\label{eq:59}
        \begin{aligned}
            &\quad\;\E_{\A}\|\bar{\y}^t_{k+1}-\hat{\y}_{\S}(\bar{\x}^t_k)\|^2\\
            &=\E_{\!\A}\|\bar{\y}^t_k\!-\!\hat{\y}_{\S}(\bar{\x}^t_k)\|^2\!+\!2\eta^t_k\E_{\!\A}\!\!<\!\!\bar{\y}^t_k\!-\!\hat{\y}_{\S}(\bar{\x}^t_k),\frac{1}{m}\!\sum_{i=1}^m\!\dy f_i(\x^t_{i,k},\y^t_{i,k};\xi_{i,j^t_k(i)})\!\!>\!\!+\eta^{t\,2}_k\E_{\!\A}\|\frac{1}{m}\sum_{i=1}^m\dy f_i(\x^t_{i,k},\y^t_{i,k};\xi_{i,j^t_k(i)})\|^2\\
            &\leq\!\E_{\!\A}\|\bar{\y}^t_k\!\!-\!\hat{\y}_{\S}(\bar{\x}^t_k)\|^2\!\!\!+\eta^{t\,2}_kG^2\\
            &\quad+\!2\eta^t_k\!\!\left(\!\!\E_{\!\A}\!\!<\!\!\bar{\y}^t_k\!\!-\!\hat{\y}_{\S}\!(\bar{\x}^t_k),\!\dy F_{\S}(\bar{\x}^t_k,\bar{\y}^t_k)\!\!>\!\!+\E_{\!\A}\!\!<\!\!\bar{\y}^t_k\!\!-\!\hat{\y}_{\S}\!(\bar{\x}^t_k),\!\frac{1}{m}\!\!\sum_{i=1}^m\!\dy f_i(\x^t_{i,k},\y^t_{i,k};\xi_{i,j^t_k(i)})\!\!-\!\!\dy F_{\S}(\bar{\x}^t_k,\bar{\y}^t_k)\!\!>\!\!\right)\\
            &\overset{(a)}{\leq}\!\E_{\!\A}\|\bar{\y}^t_k\!\!-\!\hat{\y}_{\S}(\bar{\x}^t_k)\|^2\!\!\!+\eta^{t\,2}_kG^2\\
            &\quad+2\eta^t_k\!\!\left(\!-\frac{1}{2L}\E_{\A}\|\dy F_{\S}(\bar{\x}^t_k,\bar{\y}^t_k)-\dy F_{\S}(\bar{\x}^t_k,\hat{\y}_{\S}(\bar{\x}^t_k))\|^2-\!\frac{\mu}{2}\E_{\A}\|\bar{\y}^t_k\!-\!\hat{\y}_{\S}(\bar{\x}^t_k)\|^2\!\!+\!2B_{\y}L\E_{\A}[\Delta^t_k]\right)\\
            &\overset{(b)}{\leq}\!(1\!-\!\mu\eta^t_k)\E_{\!\A}\|\bar{\y}^t_k\!\!-\!\hat{\y}_{\S}(\bar{\x}^t_k)\|^2\!\!-\!\frac{\eta^t_k}{L}\E_{\A}\|\dy F_{\S}(\bar{\x}^t_k,\bar{\y}^t_k)\|^2\!+\!4\eta^t_kB_{\y}L\E_{\A}[\Delta^t_k]\!+\!\eta^{t\,2}_kG^2
        \end{aligned}
    \end{equation}
    \end{small}
    
    where in inequality $(a)$, we first make use of the property of $\mu$-strong concavity on $F_{\S}(\x,\cdot)$ that: $<\bar{\y}^t_k-\hat{\y}_{\S}(\bar{\x}^t_k),\dy F_{\S}(\bar{\x}^t_k,\bar{\y}^t_k)>\leq F_{\S}(\bar{\x}^t_k,\bar{\y}^t_k)-F_{\S}(\bar{\x}^t_k,\hat{\y}_{\S}(\bar{\x}^t_k))-\frac{\mu}{2}\|\bar{\y}^t_k-\hat{\y}_{\S}(\bar{\x}^t_k)\|^2$. And then we can conclude that: $F_{\S}(\bar{\x}^t_k,\bar{\y}^t_k)-F_{\S}(\bar{\x}^t_k,\hat{\y}_{\S}(\bar{\x}^t_k))\leq<\dx F_{\S}(\bar{\x}^t_k,\hat{\y}_{\S}(\bar{\x}^t_k)),\bar{\y}^t_k-\hat{\y}_{\S}(\bar{\x}^t_k)>-\frac{1}{2L}\|\dy F_{\S}(\bar{\x}^t_k,\bar{\y}^t_k)-\dy F_{\S}(\bar{\x}^t_k,\hat{\y}_{\S}(\bar{\x}^t_k))\|^2$ according to Property~\ref{pro:smooth+}. Since $\hat{\y}_{\S}(\bar{\x}^t_k)$ maximizes $F_{\S}(\bar{\x}^t_k,\cdot)$, consequently $<\dx F_{\S}(\bar{\x}^t_k,\hat{\y}_{\S}(\bar{\x}^t_k)),\y-\hat{\y}_{\S}(\bar{\x}^t_k)>\leq0, \forall\y\in\yset$. We exploit the optimality of $\hat{\y}_{\S}(\bar{\x}^t_k)$ in inequality $(b)$.

    As a whole, substituting the inequality~\eqref{eq:58} and \eqref{eq:59} into the decomposition~\eqref{eq:57}, we can acquire the following:
    \begin{equation}\label{eq:61}
        \begin{aligned}
            &\quad\;\E_{\A}\|\bar{\y}^t_{k+1}\!-\hat{\y}_{\S}(\bar{\x}^t_{k+1})\|^2\\
            &\leq((1+\frac{1}{\mathbf{a}})(1-\mu\eta^t_k)+(1+\mathbf{a})\frac{4L^4}{\mu^2}\eta^{t\,2}_k)\E_{\!\A}\|\bar{\y}^t_k\!-\hat{\y}_{\S}(\bar{\x}^t_k)\|^2+(1+\frac{1}{\mathbf{a}})(4\eta^t_kB_{\y}L\E_{\A}[\Delta^t_k]+\eta^{t\,2}_kG^2)\\
            &\quad+(1+\mathbf{a})\frac{8L^2(L+L^2/\mu)}{\mu^2}\eta^{t\,2}_k\E_{\A}[R_{\S}(\bar{\x}^t_k)-R^*_{\S}]+(1+\mathbf{a})\frac{2}{\mu}\eta^t_k(\E_{\A}[\Delta^t_k])^2.
        \end{aligned}
    \end{equation}

    Sum up inequalities \eqref{eq:60} and \eqref{eq:61} in the following way, in which we denote $A^t_k=\E_{\A}[R_{\S}(\bar{\x}^t_k)-R^*_{\S}]$ and $B^t_k=\E_{\A}\|\bar{\y}^t_k-\hat{\y}_{\S}(\bar{\x}^t_k)\|^2$:
    \begin{equation*}
        \begin{aligned}
            &\quad\;A^t_{k+1}+\mathbf{b} B^t_{k+1}\\
            &\leq\left(1-\rho\eta^t_k+\mathbf{b}(1+\mathbf{a})\frac{8L^2(L+L^2/\mu)}{\mu^2}\eta^{t\,2}_k\right)A^t_k+\left(\frac{\eta^t_kL^2}{2}+\mathbf{b}((1+\frac{1}{\mathbf{a}})(1-\mu\eta^t_k)+(1+\mathbf{a})\frac{4L^2}{\mu^2}\eta^{t\,2}_k)\right)B^t_k\\
            &\quad+\mathbf{b}(1+\frac{1}{\mathbf{a}})(4B_{\y}L\eta^t_k\E_{\A}[\Delta^t_k]+\eta^{t\,2}_kG^2)+\mathbf{b}(1+\mathbf{a})\frac{2}{\mu}\eta^t_k(\E_{\A}[\Delta^t_k])^2+2G\eta^t_k\E_{\A}[\Delta^t_k]+2(L+L^2/\mu)\eta^{t\,2}_kG^2.
        \end{aligned}
    \end{equation*}

    Letting $\mathbf{a}=\frac{1-\mu\eta^t_k}{\frac{\mu}{2}\eta^t_k}$ and $\mathbf{b}=\frac{2\mu^3L^2}{\mu^4-32L^2}\leq\frac{(\rho-\mu/4)\mu^3}{16L^2(L+L^2/\mu)}$, then $1+\frac{1}{\mathbf{a}}=\frac{1-\frac{\mu}{2}\eta^t_k}{1-\mu\eta^t_k}\leq\frac{1}{1-\mu/L}$ and $1+\mathbf{a}=\frac{1-\frac{\mu}{2}\eta^t_k}{\frac{\mu}{2}\eta^t_k}\leq\frac{1}{\frac{\mu}{2}\eta^t_k}$ and the coefficients of $A^t_k$ and $B^t_k$ can be scaled by $1-\frac{\mu}{4}\eta^t_k$ and $\mathbf{b}(1-\frac{\mu}{4}\eta^t_k)$ respectively, and we can deduce:
    \begin{equation*}
        \begin{aligned}
            &\quad\;A^t_{k+1}+\mathbf{b} B^t_{k+1}\\
            &\leq(1-\frac{\mu}{4}\eta^t_k)(A^t_k+\mathbf{b} B^t_k)\!+\!\frac{\mathbf{b}}{1-\mu/ L}(4B_{\y}L\eta^t_k\E_{\A}[\Delta^t_k]\!+\!\eta^{t\,2}_kG^2)\!+\!\frac{4\mathbf{b}}{\mu^2}(\E_{\A}[\Delta^t_k])^2\\
            &\quad+2G\eta^t_k\E_{\A}[\Delta^t_k]\!+\!2(L\!+\!L^2/\mu)\eta^{t\,2}_kG^2.
        \end{aligned}
    \end{equation*}

    Repeating the process from $t=T,k=K-1$ to $t=1,k=0$, representing the remaining term $\frac{\mathbf{b}}{1-\mu/ L}(4B_{\y}L\eta^t_k\E_{\A}[\Delta^t_k]\!+\!\eta^{t\,2}_kG^2)\!+\!\frac{4\mathbf{b}}{\mu^2}(\E_{\A}[\Delta^t_k])^2\!+\!2G\eta^t_k\E_{\A}[\Delta^t_k]\!+\!2(L\!+\!L^2/\mu)\eta^{t\,2}_kG^2$ as $\Psi^t_k$ for brevity, we can derive the following:
    \begin{equation}\label{eq:62}
        \begin{aligned}
            A^T_K+\mathbf{b} B^T_K&\leq(1-\frac{\mu}{4}\eta^T_{K-1})(A^T_{K-1}+\mathbf{b} B^T_{K-1})+\Psi^T_{K-1}\\
            &\leq\cdots\\
            &\leq\prod_{k=0}^{K\!-\!1}\!(1\!-\!\frac{\mu}{4}\eta^T_k)(A^T_0\!+\!\mathbf{b} B^T_0)\!+\!\sum_{k=0}^{K\!-\!1}\Psi^T_k\prod_{k'=k+1}^{K-1}(1-\frac{\mu}{4}\eta^T_{k'})\\
            &\overset{(a)}{=}\prod_{k=0}^{K\!-\!1}\!(1\!-\!\frac{\mu}{4}\eta^T_k)(A^{T-1}_K\!+\!\mathbf{b} B^{T-1}_K)\!+\!\sum_{k=0}^{K\!-\!1}\Psi^T_k\prod_{k'=k+1}^{K-1}(1-\frac{\mu}{4}\eta^T_{k'})\\
            &\leq\cdots\\
            &\leq\prod_{t=1}^T\prod_{k=0}^{K\!-\!1}(1\!-\!\frac{\mu}{4}\eta^t_k)(A^1_0\!+\!\mathbf{b} B^1_0)\!+\sum_{t=1}^T\prod_{t'=t+1}^T\prod_{k_1=0}^{K-1}(1-\frac{\mu}{4}\eta^{t'}_{k_1})\sum_{k=0}^{K-1}\Psi^t_k\prod_{k'=k+1}^{K-1}(1-\frac{\mu}{4}\eta^t_{k'})\\
            &\overset{(b)}{=}\sum_{t=1}^T\prod_{t'=t+1}^T\prod_{k_1=0}^{K-1}(1-\frac{\mu}{4}\eta^{t'}_{k_1})\sum_{k=0}^{K-1}\Psi^t_k\prod_{k'=k+1}^{K-1}(1-\frac{\mu}{4}\eta^t_{k'})
        \end{aligned}
    \end{equation}
    where the equation $(a)$ results from the fact that $A^t_0+\mathbf{b} B^t_0=A^{t-1}_K+\mathbf{b} B^{t-1}_K\forall t,k$. While in inequality $(b)$, we adopt decaying learning rates that $\eta^t_k=\frac{4}{\mu(k+1)^\alpha t^\beta}$ so that $1-\frac{\mu}{4}\eta^1_0=0$.

    When learning rates are configured as $\eta^t_k=\frac{4}{\mu(k+1)^\alpha t}$ (where we require $1/2<\alpha<1$ and $\beta=1$), the consensus term is bounded by the following according to Lemma~\ref{le:delta:localdecentra}:
    \begin{equation*}
        \E_{\A}[\Delta^t_k]\leq\frac{4G}{\mu t}\left(\sqrt{\frac{\frac{2\alpha}{2\alpha-1}k}{1-\lambdabm}}\!+\!\sqrt{(C_{\lambdabm^2}\!+\!\frac{C_{\lambdabm}}{1-\lambdabm})\frac{2\alpha}{2\alpha-1}K}\right).
    \end{equation*}

    The successive-product term can be evaluated as:
    \begin{equation*}
        \prod_{t'=t+1}^T\prod_{k_1=0}^{K-1}(1-\frac{1}{(k_1+1)^\alpha t'})\leq\prod_{t'=t+1}^T\prod_{k_1=0}^{K-1}e^{-\frac{1}{(k_1+1)^\alpha t'}}
                =e^{-\sum_{t'=t+1}^T\sum_{k_1=0}^{K-1}\frac{1}{(k_1+1)^\alpha t'}}
                \leq(\frac{t}{T})^{1+\frac{K^{1-\alpha}}{1-\alpha}}.
    \end{equation*}

    Continuing with inequality~\eqref{eq:62}, we can obtain:
    \begin{equation*}
        \begin{aligned}
            &\quad\;A^T_K+\mathbf{b} B^T_K\\
            &\leq\sum_{t=1}^T(\frac{t}{T})^{1+\frac{K^{1-\alpha}}{1-\alpha}}\sum_{k=0}^{K-1}\\
            &\quad\left(\frac{\mathbf{b}}{1-\mu/ L}(4B_{\y}L\eta^t_k\E_{\A}[\Delta^t_k]\!+\!\eta^{t\,2}_kG^2)\!+\!\frac{4\mathbf{b}}{\mu^2}(\E_{\A}[\Delta^t_k])^2\!+\!2G\eta^t_k\E_{\A}[\Delta^t_k]\!+\!2(L\!+\!L^2/\mu)\eta^{t\,2}_kG^2\right)\\
            &=\frac{1}{T^{1+\frac{K^{1-\alpha}}{1-\alpha}}}\sum_{t=1}^T\sum_{k=0}^{K-1}\bigg(\frac{16G}{\mu^2}(\frac{4B_{\y}\mathbf{b} L}{1-\mu L}+2G)\frac{t^{\frac{K^{1-\alpha}}{1-\alpha}-1}}{(k+1)^\alpha}\Big(\sqrt{\frac{\frac{2\alpha}{2\alpha-1}}{1-\lambdabm}k}\!+\!\sqrt{(C_{\lambdabm^2}\!+\!\frac{C_{\lambdabm}}{1-\lambdabm})\frac{2\alpha}{2\alpha-1}K}\Big)\\
            &\quad\quad\quad\quad\quad\quad\quad\quad+\frac{64G^2\mathbf{b}}{\mu^4}t^{\frac{K^{1-\alpha}}{1-\alpha}-1}\Big(\frac{\frac{2\alpha}{2\alpha-1}}{1-\lambdabm}k\!+\!(C_{\lambdabm^2}\!+\!\frac{C_{\lambdabm}}{1-\lambdabm})\frac{2\alpha}{2\alpha-1}K\Big)\\
            &\quad\quad\quad\quad\quad\quad\quad\quad+\frac{16G^2}{\mu^2}(\frac{\mathbf{b}}{1-\mu L}+2(L+L^2/\mu))\frac{t^{\frac{K^{1-\alpha}}{1-\alpha}-1}}{(k+1)^{2\alpha}}\bigg)\\
            &\leq\frac{1}{T}\Big\{\frac{16G}{\mu^2}(\frac{4B_{\y}\mathbf{b} L}{1-\mu L}+2G)\bigg(\frac{K^{\frac{3}{2}}+\frac{3}{2}-\alpha}{K^{1-\alpha}}\frac{1-\alpha}{\frac{3}{2}-\alpha}\sqrt{\frac{\frac{2\alpha}{2\alpha-1}}{1-\lambdabm}}+\frac{K^{1-\alpha}+1-\alpha}{K^{1-\alpha}}\sqrt{K}\sqrt{(C_{\lambdabm^2}+\frac{C_{\lambdabm}}{1-\lambdabm})\frac{2\alpha}{2\alpha-1}}\bigg)\\
            &+\frac{64G^2\mathbf{b}}{\mu^4}\frac{(1-\alpha)K^2}{K^{1-\alpha}}\bigg(\frac{\frac{2\alpha}{2\alpha-1}}{1-\lambdabm}+(C_{\lambdabm^2}+\frac{C_{\lambdabm}}{1-\lambdabm})\frac{2\alpha}{2\alpha-1}\bigg)+\frac{1}{K^{1-\alpha}}\frac{2\alpha(1-\alpha)}{2\alpha-1}\frac{16G^2}{\mu^2}(\frac{\rho}{1-\mu L}+2(L+L^2/\mu))\Big\}.
        \end{aligned}
    \end{equation*}

    Since $A^T_K+\mathbf{b} B^T_K=\E_{\A}[R_{\S}(\bar{\x}^T_K)-R^*_{\S}]+\mathbf{b}\E_{\A}\|\bar{\y}^T_K-\hat{\y}_{\S}(\bar{\x}^T_K)\|^2\leq\E_{\A}[R_{\S}(\bar{\x}^T_K)-R^*_{\S}]=\E_{\A}[R_{\S}(\x^T)-R^*_{\S}]$, we can conclude the excess primal empirical risk as following:
    \begin{equation*}
        \begin{aligned}
            &\quad\;\E_{\A}[R_{\S}(\x^T)-R^*_{\S}]\\
            &\leq\frac{1}{T}\Big\{\frac{16G}{\mu^2}(\frac{4B_{\y}\mathbf{b} L}{1-\mu L}+2G)\bigg(\frac{K^{\frac{3}{2}}+\frac{3}{2}-\alpha}{K^{1-\alpha}}\frac{1-\alpha}{\frac{3}{2}-\alpha}\sqrt{\frac{\frac{2\alpha}{2\alpha-1}}{1-\lambdabm}}+\frac{K^{1-\alpha}+1-\alpha}{K^{1-\alpha}}\sqrt{K}\sqrt{(C_{\lambdabm^2}+\frac{C_{\lambdabm}}{1-\lambdabm})\frac{2\alpha}{2\alpha-1}}\bigg)\\
            &+\frac{64G^2\mathbf{b}}{\mu^4}\frac{(1-\alpha)K^2}{K^{1-\alpha}}\bigg(\frac{\frac{2\alpha}{2\alpha-1}}{1-\lambdabm}+(C_{\lambdabm^2}+\frac{C_{\lambdabm}}{1-\lambdabm})\frac{2\alpha}{2\alpha-1}\bigg)+\frac{1}{K^{1-\alpha}}\frac{2\alpha(1-\alpha)}{2\alpha-1}\frac{16G^2}{\mu^2}(\frac{\rho}{1-\mu L}+2(L+L^2/\mu))\Big\}.
        \end{aligned}
    \end{equation*}
\end{proof}

\begin{proof}[\textbf{Proof of Theorem~\ref{thm:ep:population}}]
    The excess primal population risk can be decomposed into excess primal generalization gap and the corresponding excess primal empirical risk as follows:
    \begin{equation*}
        \Delta^{\!\mathrm{ep}}({\x}^T)=\underbrace{\Delta^{\!\mathrm{ep}}({\x}^T)-\Delta^{\!\mathrm{w}}_{\S}({\x}^T)}_{\mathrm{excess\; primal\; generalization\; gap}}+\underbrace{\Delta^{\!\mathrm{ep}}_{\S}({\x}^T)}_{\mathrm{excess\; primal\; empirical\; risk}}.
    \end{equation*}

    According to Theorem \ref{thm:primal} combined with Theorem \ref{thm:connection} $(ii)$, we can conduct the excess primal generalization gap $\zeta^{\mathrm{rp}}_{\mathrm{gen}}(\A,\S)$ when each local function $f_i$ satisfies $\rho$PL-$\mu$SC condition and further requiring function $F(\x,\cdot)$ is $\mu$-SC. Then the excess primal empirical risk can be derived from Theorem \ref{thm:ep:empirical} when choosing decaying learning rates $\eta^t_k=\frac{4}{\mu(k+1)^\alpha t}$ with $\frac{1}{2}<\alpha<1$. It is noteworthy that we adopt the last iterate instead of the averaged one. And it is interesting to find out that generalization gap plays a dominant role in excess primal population risk since $\zeta^{\!\mathrm{ep}}_{\mathrm{gen}}$ is upper bounded by $\sqrt{\Delta^{\!\mathrm{ep}}_{\S}}$ as shown in Theorem \ref{thm:connection}.

    And we can conclude the excess primal population risk:
    \begin{align*}
        &\quad\Delta^{\!\mathrm{ep}}(\A_{\x}(\S))\\
        &\leq\frac{2G}{\sqrt{T}}\sqrt{\frac{2}{\rho}\!+\!\frac{2L^2}{\rho\mu^2}}\\
        &\quad\bigg(\sqrt{\frac{16G}{\mu^2}(\frac{4B_{\y}\mathbf{b} L}{1-\mu L}\!+\!2G)\bigg(\frac{K^{\frac{3}{2}}\!+\!\frac{3}{2}\!-\!\alpha}{K^{1-\alpha}}\frac{1\!-\!\alpha}{\frac{3}{2}\!-\!\alpha}\sqrt{\frac{\frac{2\alpha}{2\alpha-1}}{1-\lambdabm}}\!+\!\frac{K^{1-\alpha}\!+\!1\!-\!\alpha}{K^{1-\alpha}}\sqrt{K}\sqrt{(C_{\lambdabm^2}\!+\!\frac{C_{\lambdabm}}{1-\lambdabm})\frac{2\alpha}{2\alpha-1}}\bigg)}\\
        &+\!\sqrt{\frac{64G^2\mathbf{b}}{\mu^4}\frac{(1\!-\!\alpha)K^2}{K^{1-\alpha}}\!\bigg(\!\frac{\frac{2\alpha}{2\alpha-1}}{1-\lambdabm}\!+\!(C_{\lambdabm^2}\!+\!\frac{C_{\lambdabm}}{1\!-\!\lambdabm})\frac{2\alpha}{2\alpha\!-\!1}\bigg)\!\!+\!\!\frac{1}{K^{1\!-\!\alpha}}\frac{2\alpha(1-\alpha)}{2\alpha-1}\frac{16G^2}{\mu^2}(\frac{\rho}{1\!-\!\mu L}\!+\!2(L\!+\!L^2/\mu))}\bigg)\\
        &+\frac{2G^2}{n}\sqrt{1+\frac{L^2}{\mu^2}}\sqrt{\frac{1}{4\rho}+\frac{1}{\mu}}+\frac{4G^2}{\mu mn}\\
        &\leq\!\frac{2\sqrt{2}G^2}{\sqrt{T}}\!\sqrt{\frac{1}{\rho}\!+\!\kappa^2}\bigg(\!\sqrt{\frac{64B_{\y}\kappa^3}{\mu^2}}\sqrt{\mathcal{O}(\!\sqrt{\!\lambdabm_1\!})\!+\!\mathcal{O}(\!\sqrt{\!\lambdabm_2})}K^{\frac{1}{4}}\!\!+\!(\mathcal{O}(\!\sqrt{\!\lambdabm_1}\!)\!+\!\mathcal{O}(\!\sqrt{\!\lambdabm_2}))\!\sqrt{\!\frac{128\kappa^2}{\mu^3}}\!K^{\frac{1\!+\!\alpha}{2}}\!\!+\!\!\sqrt{\frac{32\rho\kappa^2}{\mu^2}}\frac{1}{K^{\frac{1\!-\!\alpha}{2}}}\bigg)\\
        &\quad+\frac{2G^2}{n}\sqrt{1+\frac{L^2}{\mu^2}}\sqrt{\frac{1}{4\rho}+\frac{1}{\mu}}+\frac{4G^2}{\mu mn}
    \end{align*}
    where we denote $\kappa=\frac{L}{\mu}$.
\end{proof}
\section{Proof in Nonconvex-Nonconcave Case}\label{sec:NC-NC}
We still adopt the denotations that $(\x^t_{i,k},\y^t_{i,k})$ and $(\xdot^t_{i,k},\ydot^t_{i,k})$ represent the model for the distributed algorithm $\A$ training on the neighboring dataset $\S$ and $\S'$ on the $t-k$ round for agent $i$, respectively. Without loss of generality, assume each local dataset $\S_i$ and $\S'_i$ differs at the last sample, i.e., $\S_i=\{\xi_{i,1},...,\xi_{i,n-1},\xi_{i,n}\}$ and $\S'_i=\{\xi_{i,1},...,\xi_{i,n-1},\xi'_{i,n}\}$ respectively.

In this section, we employ a technique that assuming before $t_0-k_0$-th iteration won't they meet the different sample. Using $\delta^t_k$ to denote $\left\|\left(\begin{array}{c}
   \bar{\x}^t_k-\bar{\xdot}^t_k \\
   \bar{\y}^t_k-\bar{\ydot}^t_k
\end{array}\right)\right\|$, so that $\delta^{t_0}_{k_0}=0$.

\begin{proof}[\textbf{Proof of Theorem~\ref{thm:weak}}]
    According to Lemma \ref{le:iteration}, and substituting condition $(ii)$ with condition $(i)$ in Lemma \ref{le:nonexpan}, we can acquire:
    \begin{equation*}
        \begin{aligned}
            &\quad\;\E_{\A}[\left\|\left(\begin{array}{c}
                \bar{\x}^t_{k+1}-\bar{\xdot}^t_{k+1}   \\
                \bar{\y}^t_{k+1}-\bar{\ydot}^t_{k+1} 
            \end{array}\right)\right\||\delta^{t_0}_{k_0}=0]\\
            &\leq(1+\eta^t_kL)\E_{\A}[\left\|\left(\begin{array}{c}
                \bar{\x}^t_k-\bar{\xdot}^t_k   \\
                \bar{\y}^t_k-\bar{\ydot}^t_k 
            \end{array}\right)\right\||\delta^{t_0}_{k_0}=0]+2\eta^t_kL\E_{\A}[\Delta^t_k]+\frac{2\eta^t_kG}{n}
        \end{aligned}
    \end{equation*}

    Performing above process from $t=T,k=K-1$ to $t=t_0,k=K-1$, we can acquire:
    \begin{equation*}
        \begin{aligned}
            &\quad\;\E_{\A}[\left\|\left(\begin{array}{c}
                \bar{\x}^T_K-\bar{\xdot}^T_K   \\
                \bar{\y}^T_K-\bar{\ydot}^T_K 
            \end{array}\right)\right\||\delta^{t_0}_{k_0}=0]\\
            &\leq\prod_{t=t_0}^T\prod_{k=k_0}^{K-1}(1+\eta^{t_0}_{k_0}L)\E_{\A}[\left\|\left(\begin{array}{c}
                \bar{\x}^{t_0}_{k_0}-\bar{\xdot}^{t_0}_{k_0}\\
                \bar{\y}^{t_0}_{k_0}-\bar{\ydot}^{t_0}_{k_0}
            \end{array}\right)\right\||\delta^{t_0}_{k_0}=0]\\
            &\quad\;+2\sum_{t=t_0}^T\prod_{t'=t+1}^T\prod_{k_1=0}^{K-1}(1+\eta^{t'}_{k_1}L)\sum_{k=k_0}^{K-1}\eta^t_k(L\E_{\A}[\Delta^t_k]+\frac{G}{n})\prod_{k'=k+1}^{K-1}(1+\eta^t_{k'}L)\\
            &=2\sum_{t=t_0}^T\prod_{t'=t+1}^T\prod_{k_1=0}^{K-1}(1+\eta^{t'}_{k_1}L)\sum_{k=k_0}^{K-1}\eta^t_k(L\E_{\A}[\Delta^t_k]+\frac{G}{n})\prod_{k'=k+1}^{K-1}(1+\eta^t_{k'}L)
        \end{aligned}
    \end{equation*}

    Then we make use of the Lemma 5 in \citet{zhu2023stability} that:
    \begin{equation}\label{eq:81}
        \begin{aligned}
            &\quad\;\E_{\A}[\frac{1}{m}\sum_{i=1}^m\big(f_i(\bar{\x}^t_k,\y';\xi_{i,l_i})-f_i(\bar{\xdot}^t_k,\y';\xi_{i,l_i})+f_i(\x',\bar{\y}^t_k;\xi_{i,l_i})-f_i(\x',\bar{\ydot}^t_k;\xi_{i,l_i})\big)]\\          &\leq\sqrt{2}G\E_{\A}[\delta^t_k|\delta^{t_0}_{k_0}=0]+\frac{Mm(Kt_0+k_0))}{n}
        \end{aligned}
    \end{equation}
    where we require each local function is bounded by $M$, i.e., $|f_i(\x,\y;\xi_{i,l})|\leq M, \forall \x\in\xset,\y\in\yset$.

    For fixed learning rates (see condition \ref{fixed} in Lemma \ref{le:delta:localdecentra}), we have:
    \begin{equation*}
        \E_{\A}[\delta^T_K|\delta^{t_0}_{k_0}=0]\leq(2\eta^2KGL\sqrt{\frac{1}{1-\lambdabm}+\frac{2\lambdabm}{(1-\lambdabm)(1-\lambdabm^2)}}+\frac{2\eta G}{n})\frac{(1+\eta L)^{K(T-t_0+1)}-1}{(1+\eta L)^K-1}\frac{(1+\eta L)^{K-k_0}-1}{\eta L }.
    \end{equation*}

    Combining with inequality~\eqref{eq:81}, we can conduct that:
    \begin{equation*}
        \begin{aligned}
            &\quad\;\sup_{\xi_{i,l_i}}\big[\sup_{\y'\in\yset}\E_{\A}[\frac{1}{m}\sum_{i=1}^m\big(f_i(\x^T,\y';\xi_{i,l_i})-f_i(\xdot^T,\y';\xi_{i,l_i})+f_i(\x',\y^T;\xi_{i,l_i})-f_i(\x',\ydot^T;\xi_{i,l_i})\big)]\big]\\
            &\leq2\sqrt{2}\eta G^2(\eta KL\sqrt{\frac{1}{1-\lambdabm}+\frac{2\lambdabm}{(1-\lambdabm)(1-\lambdabm^2)}}+\frac{1}{n})(K-k_0)(T-t_0+1)+\frac{Mm(Kt_0+k_0))}{n}.
        \end{aligned}
    \end{equation*}

    It can reach optimal when $k_0=0$ and $t_0=0$, and the result comes out to be $2\sqrt{2}\eta G^2(\eta KL\sqrt{\frac{1}{1-\lambdabm}+\frac{2\lambdabm}{(1-\lambdabm)(1-\lambdabm^2)}}+\frac{1}{n})K(T+1)$.

    For decaying learning rates $\eta^t_k=\frac{1}{L(k+1)^\alpha t^\beta}$ (see condition \ref{decaying} in Lemma \ref{le:delta:localdecentra}), we choose $\alpha=\beta=1$ and obtain:
    \begin{equation*}
        \prod_{t'=t+1}^T\prod_{k_1=0}^{K-1}(1+\eta^{t'}_{k_1}L)=\prod_{t'=t+1}^T\prod_{k_1=0}^{K-1}(1+\frac{1}{(k_1+1)t'})\leq\prod_{t'=t+1}^T\prod_{k_1=0}^{K-1}e^{\frac{1}{(k_1+1)t'}}\leq(\frac{T}{t})^{1+\ln{K}}.
    \end{equation*}
    \begin{equation*}
        \prod_{k'=k+1}^{K-1}(1+\eta^t_{k'}L)=\prod_{k'=k+1}^{K-1}(1+\frac{1}{(k'+1)t})\leq(\frac{K}{k+1})^{\frac{1}{t}}\leq\frac{K}{k+1}.
    \end{equation*}
    Recalling Lemma \ref{le:delta:localdecentra}, the consensus term is bounded by:
    \begin{equation*}
        \E_{\A}[\Delta^t_k]\leq\frac{G}{Lt}\sqrt{\frac{1}{1-\lambdabm}\frac{2\alpha}{2\alpha-1}k}+\frac{G}{Lt}\sqrt{(C_{\lambdabm^2}+\frac{C_{\lambdabm}}{1-\lambdabm})\frac{2\alpha}{2\alpha-1}K}.
    \end{equation*}

    Therefore, there holds that:
    \begin{equation*}
        \begin{aligned}
            \E_{\A}[\delta^T_K|\delta^{t_0}_{k_0}=0]&\leq\frac{2G}{n}\frac{T^{1+\ln{K}}K}{L(1+\ln{K})}\frac{1}{k_0t_0^{1+\ln{K}}}+\frac{4G}{L}\sqrt{\frac{1}{1-\lambdabm}\frac{2\alpha}{2\alpha-1}}\frac{T^{1+\ln{K}}K}{2+\ln{K}}\frac{1}{k_0^{\frac{1}{2}}t_0^{2+\ln{K}}}\\
            &\quad+\frac{2G}{L}\sqrt{(C_{\lambdabm^2}+\frac{C_{\lambdabm}}{1-\lambdabm})\frac{2\alpha}{2\alpha-1}}\frac{T^{1+\ln{K}}K^{\frac{3}{2}}}{2+\ln{K}}\frac{1}{k_0t_0^{2+\ln{K}}}.
        \end{aligned}
    \end{equation*}

    Combining with inequality \eqref{eq:81}, we can acquire:
    \begin{equation*}
        \begin{aligned}
            &\quad\;\sup_{\xi_{i,l_i}}\big[\sup_{\y'\in\yset}\E_{\A}[\frac{1}{m}\sum_{i=1}^m\big(f_i(\x^T,\y';\xi_{i,l_i})-f_i(\xdot^T,\y';\xi_{i,l_i})+f_i(\x',\y^T;\xi_{i,l_i})-f_i(\x',\ydot^T;\xi_{i,l_i})\big)]\big]\\
            &\leq\sqrt{2}G\Bigg(\frac{2G}{n}\frac{T^{1+\ln{K}}K}{L(1+\ln{K})}\frac{1}{k_0t_0^{1+\ln{K}}}+\frac{4G}{L}\sqrt{\frac{1}{1-\lambdabm}\frac{2\alpha}{2\alpha-1}}\frac{T^{1+\ln{K}}K}{2+\ln{K}}\frac{1}{k_0^{\frac{1}{2}}t_0^{2+\ln{K}}}\\
            &\quad+\frac{2G}{L}\sqrt{(C_{\lambdabm^2}+\frac{C_{\lambdabm}}{1-\lambdabm})\frac{2\alpha}{2\alpha-1}}\frac{T^{1+\ln{K}}K^{\frac{3}{2}}}{2+\ln{K}}\frac{1}{k_0t_0^{2+\ln{K}}}\Bigg)+\frac{Mm(Kt_0+k_0)}{n}.
        \end{aligned}
    \end{equation*}
    When we choose the argument as $k_0^{\frac{3}{2}}t_0^{\frac{5}{2}+\ln{K}}=\frac{nG}{Mm\sqrt{2K}}\bigg(\frac{2G}{n}\frac{T^{1+\ln{K}}K}{L(1+\ln{K})}+\frac{4G}{L}\sqrt{\frac{1}{1-\lambdabm}\frac{2\alpha}{2\alpha-1}}\frac{T^{1+\ln{K}}K}{2+\ln{K}}+\frac{2G}{L}\sqrt{(C_{\lambdabm^2}+\frac{C_{\lambdabm}}{1-\lambdabm})\frac{2\alpha}{2\alpha-1}}\frac{T^{1+\ln{K}}K^{\frac{3}{2}}}{2+\ln{K}}\bigg)$, the weak stability can attain optimal of $2\bigg(\frac{2G}{n}\frac{T^{1+\ln{K}}K}{L(1+\ln{K})}+\frac{4G}{L}\sqrt{\frac{1}{1-\lambdabm}\frac{2\alpha}{2\alpha-1}}\frac{T^{1+\ln{K}}K}{2+\ln{K}}+\frac{2G}{L}\sqrt{(C_{\lambdabm^2}+\frac{C_{\lambdabm}}{1-\lambdabm})\frac{2\alpha}{2\alpha-1}}\frac{T^{1+\ln{K}}K^{\frac{3}{2}}}{2+\ln{K}}\bigg)^{\frac{1}{5}}\frac{(2Mm)^{\frac{4}{5}}K^{\frac{2}{5}}}{n^{\frac{4}{5}}}$.
\end{proof}


\begin{thebibliography}{54}
\providecommand{\natexlab}[1]{#1}
\providecommand{\url}[1]{\texttt{#1}}
\expandafter\ifx\csname urlstyle\endcsname\relax
  \providecommand{\doi}[1]{doi: #1}\else
  \providecommand{\doi}{doi: \begingroup \urlstyle{rm}\Url}\fi

\bibitem[Arjovsky et~al.(2017)Arjovsky, Chintala, and Bottou]{arjovsky2017wasserstein}
Martin Arjovsky, Soumith Chintala, and L{\'e}on Bottou.
\newblock Wasserstein generative adversarial networks.
\newblock In \emph{International conference on machine learning}, pages 214--223. PMLR, 2017.

\bibitem[Beznosikov et~al.()Beznosikov, Rogozin, and Kovalev]{beznosikovnear}
Aleksandr Beznosikov, Alexander Rogozin, and Dmitry Kovalev.
\newblock Near-optimal decentralized algorithms for saddle point problems over time-varying networks.
\newblock \emph{Optimization and Applications}, page 246.

\bibitem[Beznosikov et~al.(2022{\natexlab{a}})Beznosikov, Dvurechenskii, Koloskova, Samokhin, Stich, and Gasnikov]{beznosikov2022decentralized}
Aleksandr Beznosikov, Pavel Dvurechenskii, Anastasiia Koloskova, Valentin Samokhin, Sebastian~U Stich, and Alexander Gasnikov.
\newblock Decentralized local stochastic extra-gradient for variational inequalities.
\newblock \emph{Advances in Neural Information Processing Systems}, 35:\penalty0 38116--38133, 2022{\natexlab{a}}.

\bibitem[Beznosikov et~al.(2022{\natexlab{b}})Beznosikov, Richt{\'a}rik, Diskin, Ryabinin, and Gasnikov]{beznosikov2022distributed}
Aleksandr Beznosikov, Peter Richt{\'a}rik, Michael Diskin, Max Ryabinin, and Alexander Gasnikov.
\newblock Distributed methods with compressed communication for solving variational inequalities, with theoretical guarantees.
\newblock \emph{Advances in Neural Information Processing Systems}, 35:\penalty0 14013--14029, 2022{\natexlab{b}}.

\bibitem[Beznosikov et~al.(2023)Beznosikov, Gasnikov, Zainullina, Maslovskii, and Pasechnyuk]{beznosikov2023unified}
Aleksandr~N Beznosikov, Alexander~Vladimirovich Gasnikov, Karina~E Zainullina, A~Yu Maslovskii, and Dmitry~Arkad'evich Pasechnyuk.
\newblock A unified analysis of variational inequality methods: variance reduction, sampling, quantization, and coordinate descent.
\newblock \emph{Computational Mathematics and Mathematical Physics}, 63\penalty0 (2):\penalty0 147--174, 2023.

\bibitem[Bousquet and Elisseeff(2002)]{bousquet2002stability}
Olivier Bousquet and Andr{\'e} Elisseeff.
\newblock Stability and generalization.
\newblock \emph{The Journal of Machine Learning Research}, 2:\penalty0 499--526, 2002.

\bibitem[Cao et~al.(2002)Cao, Shen, Milito, and Wirth]{cao2002internet}
Xi-Ren Cao, Hong-Xia Shen, Rodolfo Milito, and Patricia Wirth.
\newblock Internet pricing with a game theoretical approach: concepts and examples.
\newblock \emph{IEEE/ACM transactions on networking}, 10\penalty0 (2):\penalty0 208--216, 2002.

\bibitem[Chen et~al.(2021)Chen, Zhou, Xu, and Liang]{chen2021proximal}
Ziyi Chen, Yi~Zhou, Tengyu Xu, and Yingbin Liang.
\newblock Proximal gradient descent-ascent: Variable convergence under k $\{$$\backslash$L$\}$ geometry.
\newblock \emph{arXiv preprint arXiv:2102.04653}, 2021.

\bibitem[Cheng et~al.(2023)Cheng, Shen, Zhu, Guo, Fang, Liu, Du, and Tao]{cheng2023prescribed}
Zhihao Cheng, Li~Shen, Miaoxi Zhu, Jiaxian Guo, Meng Fang, Liu Liu, Bo~Du, and Dacheng Tao.
\newblock Prescribed safety performance imitation learning from a single expert dataset.
\newblock \emph{IEEE transactions on pattern analysis and machine intelligence}, 2023.

\bibitem[Cohen et~al.(2019)Cohen, Rosenfeld, and Kolter]{cohen2019certified}
Jeremy Cohen, Elan Rosenfeld, and Zico Kolter.
\newblock Certified adversarial robustness via randomized smoothing.
\newblock In \emph{international conference on machine learning}, pages 1310--1320. PMLR, 2019.

\bibitem[Deng and Mahdavi(2021)]{deng2021local}
Yuyang Deng and Mehrdad Mahdavi.
\newblock Local stochastic gradient descent ascent: Convergence analysis and communication efficiency.
\newblock In \emph{International Conference on Artificial Intelligence and Statistics}, pages 1387--1395. PMLR, 2021.

\bibitem[Deng et~al.(2020)Deng, Kamani, and Mahdavi]{deng2020distributionally}
Yuyang Deng, Mohammad~Mahdi Kamani, and Mehrdad Mahdavi.
\newblock Distributionally robust federated averaging.
\newblock \emph{Advances in neural information processing systems}, 33:\penalty0 15111--15122, 2020.

\bibitem[Elisseeff et~al.(2005)Elisseeff, Evgeniou, Pontil, and Kaelbing]{elisseeff2005stability}
Andre Elisseeff, Theodoros Evgeniou, Massimiliano Pontil, and Leslie~Pack Kaelbing.
\newblock Stability of randomized learning algorithms.
\newblock \emph{Journal of Machine Learning Research}, 6\penalty0 (1), 2005.

\bibitem[Farnia and Ozdaglar(2021)]{farnia2021train}
Farzan Farnia and Asuman Ozdaglar.
\newblock Train simultaneously, generalize better: Stability of gradient-based minimax learners.
\newblock In \emph{International Conference on Machine Learning}, pages 3174--3185. PMLR, 2021.

\bibitem[Goodfellow et~al.(2014)Goodfellow, Pouget-Abadie, Mirza, Xu, Warde-Farley, Ozair, Courville, and Bengio]{goodfellow2014generative}
Ian Goodfellow, Jean Pouget-Abadie, Mehdi Mirza, Bing Xu, David Warde-Farley, Sherjil Ozair, Aaron Courville, and Yoshua Bengio.
\newblock Generative adversarial nets.
\newblock \emph{Advances in neural information processing systems}, 27, 2014.

\bibitem[Hardt et~al.(2016)Hardt, Recht, and Singer]{hardt2016train}
Moritz Hardt, Ben Recht, and Yoram Singer.
\newblock Train faster, generalize better: Stability of stochastic gradient descent.
\newblock In \emph{International conference on machine learning}, pages 1225--1234. PMLR, 2016.

\bibitem[Ho and Ermon(2016)]{ho2016generative}
Jonathan Ho and Stefano Ermon.
\newblock Generative adversarial imitation learning.
\newblock \emph{Advances in neural information processing systems}, 29, 2016.

\bibitem[Horn and Johnson(2012)]{horn2012matrix}
Roger~A Horn and Charles~R Johnson.
\newblock \emph{Matrix analysis}.
\newblock Cambridge university press, 2012.

\bibitem[Hou et~al.(2021)Hou, Thekumparampil, Fanti, and Oh]{hou2021efficient}
Charlie Hou, Kiran~K Thekumparampil, Giulia Fanti, and Sewoong Oh.
\newblock Efficient algorithms for federated saddle point optimization.
\newblock \emph{arXiv preprint arXiv:2102.06333}, 2021.

\bibitem[Karimi et~al.(2016)Karimi, Nutini, and Schmidt]{karimi2016linear}
Hamed Karimi, Julie Nutini, and Mark Schmidt.
\newblock Linear convergence of gradient and proximal-gradient methods under the polyak-lojasiewicz condition.
\newblock In \emph{Machine Learning and Knowledge Discovery in Databases: European Conference, ECML PKDD 2016, Riva del Garda, Italy, September 19-23, 2016, Proceedings, Part I 16}, pages 795--811. Springer, 2016.

\bibitem[Lei et~al.(2021)Lei, Yang, Yang, and Ying]{lei2021stability}
Yunwen Lei, Zhenhuan Yang, Tianbao Yang, and Yiming Ying.
\newblock Stability and generalization of stochastic gradient methods for minimax problems.
\newblock In \emph{International Conference on Machine Learning}, pages 6175--6186. PMLR, 2021.

\bibitem[Lei et~al.(2023)Lei, Sun, and Liu]{lei2023stability}
Yunwen Lei, Tao Sun, and Mingrui Liu.
\newblock Stability and generalization for minibatch sgd and local sgd.
\newblock \emph{arXiv preprint arXiv:2310.01139}, 2023.

\bibitem[Liao et~al.(2022)Liao, Shen, Duan, Kolar, and Tao]{liao2022local}
Luofeng Liao, Li~Shen, Jia Duan, Mladen Kolar, and Dacheng Tao.
\newblock Local adagrad-type algorithm for stochastic convex-concave optimization.
\newblock \emph{Machine Learning}, pages 1--20, 2022.

\bibitem[Lin et~al.(2020)Lin, Jin, and Jordan]{lin2020gradient}
Tianyi Lin, Chi Jin, and Michael Jordan.
\newblock On gradient descent ascent for nonconvex-concave minimax problems.
\newblock In \emph{International Conference on Machine Learning}, pages 6083--6093. PMLR, 2020.

\bibitem[Liu et~al.(2020)Liu, Zhang, Mroueh, Cui, Ross, Yang, and Das]{liu2020decentralized}
Mingrui Liu, Wei Zhang, Youssef Mroueh, Xiaodong Cui, Jarret Ross, Tianbao Yang, and Payel Das.
\newblock A decentralized parallel algorithm for training generative adversarial nets.
\newblock \emph{Advances in Neural Information Processing Systems}, 33:\penalty0 11056--11070, 2020.

\bibitem[Lojasiewicz(1963)]{lojasiewicz1963topological}
Stanislaw Lojasiewicz.
\newblock A topological property of real analytic subsets.
\newblock \emph{Coll. du CNRS, Les {\'e}quations aux d{\'e}riv{\'e}es partielles}, 117\penalty0 (87-89):\penalty0 2, 1963.

\bibitem[Madry et~al.(2018)Madry, Makelov, Schmidt, Tsipras, and Vladu]{madry2018towards}
Aleksander Madry, Aleksandar Makelov, Ludwig Schmidt, Dimitris Tsipras, and Adrian Vladu.
\newblock Towards deep learning models resistant to adversarial attacks.
\newblock In \emph{International Conference on Learning Representations}, 2018.

\bibitem[Mohri et~al.(2019)Mohri, Sivek, and Suresh]{mohri2019agnostic}
Mehryar Mohri, Gary Sivek, and Ananda~Theertha Suresh.
\newblock Agnostic federated learning.
\newblock In \emph{International Conference on Machine Learning}, pages 4615--4625. PMLR, 2019.

\bibitem[Nedi{\'c} et~al.(2018)Nedi{\'c}, Olshevsky, and Rabbat]{nedic2018network}
Angelia Nedi{\'c}, Alex Olshevsky, and Michael~G Rabbat.
\newblock Network topology and communication-computation tradeoffs in decentralized optimization.
\newblock \emph{Proceedings of the IEEE}, 106\penalty0 (5):\penalty0 953--976, 2018.

\bibitem[Ozdaglar et~al.(2022)Ozdaglar, Pattathil, Zhang, and Zhang]{ozdaglar2022good}
Asuman Ozdaglar, Sarath Pattathil, Jiawei Zhang, and Kaiqing Zhang.
\newblock What is a good metric to study generalization of minimax learners?
\newblock \emph{Advances in Neural Information Processing Systems}, 35:\penalty0 38190--38203, 2022.

\bibitem[Polyak et~al.(1963)]{polyak1963gradient}
Boris~Teodorovich Polyak et~al.
\newblock Gradient methods for minimizing functionals.
\newblock \emph{Zhurnal vychislitel’noi matematiki i matematicheskoi fiziki}, 3\penalty0 (4):\penalty0 643--653, 1963.

\bibitem[Rakhlin et~al.(2005)Rakhlin, Mukherjee, and Poggio]{rakhlin2005stability}
Alexander Rakhlin, Sayan Mukherjee, and Tomaso Poggio.
\newblock Stability results in learning theory.
\newblock \emph{Analysis and Applications}, 3\penalty0 (04):\penalty0 397--417, 2005.

\bibitem[Razaviyayn et~al.(2020)Razaviyayn, Huang, Lu, Nouiehed, Sanjabi, and Hong]{razaviyayn2020nonconvex}
Meisam Razaviyayn, Tianjian Huang, Songtao Lu, Maher Nouiehed, Maziar Sanjabi, and Mingyi Hong.
\newblock Nonconvex min-max optimization: Applications, challenges, and recent theoretical advances.
\newblock \emph{IEEE Signal Processing Magazine}, 37\penalty0 (5):\penalty0 55--66, 2020.

\bibitem[Reisizadeh et~al.(2020)Reisizadeh, Farnia, Pedarsani, and Jadbabaie]{reisizadeh2020robust}
Amirhossein Reisizadeh, Farzan Farnia, Ramtin Pedarsani, and Ali Jadbabaie.
\newblock Robust federated learning: The case of affine distribution shifts.
\newblock \emph{Advances in Neural Information Processing Systems}, 33:\penalty0 21554--21565, 2020.

\bibitem[Shalev-Shwartz et~al.(2010)Shalev-Shwartz, Shamir, Srebro, and Sridharan]{shalev2010learnability}
Shai Shalev-Shwartz, Ohad Shamir, Nathan Srebro, and Karthik Sridharan.
\newblock Learnability, stability and uniform convergence.
\newblock \emph{The Journal of Machine Learning Research}, 11:\penalty0 2635--2670, 2010.

\bibitem[Sharma et~al.(2022)Sharma, Panda, Joshi, and Varshney]{sharma2022federated}
Pranay Sharma, Rohan Panda, Gauri Joshi, and Pramod Varshney.
\newblock Federated minimax optimization: Improved convergence analyses and algorithms.
\newblock In \emph{International Conference on Machine Learning}, pages 19683--19730. PMLR, 2022.

\bibitem[Sharma et~al.(2023)Sharma, Panda, and Joshi]{sharma2023federated}
Pranay Sharma, Rohan Panda, and Gauri Joshi.
\newblock Federated minimax optimization with client heterogeneity.
\newblock \emph{Transactions on machine learning research}, 2023.

\bibitem[Shen et~al.(2023)Shen, Huang, Zhang, and Shen]{shen2023stochastic}
Wei Shen, Minhui Huang, Jiawei Zhang, and Cong Shen.
\newblock Stochastic smoothed gradient descent ascent for federated minimax optimization.
\newblock \emph{arXiv preprint arXiv:2311.00944}, 2023.

\bibitem[Song et~al.(2019)Song, Kalluri, Grover, Zhao, and Ermon]{song2019learning}
Jiaming Song, Pratyusha Kalluri, Aditya Grover, Shengjia Zhao, and Stefano Ermon.
\newblock Learning controllable fair representations.
\newblock In \emph{The 22nd International Conference on Artificial Intelligence and Statistics}, pages 2164--2173. PMLR, 2019.

\bibitem[Sun and Wei(2022)]{sun2022communication}
Zhenyu Sun and Ermin Wei.
\newblock A communication-efficient algorithm with linear convergence for federated minimax learning.
\newblock \emph{Advances in Neural Information Processing Systems}, 35:\penalty0 6060--6073, 2022.

\bibitem[Tarzanagh et~al.(2022)Tarzanagh, Li, Thrampoulidis, and Oymak]{tarzanagh2022fednest}
Davoud~Ataee Tarzanagh, Mingchen Li, Christos Thrampoulidis, and Samet Oymak.
\newblock Fednest: Federated bilevel, minimax, and compositional optimization.
\newblock In \emph{International Conference on Machine Learning}, pages 21146--21179. PMLR, 2022.

\bibitem[Wang and Joshi(2021)]{wang2021cooperative}
Jianyu Wang and Gauri Joshi.
\newblock Cooperative sgd: A unified framework for the design and analysis of local-update sgd algorithms.
\newblock \emph{Journal of Machine Learning Research}, 22\penalty0 (213):\penalty0 1--50, 2021.

\bibitem[Wu et~al.(2023)Wu, Sun, Hu, Zhang, and Huang]{wu2023solving}
Xidong Wu, Jianhui Sun, Zhengmian Hu, Aidong Zhang, and Heng Huang.
\newblock Solving a class of non-convex minimax optimization in federated learning.
\newblock In \emph{Thirty-seventh Conference on Neural Information Processing Systems}, 2023.

\bibitem[Xian et~al.(2021)Xian, Huang, Zhang, and Huang]{xian2021faster}
Wenhan Xian, Feihu Huang, Yanfu Zhang, and Heng Huang.
\newblock A faster decentralized algorithm for nonconvex minimax problems.
\newblock \emph{Advances in Neural Information Processing Systems}, 34:\penalty0 25865--25877, 2021.

\bibitem[Xie et~al.(2023)Xie, Zhang, Shen, Liu, and Qian]{xie2023cdma}
Jiahao Xie, Chao Zhang, Zebang Shen, Weijie Liu, and Hui Qian.
\newblock Cdma: a practical cross-device federated learning algorithm for general minimax problems.
\newblock In \emph{Proceedings of the AAAI Conference on Artificial Intelligence}, volume~37, pages 10481--10489, 2023.

\bibitem[Xing et~al.(2021)Xing, Song, and Cheng]{xing2021algorithmic}
Yue Xing, Qifan Song, and Guang Cheng.
\newblock On the algorithmic stability of adversarial training.
\newblock \emph{Advances in neural information processing systems}, 34:\penalty0 26523--26535, 2021.

\bibitem[Yang et~al.(2022{\natexlab{a}})Yang, Liu, Zhang, and Liu]{yang2022sagda}
Haibo Yang, Zhuqing Liu, Xin Zhang, and Jia Liu.
\newblock Sagda: Achieving $\mathcal{O}(\epsilon^{-2})$ communication complexity in federated min-max learning.
\newblock \emph{Advances in Neural Information Processing Systems}, 35:\penalty0 7142--7154, 2022{\natexlab{a}}.

\bibitem[Yang et~al.(2022{\natexlab{b}})Yang, Hu, Lei, Vashney, Lyu, and Ying]{yang2022differentially}
Zhenhuan Yang, Shu Hu, Yunwen Lei, Kush~R Vashney, Siwei Lyu, and Yiming Ying.
\newblock Differentially private sgda for minimax problems.
\newblock In \emph{Uncertainty in Artificial Intelligence}, pages 2192--2202. PMLR, 2022{\natexlab{b}}.

\bibitem[Ying et~al.(2021)Ying, Yuan, Chen, Hu, Pan, and Yin]{ying2021exponential}
Bicheng Ying, Kun Yuan, Yiming Chen, Hanbin Hu, Pan Pan, and Wotao Yin.
\newblock Exponential graph is provably efficient for decentralized deep training.
\newblock \emph{Advances in Neural Information Processing Systems}, 34:\penalty0 13975--13987, 2021.

\bibitem[Zhang et~al.(2021)Zhang, Hong, Wang, and Zhang]{zhang2021generalization}
Junyu Zhang, Mingyi Hong, Mengdi Wang, and Shuzhong Zhang.
\newblock Generalization bounds for stochastic saddle point problems.
\newblock In \emph{International Conference on Artificial Intelligence and Statistics}, pages 568--576. PMLR, 2021.

\bibitem[Zhang et~al.(2024)Zhang, Zhang, Yang, Souvenir, and Gao]{zhang2024federated}
Xinwen Zhang, Yihan Zhang, Tianbao Yang, Richard Souvenir, and Hongchang Gao.
\newblock Federated compositional deep auc maximization.
\newblock \emph{Advances in Neural Information Processing Systems}, 36, 2024.

\bibitem[Zhang et~al.(2023)Zhang, Qiu, and Gao]{zhang2023communication}
Yihan Zhang, Meikang Qiu, and Hongchang Gao.
\newblock Communication-efficient stochastic gradient descent ascent with momentum algorithms.
\newblock In \emph{IJCAI}, 2023.

\bibitem[Zhao et~al.(2018)Zhao, Zhang, Wu, Moura, Costeira, and Gordon]{zhao2018adversarial}
Han Zhao, Shanghang Zhang, Guanhang Wu, Jos{\'e}~MF Moura, Joao~P Costeira, and Geoffrey~J Gordon.
\newblock Adversarial multiple source domain adaptation.
\newblock \emph{Advances in neural information processing systems}, 31, 2018.

\bibitem[Zhu et~al.(2023)Zhu, Shen, Du, and Tao]{zhu2023stability}
Miaoxi Zhu, Li~Shen, Bo~Du, and Dacheng Tao.
\newblock Stability and generalization of the decentralized stochastic gradient descent ascent algorithm.
\newblock In \emph{Thirty-seventh Conference on Neural Information Processing Systems}, 2023.

\end{thebibliography}
\end{document}